\documentclass[11pt, onecolumn]{article}

\usepackage[usenames,dvipsnames]{color}
\usepackage{stmaryrd} \usepackage{url} \usepackage[latin1]{inputenc}
\usepackage{graphicx}
\usepackage{amssymb} \usepackage{subfigure} \usepackage{amsmath}
\usepackage{amsthm}
\usepackage{dcolumn}
\usepackage{bm}
\usepackage{color}
\usepackage{mathtools}

\usepackage[round]{natbib}

\usepackage[noline,boxed]{algorithm2e}

\newcommand{\STATE}{\quad}
\newcommand{\COMMENT}[1]{{\sl (#1)}}
\newcommand{\RETURN}{\tb{Return }}

\newtheorem{theorem}{Theorem}[section]

\newtheorem{lemma}[theorem]{Lemma}
\newtheorem{corollary}[theorem]{Corollary}
\newtheorem{remark}[theorem]{Remark}

\newtheorem{claim}[theorem]{Claim}
\newtheorem{condition}[theorem]{Condition}
\newtheorem{definition}[theorem]{Definition} 
 
\newtheorem{step}{Step}

\newcommand{\mc}{\mathcal}
\newcommand{\mb}{\mathbf} \newcommand{\mbb}{\mathbb} \newcommand{\wt}{\widetilde}
\newcommand{\fa}{\forall}  \newcommand{\tb}{\textbf}
\newcommand{\tx}{\text}

\newcommand{\bs}{\boldsymbol}
\newcommand{\ot}{\otimes}

\newcommand{\Proj}{\mbox{Proj}}
\newcommand{\R}{\mbb R}
\newcommand{\E}{\mbb E}

\newcommand{\ol}{\overline}
\newcommand{\wh}{\widehat}
\newcommand{\prj}{\tx{Proj}}
\newcommand{\mog}{mixture of Gaussians }
\newcommand{\vc}{\tx{vec}}
\newcommand{\mt}{\tx{mat}}
\newcommand{\diag}{\tx{diag}}
\newcommand{\poly}{\tx{poly}}
\newcommand{\ep}{\tx{exp}}
\newcommand{\lt}{\left}
\newcommand{\rt}{\right}

\newcommand{\defeq}{\vcentcolon=}

\newcommand{\od}{\odot}

\usepackage{enumitem}
\usepackage[margin=1in]{geometry}
\pdfoutput=1

\newcommand{\qt}[1]{}
\newcommand{\qq}[1]{#1}
\newcommand{\rg}[1]{}
\newcommand{\sk}[1]{}

\author{Rong Ge\thanks{Microsoft Research, New England. Email: rongge@microsoft.com} \and Qingqing
  Huang\thanks{MIT EECS. Email:qqh@mit.edu. Part of this work was done when the author was interning at Microsoft Research New England} \and Sham M. Kakade\thanks{Microsoft Research, New England. Email: skakade@microsoft.com}} \date{}

\begin{document}
\title{Learning Mixtures of Gaussians in High Dimensions}

\maketitle
\thispagestyle{empty}

\begin{abstract}
  Efficiently learning mixture of Gaussians is a fundamental problem in statistics and
  learning theory.  Given samples coming from a random one out of $k$ Gaussian
  distributions in $\R^n$, the learning problem asks to estimate the means and the
  covariance matrices of these Gaussians.  This learning problem arises in many areas
  ranging from the natural sciences to the social sciences, and has also found many
  machine learning applications.

  Unfortunately,  learning mixture of Gaussians is an information theoretically
  hard problem: in order to learn the parameters up to a reasonable accuracy, the number of samples
  required is exponential in the number of Gaussian components in the worst case.
  %
  In this work, we show that provided we are in high enough dimensions, the class of Gaussian
  mixtures is learnable in its most general form under a smoothed analysis framework, where the
  parameters are randomly perturbed from an adversarial starting point.

  In particular, given samples from a mixture of Gaussians with randomly perturbed
  parameters, when $n\ge \Omega(k^2)$, we give an algorithm that learns the parameters
  with polynomial running time and using polynomial number of samples.

  The central algorithmic ideas consist of new ways to decompose the  moment tensor of
  the Gaussian mixture by exploiting its structural properties.
  The symmetries of this tensor are derived from the combinatorial structure of higher
  order moments of Gaussian distributions (sometimes referred to as Isserlis' theorem or
  Wick's theorem).
  We also develop new tools for bounding smallest singular values of structured random
  matrices, which could be useful in other smoothed analysis settings.

  \rg{There were some confusions in notations ($\wt .$ and $\wh .$), should propagate the changes to the proof.}
   \qt{ need to highlight the deterministic conditions.}
\end{abstract}

\newpage
\setcounter{page}{1}

\section{Introduction}
Learning mixtures of Gaussians is a fundamental problem in statistics and learning theory,
whose study dates back to \citet{pearson1894contributions}. Gaussian mixture models arise
in numerous areas including physics, biology and the social sciences
(\cite{mclachlan2004finite, titterington1985statistical}), as well as in image processing
(\cite{reynolds1995robust}) and speech (\cite{permuter2003gaussian}).

In a Gaussian mixture model, there are $k$ unknown $n$-dimensional multivariate Gaussian
distributions.  Samples are generated by first picking one of the $k$ Gaussians, then drawing a
sample from that Gaussian distribution. Given samples from the mixture distribution, our goal is to
estimate the means and covariance matrices of these underlying Gaussian distributions\footnote{ This
  is different from the problem of {\em density estimation} considered in
  \cite{feldman2006pac,densityestimation14}}.
%
%


This problem has a long history in theoretical computer science. The seminal work of
\cite{dasgupta1999learning} gave an algorithm for learning spherical Gaussian mixtures
when the means are well separated. Subsequent works
(\cite{dasgupta2000two,sanjeev2001learning,vempala2004spectral,brubaker2008isotropic})
developed better algorithms in the well-separated case, relaxing the spherical assumption
and the amount of separation required.

When the means of the Gaussians are not separated, after several works (\cite{belkin2009learning,
  kalai2010efficiently}), \cite{belkin2010polynomial} and \cite{moitra2010settling} independently gave
algorithms that run in polynomial time and with polynomial number of samples for a fixed number of
Gaussians. However, both running time and sample complexity depend {\sl super} exponentially on the
number of components $k$\footnote{ In fact, it is in the order of \qq{$O({e^{O(k)}}^k)$} as shown in
  Theorem 11.3 in \cite{valiant2012algorithmic}.  }. Their
algorithm is based on the {\em method of moments} introduced by \cite{pearson1894contributions}:
first estimate the $O(k)$-order moments of the distribution, then try to find the parameters that
agree with these moments.
\cite{moitra2010settling} also show that the exponential dependency of the sample
complexity on the number of components is necessary, by constructing an example of two
mixtures of Gaussians with very different parameters, yet with exponentially small
statistical distance.

Recently, \cite{hsu2013learning} applied spectral methods to learning mixture of spherical
Gaussians. When $n\ge k+1$ and the means of the Gaussians are linearly independent, their algorithm
can learn the model in polynomial time and with polynomial number of samples. This result suggests
that the lower bound example in \cite{moitra2010settling} is only a {\em degenerate} case in high
dimensional space. In fact, {\em most} (in general position) mixture of spherical Gaussians are {\em
  easy} to learn. This result is also based on the method of moments, and only uses second and third
moments. Several follow-up works (\cite{bhaskara2013smoothed,anderson2013more}) use higher order
moments to get better dependencies on $n$ and $k$.

However, the algorithm in \cite{hsu2013learning} as well as in the follow-ups all make
strong requirements on the covariance matrices. In particular, most of them only apply to
learning mixture of spherical Gaussians.
For mixture of Gaussians with general covariance matrices, the best known result is still
\cite{belkin2010polynomial} and \cite{moitra2010settling}, which algorithms are not polynomial in the
number of components $k$.
This leads to the following natural question:

\medskip

{\noindent {\bf Question:} \sl Is it possible to learn  {\em most} mixture of Gaussians in
  polynomial time using a polynomial number of samples?}

\paragraph{Our Results}
In this paper, we give an algorithm that learns {\em most} mixture of Gaussians in high dimensional
space (when $n \ge \Omega(k^2)$), and the argument is formalized under the {\em smoothed analysis}
framework first proposed in \cite{spielman2004smoothed}.

In the smoothed analysis framework, the adversary first choose an arbitrary mixture of
Gaussians.
Then the mean vectors and covariance matrices of this Gaussian mixture are randomly {\em perturbed}
by a small amount $\rho$ \footnote{See Definition~\ref{def:smoothgaussian} in
  Section~\ref{sec:smooth} for the details.}.  The samples are then generated from the Gaussian
mixture model with the perturbed parameters. The goal of the algorithm is to learn the perturbed
parameters from the samples.

The smoothed analysis framework is a natural bridge between worst-case
and average-case analysis. On one hand, it is similar to worst-case
analysis, as the adversary chooses the initial instance, and the
perturbation allowed is small. On the other hand, even with small
perturbation, we may hope that the instance be different enough from
degenerate cases. A successful algorithm in the smoothed analysis setting
suggests that the bad instances must be very ``sparse'' in the
parameter space: they are highly unlikely in any small neighborhood of
any instance.
Recently, the smoothed analysis framework has also motivated several research work
(\cite{kalai2009learning} \cite{bhaskara2013smoothed}) in analyzing learning algorithms.

In the smoothed analysis setting, we show that it is easy to learn most Gaussian mixtures:

\begin{theorem}
  \label{thm:inform-general}
  (informal statement of Theorem~\ref{thm:main}) In the smoothed analysis setting, when $ n\ge \Omega(k^2)$, given samples
  from the perturbed $n$-dimensional Gaussian mixture model with $k$ components, there is
  an algorithm that learns the correct parameters up to accuracy $\epsilon$ with high
  probability, using polynomial time and number of samples.
\end{theorem}

An important step in our algorithm is to learn Gaussian mixture models whose components
all have mean zero, which is also a problem of independent interest (\cite{NIPS2012_4758}).
Intuitively this is also a ``hard'' case, as there is no separation in the means.
Yet algebraically, this case gives rise to a novel tensor decomposition algorithm.
The ideas for solving this decomposition problem are then generalized to tackle the most
general case.


%

\begin{theorem} \label{thm:zeromean:informal}
  (informal statement of Theorem~\ref{thm:main-zero-mean}) In the smoothed analysis setting,
  when $n\ge \Omega(k^2)$, given samples from the perturbed mixture of {\em zero-mean}
  $n$-dimensional Gaussian mixture model with $k$ components, there is an algorithm that
  learns the parameters up to accuracy $\epsilon$ with high probability, using polynomial
  running time and number of samples.
\end{theorem}

\paragraph{Organization} The main part of the paper will focus on learning mixtures of
zero-mean Gaussians.
The proposed algorithm for this special case contains most of the new ideas and
techniques.
In Section~\ref{sec:notations} we introduce the notations for matrices and tensors which are used to
handle higher order moments throughout the discussion. Then in Section~\ref{sec:main-results} we
introduce the smoothed analysis model for learning mixture of Gaussians and discuss the moment
structure of mixture of Gaussians, then we formally state our main
theorems. Section~\ref{sec:our-algorithm-zero} outlines our algorithm for learning zero-mean mixture
of Gaussians. The details of the steps are presented in Section~\ref{sec:implement}. The detailed
proofs for the correctness and the robustness are deferred to Appendix (Sections~\ref{sec:step-1} to
\ref{sec:step-3-zero}). In Section~\ref{sec:general} we briefly discuss how the ideas for zero-mean
case can be generalized to learning mixture of nonzero Gaussians, for which the detailed algorithm
and the proofs are deferred to Appendix~\ref{sec:general-case}.


\section{Notations}
\label{sec:notations}

\paragraph{Vectors and Matrices}
In the vector space $\R^{n}$, let $\lt<\cdot,\cdot\rt>$  denote the inner product of two
vectors, and $\|\cdot\|$ to denote the Euclidean norm.

For a  tall matrix $A\in\R^{m\times n}$, let $A_{[:,j]}$ denote its $j$-th column vector, let $A^\top$ denote
its transpose, $A^\dag = (A^\top A)^{-1}A^\top$ denote the pseudoinverse, and let $\sigma_k (A)$ denote
its $k$-th singular value. Let  $I_n$ be the identity matrix of dimension $n\times n$. The
spectral norm of a matrix is denoted as $\|\cdot \|$, and the Frobenius norm is
denoted as $\|\cdot\|_F$. We use $A\succeq 0$ for positive semidefinite matrix $A$.

In the discussion, we often need to convert between vectors and matrices.  Let $\vc(A)\in\R^{mn}$ denote the vector
obtained by stacking all the columns of $A$. For a vector
$x\in\R^{m^2}$, let $\mt(x)\in\R^{m\times m}$ denote
the inverse mapping such that $\vc(\mt(x)) = x$.

We use $[n]$ to denote the set $\{1,2,...,n\}$ and $[n]\times [n]$ to denote the set $\{(i,j):i,j\in
[n]\}$. These are often used as indices of matrices.

\paragraph{Symmetric matrices}
We use $\R_{sym}^{n\times n}$ to denote the space of all $n\times n$ symmetric matrices, which subspace has dimension
${n+1\choose 2}$. Since we will frequently use $n\times n$ and $k\times k$ symmetric matrices, we denote their
dimensions by the constants $\qq{n_2 = {n+1\choose 2}}$ and $\qq{k_2 = {k+1\choose 2}}$.
Similarly, we use $\R_{sym}^{n\times \dots\times n}$ to denote the symmetric $k$-dimensional
multi-arrays (tensors), which subspace has dimension ${n+k-1\choose k}$. If a $k$-th order tensor
$X\in\R_{sym}^{n\times\dots\times n}$, then for any permutation $\pi$ over $[k]$, we have
$X_{n_1,\dots, n_k} = X_{n_{\pi(1)}, \dots, n_{\pi(k)}} $.
%
%

%

%
\paragraph{Linear subspaces} We represent a linear subspace $\mc S \in \R^n$ of dimension $d$ by a matrix $S\in
\R^{n\times d}$, whose columns of $S$ form an (arbitrary) orthonormal basis of the subspace.
The projection matrix onto the subspace $\mc S$  is denoted by $
\prj_{S} = S S^{\top},$ and the projection onto the orthogonal subspace $\mc S^{\perp}$ is
denoted by $\prj_{S^{\perp} } = I_n - S S^{\top}.$ When we talk about the span of several matrices, we mean the space spanned by their vectorization.

%

\paragraph{Tensors}
A tensor is a multi-dimensional array. Tensor notations are useful for handling higher order moments. We use $\ot$ to denote tensor product, suppose $a,b,c\in \R^n$, $T=a\otimes b\otimes c\in \R^{n\times n\times n}$ and $T_{i_1,i_2,i_3} = a_{i_1}b_{i_2}c_{i_3}$.
For a vector $x\in\R^{n }$, let the $t$-fold tensor product $x\ot^{t}$  denote the $t$-th order
rank one tensor $(x\ot^t)_{i_1,i_2,...,i_t} = \prod_{j=1}^t x_{i_j}$.
%


Every tensor defines a multilinear mapping. Consider a 3-rd order tensor $X\in\R^{n_A\times n_B\times n_C}$.  For given
dimension $m_A,m_B,m_C$, it defines a multi-linear mapping $X(\cdot,\cdot,\cdot): \mbb R^{n_A\times m_A}\times \mbb
R^{n_B\times m_B}\times\mbb R^{n_C\times m_C} \to \mbb R^{m_A\times m_B\times m_C}$ defined as below: ($\fa j_1\in[m_A],
j_2\in[m_B], j_3\in[m_C]$)
\begin{align*}
  [ X(V_1, V_2, V_3)]_{j_1,j_2,j_3}= \sum_{ i_1\in[n_A], i_2\in[n_B],
    i_3\in[n_C]} X_{i_1,i_2,i_3} [V_1]_{j_1,i_1} [V_2]_{j_2,i_2} [V_3]_{j_3,i_3}.
\end{align*}
If $X$ admits a decomposition  $X = \sum_{i=1}^k A_{[:,i]}\ot B_{[:,i]}\ot C_{[:,i]}$ for $A\in\R^{n_A\times k},
B\in\R^{n_B\times k}, C\in\R^{n_C\times k}$, the multi-linear mapping has the form $ X(V_1, V_2, V_3) =\sum_{i=1}^k
(V_1^\top A_{[:,i]})\ot (V_2^\top B_{[:,i]}) \ot (V_3^\top C_{[:,i]}).$

In particular, the vector given by $X(\mb e_i, \mb e_j, I)$ is the one-dimensional slice
of the 3-way array, with the index for the first dimension to be $i$ and the second
dimension to be $j$.

\paragraph{Matrix Products}
We use $\odot$ to denote column wise Katri-Rao product, and $\ot_{kr}$ to
denote Kronecker product.
As an example, for matrices $A\in\R^{m_A\times n}$, $B\in\R^{m_B\times n}$, $C\in\R^{m_C\times n}$:
\begin{align*}
&  A\ot B \ot C\in \R^{m_A\times m_B\times m_C},\quad [A\ot B\ot C]_{j_1,j_2,j_3} = \sum_{i=1}^{n} A_{j_1,i}
  B_{j_2,i} C_{j_3,i},
\\
&A\od B \in\R^{m_A m_B\times n},\quad [A\od B]_{[:, j]} = A_{[:,j]} \ot_{kr} B_{[:,j]}.
\\
&  A\ot_{kr}B \in\R^{m_A m_B \times n^2}, \quad A\ot_{kr}B =
\lt[  \begin{array}[c]{ccc}
    A_{1,1} B & \cdots & A_{1,n} B
    \\
    \vdots & \ddots & \vdots
    \\
    A_{m_A, 1} B & \cdots & A_{m_A,n} B
  \end{array}
 \rt],
\end{align*}

\section{Main results}
\label{sec:main-results}
In this section, we first formally introduce the smoothed analysis framework for our
problem and state our main theorems.  Then we will discuss the structure of the moments of
Gaussian mixtures, which is crucial for understanding our method of moments based algorithm.

\subsection{Smoothed Analysis for Learning Mixture of  Gaussians }
\label{sec:smooth}
Let $ \mc G_{n,k}$ denote the class of Gaussian mixtures with $k$ components in $\R^n$.  A
distribution in this family is specified by the following parameters: the mixing weights
$\omega_i$, the mean vectors $\mu^{(i)}$ and the covariance matrices $\Sigma^{(i)}$, for
$i\in[k]$.
\begin{align*}
  \mc G_{n,k}:= \lt\{ \mathcal{G} = \{(\omega_i, \mu^{(i)}, \Sigma^{(i)})\}_{i\in[k]}:
  \omega_i\in \R_+,\ \sum_{i=1}^k \omega_i = 1,\
  \mu^{(i)}\in \R^n,\ \Sigma^{(i)}\in \R_{sym}^{n\times n},\ \Sigma^{(i)} \succeq 0\rt\}.
\end{align*}
%
As an interesting special case of the general model, we also consider the mixture of ``zero-mean''
Gaussians, which has $\mu^{(i)}=0$ for all components $i\in[k]$.


A sample $x$ from a mixture of Gaussians is generated in two steps:
\begin{enumerate}{\itemsep=0pt}
\item Sample $h\in [k]$ from a multinomial distribution, with probability $\Pr[h = i] =
  \omega_i$ for $i\in[k]$.
\item Sample $x\in\R^n$ from the $h$-th Gaussian distribution $\mc N(\mu^{(h)},
  \Sigma^{(h)})$.
\end{enumerate}
The learning problem asks to estimate the parameters of the underlying mixture of Gaussians:

\begin{definition} [Learning mixture of Gaussians]
  \label{def:learn-mix-G}
  Given $N$ samples $x_1,x_2,...,x_N$ drawn i.i.d. from a mixture of Gaussians
  $\mathcal{G} =\{(\omega_i, \mu^{(i)}, \Sigma^{(i)})\}_{i\in[k]} $, an algorithm learns
  the mixture of Gaussians with accuracy $\epsilon$, if it outputs an estimation $\wh{\mc
    G }= \{(\wh\omega_i, \wh\mu^{(i)}, \wh\Sigma^{(i)})\}_{i\in[k]}$ such that there
  exists a permutation $\pi$ on $[k]$, and for all $i\in[k]$, we have
  $|\wh\omega_i-\omega_{\pi(i)}|\le \epsilon$, $\|\wh \mu^{(i)} - \mu^{(\pi(i))}\| \le
  \epsilon$ and $\|\wh \Sigma^{(i)} - \Sigma^{(\pi(i))}\| \le \epsilon$.
\end{definition}

In the worst case, learning mixture of Gaussians is a information theoretically hard
problem (\cite{moitra2010settling}). There exists worst-case examples where the number of
samples required for learning the instance is at least exponential in the number of
components $k$ (\cite{mclachlan2004finite}).
The non-convexity arises from the hidden variable $h$: without knowing $h$ we cannot
determine which Gaussian component each sample comes from.

The smoothed analysis framework provides a way to circumvent the worst case instances,
yet still studying this problem in its most general form.
The basic idea is that, with high probability over the small random perturbation to any
instance, the instance will not be a ``worst-case'' instance, and actually has reasonably
good condition for the algorithm.

Next, we show how the parameters of the mixture of Gaussians are {\em perturbed} in our
setup.

\begin{definition}[$\rho$-smooth mixture of Gaussian]
  \label{def:smoothgaussian}
  For $\rho < 1/n$, a $\rho$-smooth $n$-dimensional $k$-component mixture of Gaussians
  $\wt{\mc G}=\{(\wt \omega_i,\wt \mu^{(i)},\wt \Sigma^{(i)})\}_{i\in[k]}\in\mc G_{n,k}$
  is generated as follows:
  \begin{enumerate}{\itemsep=-0.1pt}
  \item Choose an arbitrary (could be adversarial) instance $\mc G=\{(\omega_i, \mu^{(i)},
    \Sigma^{(i)})\}_{i\in [k]}\in\mc G_{n,k}$. Scale the distribution such that \qq{
      $0\preceq \Sigma^{(i)} \preceq {1\over 2} I_n$ and $\|\mu^{(i)}\|\le{ 1\over 2}$}
    for all $i\in[k]$.
  \item Let $\Delta_i \in \R_{sym}^{n\times n}$ be a random symmetric matrix with zeros on
    the diagonals, and the upper-triangular entries are independent random Gaussian
    variables $\mc N(0,\rho^2)$. Let $\delta_i\in \R^n$ be a random Gaussian vector with
    independent Gaussian variables $\mc N(0,\rho^2)$.
  \item Set $\wt \omega_i = \omega_i$, $\wt \mu^{(i)} = \mu^{(i)} + \delta_i$, $\wt \Sigma^{(i)} = \Sigma^{(i)} +
    \Delta_i$.
  \item Choose the diagonal entries of $\wt \Sigma^{(i)}$ arbitrarily, while ensuring the
    positive semi-definiteness of the covariance matrix $\wt \Sigma^{(i)}$, and the
    diagonal entries are upper bounded by $1$. The perturbation procedure fails if this
    step is infeasible\footnote{ Note that by standard random matrix theory, with high
      probability the 4-th step is feasible and the perturbation procedure in
      Definition~\ref{def:smoothgaussian} succeeds.  Also, with high probability we have
     \qq{ $\|\wt \mu^{(i)}\|\le 1$ and $0\preceq \wt\Sigma^{(i)} \preceq I_n$} for all
      $i\in[k]$.
    }.
  \end{enumerate}
  A $\rho$-smooth zero-mean mixture of Gaussians is generated using the same procedure, except that we set $\wt
  \mu^{(i)} = \mu^{(i)} = 0$, for all $i\in[k]$.
\end{definition}

\begin{remark}
  When the original matrix is of low rank, a simple random perturbation may not lead to a
  positive semidefinite matrix, which is why our procedure of perturbation is more
  restricted in order to guarantee that the perturbed matrix is still a valid covariance
  matrix.

There could be other ways of locally perturbing the covariance matrix.  Our procedure
actually gives {\em more power} to the adversary as it can change the diagonals {\em
  after} observing the perturbations for other entries.
Note that with high probability if we just let the new diagonal to be $5\sqrt{n}\rho$
larger than the original ones, the resulting matrix is still a valid covariance matrix. In
other words, the adversary can always keep the perturbation small if it wants to.
\end{remark}

Instead of the worst-case problem in Definition~\ref{def:learn-mix-G}, our algorithms work on the smoothed instance. Here the model first gets perturbed to $\wt{\mc{G}} =\{(\wt \omega_i,\wt \mu^{(i)}, \wt\Sigma^{(i)})\}_{i\in[k]} $, the samples are drawn according to the perturbed model, and the algorithm tries to learn the perturbed parameters. We give a polynomial time algorithm in this case:



\begin{theorem}[Main theorem]
  \label{thm:main}
  Consider a $\rho$-smooth mixture of Gaussians $\wt{\mathcal{G}}=\{(\wt \omega_i, \wt \mu^{(i)},
  \wt \Sigma^{(i)})\}_{i\in[k]}\in\mc G_{n,k}$ for which the number of components is at least
  \footnote{Note that the algorithms of \cite{belkin2010polynomial} and \cite{moitra2010settling}
    run in polynomial time for fixed $k$.
  }
  \qq{$k \ge C_0$ and the dimension $n\ge C_1k^2$}, for some fixed constants $C_0$ and $C_1$.  Suppose
  that the mixing weights $\wt \omega_i\ge \omega_o$ for all $i\in[k]$.
  Given $N$ samples drawn i.i.d. from $\wt{\mc G}$, there is an algorithm that learns the
  parameters of $\wt{\mc G}$ up to accuracy $\epsilon$, with high probability over the
  randomness in both the perturbation and the samples. Furthermore, the running time and
  number of samples $N$ required are both upper bounded by
  $\poly(n,k,1/\omega_o,1/\epsilon,1/\rho)$.
\end{theorem}

%
To better illustrate the algorithmic ideas for the general case, we first present an
algorithm for learning mixtures of zero-mean Gaussians. Note that this is not just a
special case of the general case, as with the smoothed analysis, the zero mean vectors are
not perturbed.

\begin{theorem}[Zero-mean]
  \label{thm:main-zero-mean}
  Consider a $\rho$-smooth mixture of {zero-mean} Gaussians
  $\wt{\mathcal{G}}=\{(\wt\omega_i, 0,\wt \Sigma^{(i)})\}_{i\in[k]}\in\mc G_{n,k}$ for
  which the number of components is at least \qq{$k \ge C_0$ and the dimension $n\ge
    C_1k^2$}, for some fixed constants $C_0$ and $C_1$.  Suppose that the mixing weights
  $\wt \omega_i\ge \omega_o$ for all $i\in[k]$.
  Given $N$ samples drawn i.i.d. from $\wt{\mc G}$, there is an algorithm that learns the
  parameters of $\wt{\mc G}$ up to accuracy $\epsilon$, with high probability over the
  randomness in both the perturbation and the samples. Furthermore, the running time and
  number of samples $N$ are both upper bounded by $\poly(n,k,1/\omega_o,1/\epsilon,1/\rho)$.
\end{theorem}

Throughout the paper we always assume that $n \ge C_1k^2$ and $\wt \omega_i \ge \omega_o$.

\subsection{Moment Structure of Mixture of Gaussians}
\label{subsec:moment-structure-mog}




Our algorithm is also based on the method of moments, and we only need to estimate the $3$-rd, the
$4$-th and the $6$-th order moments. In this part we briefly discuss the structure of $4$-th and
$6$-th moments in the zero-mean case ($3$-rd moment is always 0 in the zero-mean case). These
structures are essential to the proposed algorithm.  For more details, and discussions on the
general case see Appendix~\ref{sec:app:moment}.

The $m$-th order moments of the {\em zero-mean} Gaussian mixture model ${\mc G}\in\mc
G_{n,k}$ are given by the following $m$-th order symmetric tensor
$M_m\in\R^{n\times\dots\times n}_{sym}$:
\begin{align*}
  \lt[ M_m \rt]_{j_1,\dots, j_m}:= \mbb E\lt[x_{j_1}\dots x_{j_m}\rt] = \sum_{i=1}^{k}\omega_i  \mbb
  E\lt[y^{(i)}_{j_1}\dots y^{(i)}_{j_m}\rt], \quad \fa j_1,\dots,j_m\in[n],
\end{align*}
where $ y^{(i)}$ corresponds to the $n$-dimensional zero-mean Gaussian distribution $\mc N(0, \Sigma^{(i)})$.
The moments for each Gaussian component are characterized by Isserlis's theorem as below:

\begin{theorem}[Isserlis' Theorem]
\label{prop:isserlis}
Let $(y_1,\dots, y_{2t})$ be a multivariate zero-mean Gaussian random vector $\mc
N(0,\Sigma)$, then
  \begin{align*}
    &\mbb E[y_1\dots y_{2t}] = \sum \prod \Sigma_{u,v},
  \end{align*}
  where the summation is taken over all distinct ways of partitioning $y_1,\dots, y_{2t}$ into $t$
  pairs, which correspond to all the perfect matchings in a complete graph.
\end{theorem}

Ideally, we would like to obtain the following quantities (recall $\qq{n_2 = {n+1\choose 2}}$):
\begin{align}
  \label{eq:def-X4-X6}
  &X_4 = \sum_{i=1}^{k}\omega_i \vc(\Sigma^{(i)})\ot^2 \in \mbb R^{n_2\times n_2}, \quad X_6 =
  \sum_{i=1}^{k}\omega_i\vc (\Sigma^{(i)})\ot^3 \in \mbb R^{n_2\times n_2\times n_2}.
\end{align}
%

Note that the entries in $X_4$ and $X_6$ are quadratic and cubic monomials of the
covariance matrices, respectively. If we have $X_4$ and $X_6$, the tensor decomposition
algorithm in \cite{anandkumar2012tensor} can be immediately applied to recover $
\omega_i$'s and $ \Sigma^{(i)}$'s under mild conditions. It is easy to verify that those
conditions are indeed satisfied with high probability in the smoothed analysis setting.

By Isserlis's theorem, the entries of the moments $M_4$ and $M_6$ are indeed quadratic and
cubic functions of the covariance matrices, respectively.  However, the structure of the
true moments $M_4$ and $M_6$ have more symmetries, consider for example,
\begin{align*}
  [M_4]_{1,2,3,4} = \sum_{i=1}^k \omega_i
  (\Sigma^{(i)}_{1,2}\Sigma^{(i)}_{3,4}+\Sigma^{(i)}_{1,3}\Sigma^{(i)}_{2,4}+\Sigma^{(i)}_{1,4}\Sigma^{(i)}_{2,3}),
  \quad \tx{while } [X_4]_{(1,2),(3,4)} = \sum_{i=1}^k \omega_i \Sigma^{(i)}_{1,2}\Sigma^{(i)}_{3,4}.
\end{align*}
Note that due to symmetry, the number of distinct entries in $M_4$ ( ${n+3\choose
  4}\approx n^4/24$) is much smaller than the number of distinct entries in $X_4$
(${n_2+1\choose 2} \approx n^4/8$).  Similar observation can be made about $M_6$ and
$X_6$.

Therefore, it is not immediate how to find the desired $X_4$ and $X_6$ based on $M_4$ and
$M_6$.
We call the moments $M_4,M_6$ the {\em folded moments} as they have more symmetry, and the
corresponding $X_4,X_6$ the {\em unfolded moments}.
One of the key steps in our algorithm is to unfold the true moments $M_4,M_6$ to get $X_4,
X_6$ by exploiting special structure of $M_4,M_6$.


%
In some cases, it is easier to restrict our attention to the entries in $M_4$ with indices corresponding to distinct
variables.
In particular, we define
\begin{align}
  \label{eq:def-ol-M4}
  \ol M_4 = \lt[[M_4]_{j_1,j_2,j_3,j_4}: 1\le j_1<j_2<j_3<j_4\le n\rt]\in\R^{n_4},
\end{align}
where $\qq{n_4 = {n \choose 4}}$ is the number of 4-tuples with indices corresponding to
distinct variables.  We define $\ol M_6\in\R^{n_6}$ similarly where $\qq{n_6 = {n\choose
    6}}$.  We will see that these entries are nice as they are {\em linear projections} of
the desired unfolded moments $X_4$ and $X_6$ (Lemma~\ref{prop:two-linear-mapping} below),
also such projections satisfy certain ``symmetric off-diagonal'' properties which are
convenient for the proof (see Definition~\ref{def:sym-off-diag-4} in
Section~\ref{sec:step-2}).
\begin{lemma}
  \label{prop:two-linear-mapping}
  For a zero-mean Gaussian mixture model, there exist two fixed and known linear mappings $\mc F_4:\mbb R^{n_2\times
    n_2}\to \mbb R^{n_4}$ and $\mc F_6:\mbb R^{n_2\times n_2\times n_2}\to \mbb R^{n_6}$ such that:
  \begin{align}
    \label{eq:F4F6-1}
    \ol M_4= { \sqrt{3}} \mc F_4(X_4),\quad \ol M_6 ={ \sqrt{15}}\mc F_6(X_6).
  \end{align}
  Moreover $\mc F_4$ is a projection from a ${n_2+1 \choose 2}$-dimensional
  subspace to a $n_4$-dimensional subspace, and $\mc F_6$ is a projection from a
  ${n_2+2\choose 3}$-dimensional subspace to a ${n_6}$-dimensional subspace.
  %
\end{lemma}

\section{Algorithm Outline for Learning Mixture of Zero-Mean Gaussians}
\label{sec:our-algorithm-zero}
In this section, we present our algorithm for learning zero-mean Gaussian mixture model.
The algorithmic ideas and the analysis are at the core of this paper.  Later we show that
it is relatively easy to generalize the basic ideas and the techniques to handle the
general case.

For simplicity we state our algorithm using the exact moments $\wt M_4$ and $\wt M_6$,
while in implementation the empirical moments $\wh M_4$ and $\wh M_6$ obtained with the
samples are used.
In later sections, we verify the correctness of the algorithm and show that it is robust:
the algorithm learns the parameters up to arbitrary accuracy
using polynomial number of samples.

\begin{step} Span Finding:
  Find the span of covariance matrices .
  \begin{enumerate}[label={\tb{(\alph*)}}, leftmargin=*,itemsep=0pt]
  \item 
    For a \qq{set of indices $\mc H \subset [n]$ of size $|\mc H| =\sqrt{n}$},
    find the span:
    \begin{align}
      \label{eq:def-mcS}
      \mc S = \tx{span}\lt\{\wt \Sigma^{(i)}_{[:,j]}: i\in[k],j\in \mc H\rt\} \subset \R^n.
    \end{align}

  \item Find the span of the covariance matrices with the columns projected onto $\mc
    S^{\perp}$, namely,
    \begin{align}
      \label{eq:def-mcUS}
      \mc U_S =\tx{span}\lt\{\vc(\prj_{S^{\perp}}\wt \Sigma^{(i)}) : i\in[k]\rt\} \subset \R^{n^2}.
    \end{align}

  \item For two disjoint sets of indices $\mc H_1$ and $\mc H_2$, repeat Step 1 (a) and
    Step 1 (b) to obtain $\mc U_1$ and $\mc U_2$, namely the span of covariance matrices
    projected onto two subspaces $\mc S_1^{\perp}$ and $\mc S_2^{\perp}$. Merge $\mc U_1$
    and $\mc U_2$ to obtain the span of covariance matrices $\mc U$:
    \begin{align}
      \label{eq:def-mcU}
      \mc U = \tx{span}\lt\{\wt \Sigma^{(i)}:i\in[k]\rt\}\subset \R^{n_2}.
    \end{align}

  \end{enumerate}
\end{step}

\begin{step}
  Unfolding: Recover the unfolded moments $\wt X_4,\wt X_6$.\\
  Given the folded moments $\ol{ \wt M}_4, \ol {\wt M}_6$ as defined in
  \eqref{eq:def-ol-M4}, and given the subspace
  $U\in\mbb R^{n_2\times k}$ from Step 1, let $\wt Y_4 \in \R^{k\times k}_{sym}$ and $\wt Y_6\in \R^{k\times
    k\times k}_{sym}$ be the unknowns, solve the following systems of linear equations.
  \begin{align}
    \label{eq:F4F6-2}
    \ol{\wt M}_4 = \sqrt{3} \mc F_4(U\wt Y_4 U^\top), \quad \ol{\wt M}_6 = \sqrt{15} \mc F_6(\wt
    Y_6(U^\top, U^\top, U^\top)).
  \end{align}
  The unfolded moments $\wt X_4,\wt X_6$ are then given by $\wt X_4 = U\wt Y_4 U^\top, \wt
  X_6 = \wt Y_6(U^\top, U^\top,U^\top).$

\end{step}

\begin{step} Tensor Decomposition: learn $\wt \omega_i$ and $\wt \Sigma^{(i)}$ from
  $\wt Y_4$ and $\wt Y_6$.\\
  Given $U$, and given $\wt Y_4$ and $\wt Y_6$ which are relate to the parameters as follows:
    $$\wt Y_4 = \sum_{i=1}^k \wt \omega_i (U^\top\wt\Sigma^{(i)})\ot ^2,
    \quad\wt Y_6 = \sum_{i=1}^k \wt \omega_i (U^\top\wt\Sigma^{(i)})\ot ^3, $$
    we apply tensor decomposition techniques to recover $\wt \Sigma^{(i)}$'s and $\wt
    \omega_i$'s.
\end{step}

%
%

\section{Implementing the Steps for Mixture of Zero-Mean Gaussians}
\label{sec:implement}
In this part we show how to accomplish each step of the algorithm outlined in
Section~\ref{sec:our-algorithm-zero} and sketch the proof ideas.

For each step, we first explain the detailed algorithm, and list the deterministic
conditions on the underlying parameters as well as on the {\em exact} moments for the step
to work correctly.
Then we show that these deterministic conditions are satisfied with high probability over
the $\rho$-perturbation of the parameters in the smoothed analysis setting.
In order to analyze the sample complexity, we further show that when we are given the {\em
  empirical} moments which are close to the exact moments, the output of the step is
also close to that in the exact case.

In particular we show the correctness and the stability of each step in the algorithm with
two main lemmas:
the first lemma shows that with high probability over the random perturbation of the
covariance matrices, the exact moments satisfy the deterministic conditions that ensure
the correctness of each step;
the second lemma shows that when the algorithm for each step works correctly, it is
actually stable even when the moments are estimated from finite samples and have only
inverse polynomial accuracy to the exact moments.
%

The detailed proofs are deferred to Section~\ref{sec:step-1} to \ref{sec:step-3-zero} in
the appendix.

\vspace*{-0.1in}
\paragraph{Step 1: Span Finding.}
Given the 4-th order moments $\wt M_4$, Step 1 finds the span of covariance matrices $\mc
U$ as defined in \eqref{eq:def-mcU}.
Note that by definition of the unfolded moments $\wt X_4$ in \eqref{eq:def-X4-X6}, the
subspace $\mc U$ coincides with the column span of the matrix $\wt X_4$.

By Lemma~\ref{prop:two-linear-mapping}, we know that the entries in $\wt M_4$ are linear mappings of
entries in $\wt X_4$. Since the matrix $\wt X_4$ is of low rank ($k\ll n_2$), this corresponds to
the {\em matrix sensing} problem first studied in \cite{recht2010guaranteed}.  In general, matrix
sensing problems can be hard even when we have many linear observations (\cite{moritz}).  Previous
works (\cite{recht2010guaranteed, hardt2014computational, jain2013low}) showed that if the linear
mapping satisfy {\em matrix RIP} property, one can uniquely recover $\wt X_4$ from $\wt M_4$.

However, properties like RIP do not hold in our setting where the linear mapping is determined by
Isserlis' Theorem.  We can construct two different mixtures of Gaussians with different unfolded
moments $\wt X_4$, but the same folded moment $\wt M_4$ (see Section~\ref {sec:two-mixtures-with}).
Therefore the existing matrix recovery algorithm cannot be applied, and we need to develop
new tools by exploiting the special moment structure of Gaussian mixtures.

\paragraph{Step 1 (a). Find the Span of a Subset of Columns of the Covariance Matrices.}
The key observation for this step is that if we hit $\wt M_4$ with three basis vectors, we
get a vector that lies in the span of the columns of the covariance matrices:
\begin{claim}
  \label{claim:M4-proj4}
  For a mixture of zero-mean Gaussians $\mc G =\{(\omega_i, 0,
  \Sigma^{(i)})\}_{i\in[k]}\in\mc G_{n,k}$, the one-dimensional slices of the 4-th order
  moments $M_4$ are given by: \vspace*{-0.1in}
  \begin{align}
    \label{eq:M4-proj4}
    M_4(\mb e_{j_1},\mb e_{j_2},\mb e_{j_3},I) = \sum_{i=1}^k \omega_i\left( \Sigma^{(i)}_{j_1,j_2}
      \Sigma^{(i)}_{[:,j_3]}+ \Sigma^{(i)}_{j_1,j_3} \Sigma^{(i)}_{[:,j_2]}+
      \Sigma^{(i)}_{j_2,j_3} \Sigma^{(i)}_{[:,j_1]}\right), \quad\fa j_1,j_2,j_3\in [n].
  \end{align}
\end{claim}
In particular, if we pick the indices $j_1,j_2,j_3$ in the index set $\mc H$, the vector $
M_4(\mb e_{j_1},\mb e_{j_2},\mb e_{j_3},I)$ lies in the desired span $\mc S = \lt\{
\Sigma^{(i)}_{[:,j]}: i\in[k],j\in \mc H\rt\}$.

We shall partition the set $\mc H$ into three disjoint subsets $\mc H^{(i)}$ of equal size
$\sqrt{n}/3$, and pick $j_i\in H^{(i)}$ for $i=1,2,3$.  In this way, we have $(|\mc
H|/3)^3 = \Omega(n^{1.5})$ such one-dimensional slices of $ M_4$, which all lie in the
desired subspace $\mc S$. Moreover, the dimension of the subspace $\mc S$ is at most $k|\mc
H| \ll n^{1.5}$. Therefore, with the $\rho$-perturbed parameters $\wt \Sigma^{(i)}$'s,
we can expect that with high probability the slices of $\wt M_4$ span the entire subspace
$\mc S$.



\begin{condition} [Deterministic condition for Step 1 (a)]
  Let $\wt Q_S\in\R^{n\times (|\mc H|/3)^3}$ be the matrix whose columns are the vectors
  $\wt M_4(\mb e_{j_1},\mb e_{j_2},\mb e_{j_3},I)$ for $j_i\in \mc H^{(i)}$.
  If the matrix $\wt Q_S$ achieves its maximal column rank $k|\mc H|$, we can find the
  desired span $\mc S$ defined in \eqref{eq:def-mcS} by the column span of matrix $\wt
  Q_S$.
\end{condition}

We first show that this deterministic condition is satisfied with high probability by bounding the
$k|\mc H|$-th singular value of $\wt Q_S$ with smoothed analysis.

\begin{lemma}[Correctness] 
\label{lem:1asmooth}
Given the exact 4-th order moments $\wt M_4$, for any index set $\mc H$ of size \qq{$|\mc
  H| = \sqrt{n}$},
With high probability, the $k|\mc H|$-th singular value of $\wt Q_S$ is at least
$\Omega(\omega_o \rho^2 n)$.
\end{lemma}

The proof idea involves writing the matrix $\wt Q_S$ as a product of three matrices,
and using the results on spectral properties of random matrices
\cite{rudelson2009smallest} to show that with high probability the smallest singular value
of each factor is lower bounded.



Since this step only involves the singular value decomposition of the matrix $\wt Q_S$, we
then use the standard matrix perturbation theory to show that this step is stable:
\begin{lemma}[Stability]
  \label{lem:1aperturb}
  Given the empirical estimator of the 4-th order moments $\wh M_4 = \wt M_4 + E_4$, suppose that
  the entries of $E_4$ have absolute value at most $\delta$.  Let the columns of matrix $\wt
  S\in\mbb R^{n\times k|\mc H|}$ be the left singular vector of $\wt Q_S$, and let $\wh S$ be the
  corresponding matrix obtained with $\wh M_4$. When $\delta$ is inverse polynomially small, the
  distance between the two projections $ \|\prj_{\wh S} - \prj_{\wt S}\|$ is upper bounded by $
  O\left({ n^{1.25} \delta / \sigma_{k|\mc H|} ( \wt Q_S)}\right)$.
\end{lemma}

\begin{remark}
  Note that we need the high dimension assumption ($n \gg k$) to guarantee the correctness of this
  step: in order to span the subspace $\mc S$, the number of distinct vectors should be equal or
  larger than the dimension of the subspace, namely $|\mc H|^3\ge k|\mc H| $; and the subspace
  should be non-trivial, namely $k|\mc H|< n$.  These two inequalities suggest that we need $n \ge
  \Omega(k^{1.5})$. However, we used the stronger assumption $n\ge \Omega(k^2)$ to obtain the lower
  bound of the smallest singular value in the proof.
\end{remark}

\paragraph{Step 1 (b). Find the Span of Projected Covariance Matrices. }
In this step, we continue to use the structural properties of the 4-th order moments. In
particular, we look at the two-dimensional slices of $M_4$ obtained by hitting it with two
basis vectors:
\begin{claim}
  \label{claim:M4-proj34}
  For a mixture of zero-mean Gaussians $\mc G =\{(\omega_i, 0,
  \Sigma^{(i)})\}_{i\in[k]}\in\mc G_{n,k}$, the two-dimensional slices of the 4-th order
  moments $M_4$ are given by: \vspace*{-0.1in}
  \begin{align}
    \label{eq:M4-proj34}
    M_4(\mb e_{j_1},\mb e_{j_2},I,I) = \sum_{i=1}^k \omega_i \left( \Sigma^{(i)}_{j_1,j_2}
      \Sigma^{(i)}+ \Sigma^{(i)}_{[:,j_1]} ( \Sigma^{(i)}_{[:,j_2]})^\top+ \Sigma^{(i)}_{[:,j_2]}
      ( \Sigma^{(i)}_{[:,j_1]})^\top\right),
    \quad\fa j_1,j_2\in [n].
  \end{align}
\end{claim}
Note that if we take the indices $j_1$ and $j_2$ in the index set $\mc H$, the slice $
M_4(\mb e_{j_1},\mb e_{j_2},I,I)$ is {\em almost} in the span of the covariance matrices,
except $2k$ additive rank-one terms in the form of $ \Sigma^{(i)}_{[:,j_1]} (
\Sigma^{(i)}_{[:,j_2]})^\top$. These rank-one terms can be eliminated by projecting the
slice to the subspace $\mc S^{\perp}$ obtained in Step 1 (a), namely, \vspace*{-0.1in}
\begin{align*}
  \vc(\prj_{S^{\perp}} M_4(\mb e_{j_1},\mb e_{j_2},I,I)) = \sum_{i=1}^k \omega_i
  \Sigma^{(i)}_{j_1,j_2} \vc(\prj_{S^{\perp}} \Sigma^{(i)} ), \quad\fa
  j_1,j_2\in\mc H,
\end{align*}
and this projected two-dimensional slice lies in the desired span $\mc U_S$ as defined in
\eqref{eq:def-mcUS}.
%
Moreover,  there are ${|\mc H|+1 \choose 2 } = \Omega(n)$ such projected two-dimensional
slices, while the dimension of the desired span $\mc U_S$ is at most $k$.

\begin{condition}[Deterministic condition for Step 1 (b)]
  Let $\wt Q_{U_S}\in\R^{n_2\times |\mc H|(|\mc H| + 1)/2}$ be a matrix whose $(j_1,j_2)$-th column
  for is equal to the projected two-dimensional slice $\vc(\prj_{S^{\perp}}\wt M_4(\mb e_{j_1},\mb
  e_{j_2},I,I))$, for $j_1\le j_2$ and $j_1,j_2\in \mc H$.
  If the matrix $\wt Q_{U_S}$ achieves its maximal column rank $k$, the desired span $\mc
  U_S$ defined in \eqref{eq:def-mcUS} is given by the column span of the matrix $\wt
  Q_{U_S}$.
\end{condition}


We show that this deterministic condition is satisfied by bounding the $k$-th singular
value of $\wt Q_{U_S}$ in the smoothed analysis setting:

\begin{lemma}[Correctness]
  \label{lem:1bsmooth}
  Given the exact 4-th order moments $\wt M_4$, with high probability, the $k$-th singular
  value of $\wt Q_{U_S}$ is at least $\Omega(\omega_o \rho^2 n^{1.5})$.
\end{lemma}

Similar to Lemma~\ref{lem:1asmooth}, the proof is based on writing the matrix $Q_{U_S}$ as
a product of three matrices, then bound their $k$-th singular values using random matrix
theory. The stability analysis also relies on the matrix perturbation theory.


%

\begin{lemma}[Stability]
  \label{lem:1bperturb}
  Given the empirical 4-th order moments $\wh M_4 = \wt M_4 + E_4$, assume that the absolute value
  of entries of $E_4$ are at most $\delta_2$.
  Also, given the output $\prj_{\wh S^\perp}$ from Step 1 (a), and assume that $\|\prj_{\wh
    S^\perp}- \prj_{\wt S^\perp}\|\le \delta_1$. When $\delta_1$ and $\delta_2$ are inverse
  polynomially small, we have $ \|\prj_{\wh U_S} - \prj_{\wt U_S}\| \le O\left({ n^{2.5} \left(
        \delta_2 + 2\delta_1 \right)/ \sigma_{k} ( \wt Q_{U_S}) }\right)$.
\end{lemma}

%

\paragraph{Step 1 (c). Merge $\mc U_1, \mc U_2$ to get the span of covariance matrices $\mc U$. }
Note that for a given index set $\mc H$, the span $\mc U_S$ obtained in Step 1 (b) only
gives partial information about the span of the covariance matrices.
The idea of  getting the span of the full covariance matrices is to obtain two sets of
such partial information and then merge them.

In order to achieve that, we repeat Step 1 (a) and Step 1 (b) for two {\em disjoint} sets
$\mc H_1$ and $\mc H_2$, each of size $\sqrt{n}$. The two subspace $S_1$ and $ S_2$ thus
correspond to the span of two disjoint sets of covariance matrix columns. Therefore, we
can hope that $U_1$ and $U_2$, the span of covariance matrices projected to $S_1^\perp$
and $S_2^\perp$ contain enough information to recover the full span $U$.

In particular, we prove the following claim:
%

%

\begin{condition}[Deterministic condition for Step 1 (c)]
  \label{lem:merge-simple}
  Let the columns of two (unknown) matrices $V_1\in\R^{n\times k}$ and $V_2\in\R^{n\times
    k}$ form two basis of the same $k$-dimensional (unknown) subspace $\mc
  U\subset\R^{n}$, and  let $U$ denote an arbitrary orthonormal basis of $\mc U$.
  Given two $s$-dimensional  subspaces $S_1$ and $ S_2$, denote
  $S_3 = S_1^\perp \cup S_2^\perp$.
  Given two projections of  $\mc U$ onto the two subspaces $S_1^\top$ and
  $S_2^\top$: $U_1 = \prj_{S_1^\perp}V_1$ and $U_2 = \prj_{S_2^{\perp}}V_2$.
  If $\sigma_{2s}([S_1, S_2])>0$ and $\sigma_{k}(\prj_{S_3} U)>0$, there is an algorithm
  for finding $\mc U$ robustly.
\end{condition}

The main idea in the proof is that since $s$ is not too large, the two subspaces
$S_1^\perp$ and $S_2^\perp$ have a large intersection.
Using this intersection we can ``align'' the two basis $V_1$ and $V_2$ and obtain $V_1^\dag V_2$,
and then it is easy to merge the two projections of the same matrix (instead of a subspace).

Moreover, we show that when applying this result to the projected span of covariance matrices, we
have $s= k|\mc H| \le n/3$, and the two deterministic conditions $\sigma_{2 s}([S_1, S_2])>0$ and
$\sigma_{k}(\prj_{S_3} V_1)>0$ are indeed satisfied with high probability over the parameter
perturbation.
The detailed smoothed analysis (Lemma~\ref{lem:bound-B3-Sigma} and
\ref{lem:bound-sig-S12}) and the stability analysis (Lemma~\ref{lem:merge-deterministic})
are provided in Section~\ref{sec:step-1c} in the appendix.

\paragraph{Step 2.  Unfold the moments to get $\wt X_4$ and $\wt X_6$.}

We show that given the span of covariance matrices $\mc U$ obtained from Step 1, finding the
unfolded moments $\wt X_4$, $\wt X_6$ is reduced to solving two systems of linear equations.

Recall that the challenge of recovering $\wt X_4$ and $\wt X_6$ is that the two linear mappings
$\mc F_4$ and $\mc F_6$ defined in \eqref{eq:F4F6-1} are {\em not linearly
  invertible}. The key idea of this step is to make use of the span $\mc U$ to {\em reduce
  the number of variables}.
%
Note that given the basis $U\in \R^{n_2\times k}$ of the span of the covariance matrices, we can
represent each vectorized covariance matrix as $\wt \Sigma^{(i)} = U\wt \sigma^{(i)}$. Now Let
$\wt Y_4\in\R^{k\times k}_{sym}$ and $\wt Y_4\in\R^{k\times k\times k}_{sym}$ denote the unfolded moments in
this new coordinate system: \vspace*{-0.1in}
\begin{align*}
  \wt Y_4 \defeq \sum_{i=1}^{k}\wt\omega_i \wt\sigma^{(i)}\ot^2, \quad \wt Y_6 = \sum_{i=1}^{k}\wt
  \omega_i \wt \sigma^{(i)}\ot^3.
\end{align*}
Note that once we know $\wt Y_4$ and $\wt Y_6$, the unfolded moments $\wt X_4$ and $\wt X_6$ are
given by $\wt X_4 = U \wt Y_4 U^\top$ and $\wt X_6 = \wt Y_6(U^\top, U^\top, U^\top)$.
Therefore, after changing the variable, we need to solve the two linear equation systems given in
\eqref{eq:F4F6-2} with the variables $\wt Y_4$ and $\wt Y_6$.

This change of variable significantly reduces the number of unknown variables.  Note that the number
of distinct entries in $\wt Y_4$ and $\wt Y_6$ are \qq{$k_2= {k+1\choose 2}$ and $k_3= {k+2\choose
    3}$}, respectively. Since $k_2\le n_4$ and $k_3\le n_6$, we can expect that the linear mapping
from $\wt Y_4$ to $\ol{\wt M}_4$ and the one from $\wt Y_6$ to $\ol {\wt M}_6$ are linearly
invertible.
This argument is formalized below.

\begin{condition}[Deterministic condition for Step 2]
  Rewrite the two systems of linear equations in \eqref{eq:F4F6-2} in their canonical form and
  let $\wt H_4 \in \R^{n_4\times k_2}$ and $\wt H_6\in\R^{n_6\times k_3}$ denote the coefficient
  matrices. We can obtain the unfolded moments $\wt X_4$ and $\wt X_6$ if the coefficient matrices
  have full column rank.

\end{condition}

We show with smoothed analysis that the smallest singular value of the two coefficient matrices are
lower bounded with high probability:

\begin{lemma}[Correctness]
\label{lem:2smooth}
With high probability over the parameter random perturbation, the $k_2$-th singular value of the
coefficient matrix $\wt H_4$ is at least $\Omega(\rho^2n/k)$, and the $k_3$-th singular value of the
coefficient matrix $\wt H_6$ is at least $\Omega(\rho^3(n/k)^{1.5})$.
\end{lemma}

To prove this lemma we rewrite the coefficient matrix as product of two matrices and bound their
smallest singular values separately. One of the two matrices corresponds to a projection of the
Kronecker product $\wt \Sigma\ot_{kr}\wt \Sigma$.  In the smoothed analysis setting, this matrix is
not necessarily incoherent. In order to provide a lower bound to its smallest singular value, we
further apply a carefully designed projection to it, and then we use the concentration bounds for
Gaussian chaoses to show that after the projection its columns are incoherent, finally we apply
Gershgorin's Theorem to bound the smallest singular value
 \footnote{Note that the idea of unfolding using system of linear
  equations also appeared in the work of \cite{jain2014learning}. However, in order to
  show the system of linear equations in their setup is robust, i.e., the coefficient matrix has
  full rank, they heavily rely on the {\em incoherence} assumption, which we do not impose in the
  smoothed analysis setting. }.


When implementing this step with the empirical moments, we solve two least squares problems instead
of solving the system of linear equations. Again using results in matrix perturbation theory and
using the lower bound of the smallest singular values of the two coefficient matrices, we show the
stability of the solution to the least squares problems:
\begin{lemma}[Stability]
\label{lem:2perturb}
Given the empirical moments $\wh M_4 = \wt M_4 + E_4$, $\wh M_6 = \wt M_6 + E_6$, and suppose that
the absolute value of entries of $E_4$ and $E_6$ are at most $\delta_1$.  Let $\wh U$, the output of
Step 1, be the estimation for the span of the covariance matrices, and suppose that $\|\wh U - \wt
U\|\le \delta_2$.
Let $\wh Y_4$ and $\wh Y_6$ be the least squares solution respectively.  When $\delta_1$ and
$\delta_2$ are inverse polynomially small, we have $\|\wt Y_4 - \wh Y_4 \|_F \le
O(\sqrt{n_4}(\delta_1+\delta_2/\sigma_{min}(\wt{H}_4))$ and $\|\wt Y_6 - \wh Y_6 \|_F \le
O(\sqrt{n_6}(\delta_1+\delta_2/\sigma_{min}(\wt{H}_6))$.
\end{lemma}

\paragraph{Step 3. Tensor Decomposition.}

\begin{claim}
  Given $\wt Y_4$, $\wt Y_6$ and $\wt U$, the symmetric tensor decomposition algorithm can correctly
  and robustly find the mixing weights $\wt \omega_i$'s and the vectors $\wt \sigma_i$'s, up to some
  unknown permutation over $[k]$, with high probability over both the randomized algorithm and the
  parameter perturbation.

\end{claim}

The algorithm and its analysis mostly follow the algorithm of symmetric tensor decomposition in
\cite{anandkumar2012tensor}, and the details are provided in Section~\ref{sec:step-3-zero} in the
appendix.

\paragraph{Proof Sketch for the Main Theorem of Zero-mean Case. }

Theorem~\ref{thm:main-zero-mean} follows from the previous smoothed analysis and stability analysis
lemmas for each step.

First, exploiting the randomness of parameter perturbation, the smoothed analysis lemmas show that
the deterministic conditions, which guarantee the correctness of each step, are satisfied with high
probability.
%
Then using concentration bounds of Gaussian variables, we show that with high probability over the
random samples, the empirical moments $\wh M_4$ and $\wh M_6$ are entrywise $\delta$-close to the
exact moments $\wt M_4$ and $\wt M_6$. In order to achieve $\epsilon$ accuracy in the parameter
estimation, we choose $\delta$ to be inverse polynomially small, and therefore the number of samples
required will be polynomial in the relevant parameters.
The stability lemmas show how the errors propagate only ``polynomially'' through the steps of the
algorithm, which is visualized in Figure~\ref{fig:flow}.

A more detailed illustration is provided in Section~\ref{sec:mainproofzero} in the appendix.

\begin{figure}
\centering
\includegraphics[width =6in]{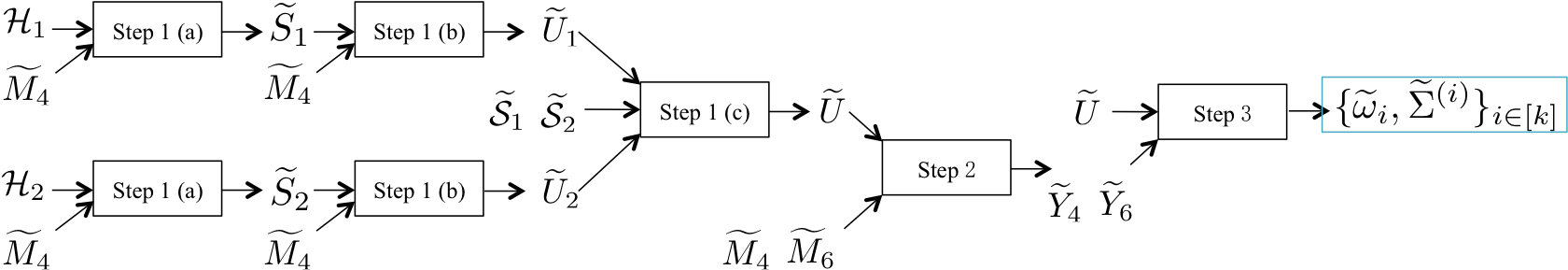}
\caption{Flow of the algorithm for learning mixture of  zero-mean Gaussians.}\label{fig:flow}
\end{figure}

\section{Algorithm Outline for Learning Mixture of General Gaussians}
\label{sec:general}

\begin{figure}
\centering
\includegraphics[width =7in]{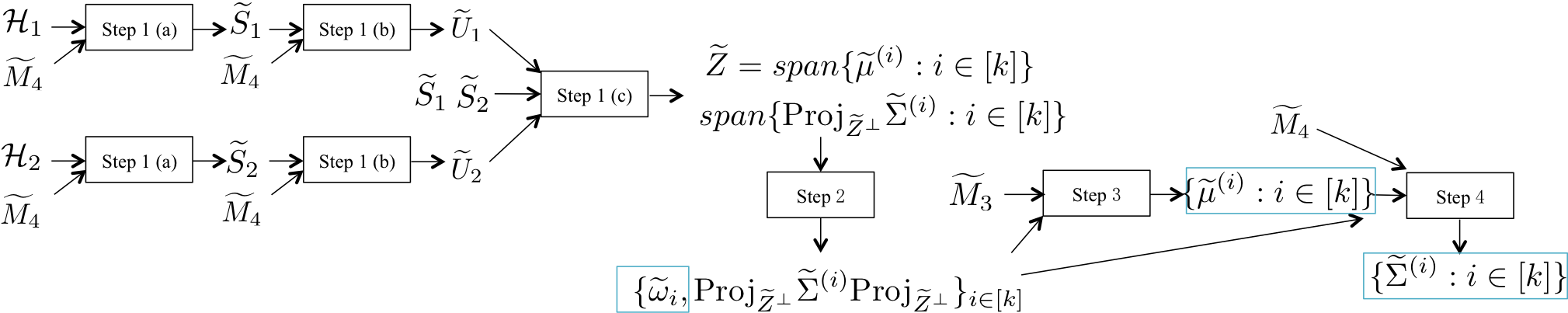}
\caption{Flow of the algorithm for learning mixtures of general Gaussians.}\label{fig:flow-general-maintext}
\end{figure}

In this section, we briefly discuss the algorithm for learning mixture of {\em general } Gaussians.
Figure~\ref{fig:flow-general-maintext} shows the inputs and outputs of each step in this algorithm.
Many steps share similar ideas to those of the algorithm for the zero-mean case in previous
sections. We only highlight the basic ideas and defer the details to Section~\ref{sec:general-case}
in the appendix.

\paragraph{Step 1. Find $\wt Z = span \{\wt \mu^{(i)}:i\in[k]\}$ and $\wt \Sigma_o=span\{\prj_{\wt
    Z^\perp} \wt \Sigma^{(i)}\prj_{\wt Z^\perp}:i\in[k]\}$. }

Similar to Step 1 in the zero-mean case, this step makes use of the structure of the 4-th order
moments $\wt M_4$, and is achieved in three small steps:
%
%

\begin{enumerate}[label={\tb{(\alph*)}}, leftmargin=*,itemsep=0pt]
\item 
  For a \qq{subset $\mc H \subset [n]$ of size $|\mc H| =\sqrt{n}$},
  find the span:
  \begin{align}
    \label{eq:def-mcS-general}
    \mc S = \tx{span}\lt\{\wt \mu^{(i)},\wt \Sigma^{(i)}_{[:,j]}: i\in[k],j\in \mc H\rt\} \subset
    \R^n.
  \end{align}

\item Find the span of the covariance matrices with the columns projected onto $\mc S^{\perp}$,
  namely,
  \begin{align}
    \label{eq:def-mcUS-general}
    \mc U_S =\tx{span}\lt\{\vc(\prj_{S^{\perp}}\wt \Sigma^{(i)}) : i\in[k]\rt\} \subset \R^{n^2}.
  \end{align}

\item For disjoint subsets $\mc H_1$ and $\mc H_2$, repeat Step 1 (a) and Step 1 (b) to obtain $\mc
  U_1$ and $\mc U_2$, the span of the covariance matrices projected onto the subspaces $\mc
  S_1^{\perp}$ and $\mc S_2^{\perp}$.  The intersection of the two subspaces $\mc U_1$ and $\mc U_2$
  gives the span of the mean vectors $\wt Z = \tx{span}\lt\{\wt \mu^{(i)}, i \in [k]\rt\}$.  Merge
  the two subspaces $\mc U_1$ and $\mc U_2$ to obtain the span of the covariance matrices projected
  to the subspace orthogonal to $\wt Z$, namely $\wt \Sigma_o =\tx{span}\lt\{\prj_{\wt Z^\perp}\wt
  \Sigma^{(i)}\prj_{\wt Z^\perp}:i\in[k]\rt\}$.
\end{enumerate}

\paragraph{Step 2. {Find the Covariance Matrices in the Subspace $\wt Z^{\perp}$} and the Mixing
  Weights $\wt \omega_i$'s.}
The key observation of this step is that when the samples are projected to the subspace orthogonal
to all the mean vectors, they are equivalent to samples from a mixture of zero-mean Gaussians with
covariance matrices $\wt \Sigma^{(i)}_o = \prj_{\wt Z^\perp} \wt \Sigma^{(i)}\prj_{\wt Z^\perp}$ and
with the same mixing weights $\wt \omega_i$'s.
Therefore, projecting the samples to $\wt Z^\perp$, the subspace orthogonal to the mean vectors, and
use the algorithm for the zero-mean case, we can obtain $\wt \Sigma^{(i)}_o$'s, the covariance
matrices projected to this subspace, as well as the mixing weights $\wt\omega_i$'s.

\paragraph{Step 3. {Find the means}}
With simple algebra, this step extracts the projected covariance matrices $\wt \Sigma_o^{(i)}$'s
from the $3$-rd order moments $\wt M_3$, the mixing weights $\wt \omega_i$ and the projected
covariance matrices $\wt \Sigma_o^{(i)}$'s obtained in Step 2.
%
%

\paragraph{Step 4. {Find the full covariance matrices}}
In Step 2, we obtained $\wt \Sigma^{(i)}_o$, the covariance matrices projected to the subspace
orthogonal to all the means. Note that they are equal to matrices $(\wt\Sigma^{(i)}+\wt
\mu^{(i)}(\wt\mu^{(i)})^\top)$ projected to the same subspace.
We claim that if we can find the span of these matrices ($(\wt\Sigma^{(i)}+\wt
\mu^{(i)}(\wt\mu^{(i)})^\top)$'s), we can get each matrix $(\wt\Sigma^{(i)}+\wt
\mu^{(i)}(\wt\mu^{(i)})^\top)$, and then subtracting the known rank-one component to find the
covariance matrix $\wt\Sigma^{(i)}$.
This is similar to the idea of merging two projections of the same subspace in Step 1 (c) for the
zero-mean case.

The idea of finding the desired span is to construct a $4$-th order tensor:
\begin{align*}
  \wt M'_4 = \wt M_4+2\sum_{i=1}^k \wt \omega_i(\wt \mu^{(i)}\otimes^{ 4}),
\end{align*}
which corresponds to the 4-th order moments of a mixture of zero-mean Gaussians with covariance
matrices $\wt \Sigma^{(i)}+\wt \mu^{(i)}(\wt \mu^{(i)})^\top$ and the same mixing weights
$\wt\omega_i$'s. Then we can then use Step 1 of the algorithm for the zero-mean case to obtain the
span of the new covariance matrices, i.e.  $span\{\wt\Sigma^{(i)}+\wt
\mu^{(i)}(\wt\mu^{(i)})^\top:i\in[k]\}$.

\section{Conclusion}

In this paper we give the first efficient algorithm for learning mixture of general Gaussians in the
smoothed analysis setting. In the algorithm we developed new ways of extracting information from
lower-order moment structure. This suggests that although the method of moments often involves
solving systems of polynomial equations that are intractable in general, for natural models there is
still hope of utilizing their special structure to obtain algebraic solution.

Smoothed analysis is a very useful way of avoiding degenerate examples in analyzing algorithms. In
the analysis, we proved several new results for bounding the smallest singular values of {\em
  structured} random matrices. We believe the lemmas and techniques can be useful in more general
settings.

Our algorithm uses only up to $6$-th order moments. We conjecture that using higher order moments
can reduce the number of dimension required to $n\ge \Omega(k^{1+\epsilon})$, or maybe even $n\ge
\Omega(k^{\epsilon})$.

\newpage
\subsection*{Acknowledgements}
We thank Santosh Vempala for many insights and for help in earlier
attempts at solving this problem.



\appendix

\newpage

\section{Moment Structures}
\label{sec:app:moment}

In this section we characterize the structure of the 3-rd, 4-th and 6-th moments of Gaussians mixtures.

As described in Section~\ref{subsec:moment-structure-mog}, the $m$-th order moments of the Gaussian mixture model are given by the following $m$-th order
symmetric tensor $M\in\R^{n\times\dots\times n}_{sym}$:
\begin{align*}
  \lt[ M_m \rt]_{j_1,\dots, j_m}:= \mbb E\lt[x_{j_1}\dots x_{j_m}\rt] = \sum_{i=1}^{k}\omega_i  \mbb
  E\lt[y^{(i)}_{j_1}\dots y^{(i)}_{j_m}\rt], \quad \fa j_1,\dots,j_m\in[n],
\end{align*}
where $ y^{(i)}$ corresponds to the $n$-dimensional Gaussian distribution $\mc N(\mu^{(i)}, \Sigma^{(i)})$.

Gaussian distribution is a highly symmetric distribution, and in the zero-mean case the higher moments are well-understood by Isserlis' Theorem:

\begin{theorem}[Isserlis]
\label{prop:isserlis}
Let $\mb y=(y_1,\dots, y_{2t})$ be a multivariate Gaussian random vector with mean zero and
covariance $\Sigma$, then
  \begin{align*}
    &\mbb E[y_1\dots y_{2t}] = \sum \prod \Sigma_{u,v},
    \\
    & \mbb E[y_{1}\dots y_{2t-1}] = 0,
  \end{align*}
  where the summation is taken over all distinct ways of partitioning $y_1,\dots, y_{2t}$ into $t$
  pairs, which correspond to all the perfect matchings in a complete graph. Thus there are $(2t-1)!!$
  terms in the sum, and each summand is a product of $t$ terms.
\end{theorem}

The non-zero mean case is a direct corollary using Isserlis' Theorem and linearity of expectation.

\begin{corollary}
\label{prop:isserlisnonzero}
Let $\mb y=(y_1,\dots, y_{t})$ be a multivariate Gaussian random vector with mean $\mu$ and
covariance $\Sigma$, then
$$
    \mbb E[y_1\dots y_{t}] = \sum \prod \Sigma_{u,v}\prod \mu_w.
$$
where the summation is taken over all distinct ways of partitioning $y_1,\dots, y_{t}$ into $p$ pairs of $(u,v)$ and $s$ singletons of $(w)$, where
$p\ge 0$, $s\ge 0$ and $2p+s = t$.
\end{corollary}
As an example, $\E[y_1y_2y_3] = \mu_1\mu_2\mu_3 +
\mu_1\Sigma_{2,3}+\mu_2\Sigma_{1,3}+\mu_3\Sigma_{1,2}$.

\subsection{Proof of Lemma~\ref{prop:two-linear-mapping}}
We shall first prove Lemma~\ref{prop:two-linear-mapping} in
Section~\ref{subsec:moment-structure-mog}. Recall that this lemma shows that for mixture
of zero-mean Gaussians, the 4-th moments $\ol M_4$ and the 6-th moments $\ol M_6$ with
distinct indices can be viewed as a linear projection of
the unfolded moment $X_4$ and $X_6$ defined in \eqref{eq:def-X4-X6}.

\begin{proof}
  (of Lemma~\ref{prop:two-linear-mapping})

  By Isserlis Theorem~\ref{prop:isserlis}, the mapping $\sqrt{3}\mc F_4$ is characterized by: ($\fa 1\le j_1< j_2<
  j_3< j_4\le n$)
  \begin{align*}
    [M_{4}]_{j_1,j_2,j_3,j_4} &= \sum_{i=1}^{k} \omega_i(
    \Sigma_{j_1,j_2}^{(i)}\Sigma_{j_3,j_4}^{(i)} + \Sigma_{j_1,j_3}^{(i)}\Sigma_{j_2,j_4}^{(i)} +
    \Sigma_{j_1,j_4}^{(i)}\Sigma_{j_2,j_3}^{(i)})
    \\
    &= [X_{4}]_{(j_1,j_2), (j_3,j_4)} + [X_{4}]_{(j_1,j_3), (j_2,j_4)} + [X_{4}]_{(j_1,j_4),
      (j_2,j_3)}.
  \end{align*}
  Therefore, with the normalization constant $\sqrt{3}$, the $(j_1,j_2,j_3,j_4)$-th
  mapping of $\mc F_4$ is a projection of the three elements in $X_4$.
  Similarly, we have for $\sqrt{15}\mc F_6$: ($\fa 1\le j_1< j_2<\dots< j_6\le n$)
  \begin{align*}
    &[M_{6}]_{j_1,j_2,j_3,j_4,j_5,j_6}
    \\
    =& [X_{6}]_{(j_1,j_2), (j_3,j_4), (j_5,j_6)} + [X_{6}]_{ (j_1,j_3, (j_2,j_4), (j_5,j_6)} +
    [X_{6}]_{ (j_1,j_4), (j_2,j_3), (j_5,j_6)} + [X_{6}]_{ (j_1,j_5), (j_2,j_3), (j_4,j_6)}
    \\
    &+ [X_{6}]_{ (j_1,j_2), (j_5,j_3), (j_4,j_6)} + [X_{6}]_{ (j_1,j_3), (j_2,j_5), (j_4,j_6)} +
    [X_{6}]_{ (j_1,j_2), (j_4,j_5), (j_3,j_6)} + [X_{6}]_{ (j_1,j_4), (j_2,j_5), (j_3,j_6)}
    \\
    &+ [X_{6}]_{ (j_1,j_5), (j_2,j_4), (j_3,j_6)} + [X_{6}]_{ (j_1,j_3), (j_4,j_5), (j_2,j_6)} +
    [X_{6}]_{ (j_1,j_4), (j_3,j_5), (j_2,j_6)} + [X_{6}]_{ (j_1,j_5), (j_3,j_2), (j_2,j_6)}
    \\
    &+ [X_{6}]_{ (j_2,j_3), (j_4,j_5), (j_1,j_6)} + [X_{6}]_{ (j_2,j_4), (j_3,j_5), (j_1,j_6)} +
    [X_{6}]_{ (j_2,j_5), (j_3,j_4), (j_1,j_6)}.
  \end{align*}
  Thus with the normalization constant $\sqrt{15}$, the mapping $\mc F_6$ is a linear projection.
\end{proof}

\subsection{Slices of Moments}

Next we shall characterize the slices of the moments of mixture of Gaussians.

For mixture of zero-mean Gaussians, a one-dimensional slice of the 4th moment tensor is a
vector in the span of corresponding columns of the covariance matrices:

\begin{claim}
[Claim~\ref{claim:M4-proj4} restated]
For a mixture of zero-mean Gaussians, the one-dimensional slices of the 4-th moments $M_4$
are given by:
  \begin{align*}
      M_4(e_{j_1},e_{j_2},e_{j_3},I) = \sum_{i=1}^k   \omega_i\left(  \Sigma^{(i)}_{j_1,j_2}
        \Sigma^{(i)}_{[:,j_3]}+  \Sigma^{(i)}_{j_1,j_3}
        \Sigma^{(i)}_{[:,j_2]}+  \Sigma^{(i)}_{j_2,j_3}   \Sigma^{(i)}_{[:,j_1]}\right),
     \quad\fa j_1,j_2,j_3\in [n].
  \end{align*}
\end{claim}

\begin{proof}
By the definition of multilinear map, $  M_4(e_{j_1},e_{j_2},e_{j_3},I)$ is a vector
whose $p$-th entry is equal to $  M_4(e_{j_1},e_{j_2},e_{j_3},e_p)$. We can compute this
entry by Isserlis' Theorem:
\begin{align*}
    M_4(e_{j_1},e_{j_2},e_{j_3},e_p) = \sum_{i=1}^k   \omega_i\left(  \Sigma^{(i)}_{j_1,j_2}
        \Sigma^{(i)}_{[p,j_3]}+  \Sigma^{(i)}_{j_1,j_3}
        \Sigma^{(i)}_{[p,j_2]}+  \Sigma^{(i)}_{j_2,j_3}   \Sigma^{(i)}_{[p,j_1]}\right),
\end{align*}
this directly implies the claim.
\end{proof}

For mixture of zero-mean Gaussians, a two-dimensional slice of the 4th moment $M_4$ is a
matrix, and it is a linear combination of the covariance matrices with some additive rank
one matrices:

\begin{claim}
  [Claim~\ref{claim:M4-proj34} restated] For a mixture of zero-mean Gaussians, the
  two-dimensional slices of the 4-th moment $ M_4$ are given by:
  \begin{align*}
      M_4(e_{j_1},e_{j_2},I,I) = \sum_{i=1}^k  \omega_i \left(  \Sigma^{(i)}_{j_1,j_2}
        \Sigma^{(i)}+  \Sigma^{(i)}_{[:,j_1]} (  \Sigma^{(i)}_{[:,j_2]})^\top+  \Sigma^{(i)}_{[:,j_2]}
      (  \Sigma^{(i)}_{[:,j_1]})^\top\right),
    \quad\fa j_1,j_2\in [n].
  \end{align*}
\end{claim}

\begin{proof}
Again this follows from Isserlis' theorem. By definition of multilinear map this is a matrix whose $(p,q)$-th entry is equal to
$$  M_4(e_{j_1},e_{j_2},e_p,e_q) = \sum_{i=1}^k   \omega_i\left(  \Sigma^{(i)}_{j_1,j_2}
        \Sigma^{(i)}_{[p,q]}+  \Sigma^{(i)}_{j_1,p}
        \Sigma^{(i)}_{[q,j_2]}+  \Sigma^{(i)}_{j_2,p}   \Sigma^{(i)}_{[q,j_1]}\right),$$
     and this directly implies the claim.

\end{proof}


Similarly, for mixture of general Gaussians, we prove the following claims:
\begin{claim} [Claim~\ref{claim:M4-mean} restated]
  For a mixture of general Gaussians, the $ (j_1,j_2,j_3)$-th one-dimensional slice of
  $M_4$ is given by:
  \begin{align*}
    M_4(e_{j_1},e_{j_2},e_{j_3}, I) =
    \sum_{i=1}^{n} \omega_i\Big(
    \mu_{j_1}^{(i)}\mu_{j_2}^{(i)}\mu_{j_3}^{(i)}\mu^{(i)} +
     \sum_{\pi\in\lt\{\substack{(j_1,j_2,j_3),\\ (j_2,j_3,j_1),\\
        (j_3,j_1,j_2)}\rt\}}
\lt(    \Sigma_{\pi_1,\pi_2}^{(i)}
    \Sigma_{[:,\pi_3]}^{(i)} + \mu_{\pi_1}^{(i)}\mu_{\pi_2}^{(i)}\Sigma_{[:,\pi_3]}^{(i)} +
    \Sigma_{\pi_1,\pi_2}^{(i)} \mu_{\pi_3}^{(i)}\mu^{(i) }
    \rt)
    \Big).
   \end{align*}
\end{claim}

\begin{proof}
  This is very similar to Claim~\ref{claim:M4-proj4} and follows from the corollary of
  Isserlis's theorem (Corollary~\ref{prop:isserlisnonzero}). There are 10 ways to
  partition the indices $\{j_1,j_2,j_3, j_4\}$ into pairs and singletons:
  $((j_1),(j_2),(j_3),(j_4))$, $((j_1,j_2), (j_3), (j_4))$, $((j_1,j_3), (j_2), (j_4))$,
  $((j_1,j_4), (j_2), (j_3))$, $((j_2,j_3), (j_1), (j_4))$, $((j_2,j_4), (j_1), (j_3))$,
  $((j_3,j_4), (j_1), (j_2))$, $((j_1,j_2), (j_3,j_4))$, $((j_1,j_3), (j_2,j_4))$,
  $((j_1,j_4), (j_2,j_3))$. From this enumeration, we can specify each element in the
  vector of the one-dimensional slice.
\end{proof}

\begin{claim}[Claim~\ref{claim:M3-proj} restated]
  For a mixture of general Gaussians, let the matrix $M_{3(1)}\in\R^{n\times n^2}$ be the
  matricization of $M_3$ along the first dimension.  The $j$-th row of $M_{3(1)}$ is
  given by:
  \begin{align*}
    [ M_{3(1)} ]_{[j,:]}
    &= \sum_{i=1}^{k} \omega_i\lt( \mu_j^{(i)} \vc(\Sigma^{(i)}) + \mu_j^{(i)} \mu^{(i)} \od
    \mu^{(i)} + \Sigma^{(i)}_{[:,j]}\od \mu^{(i)} + \mu^{(i)} \od \Sigma^{(i)}_{[:,j]} \rt)^\top.
  \end{align*}
\end{claim}

\begin{proof}
  Note that $ [ M_{3(1)} ]_{[j,:]} = \Big[\vc(\mbb E[x_{j} x x^\top])\Big] = \vc(\mbb
  E[x_j x\od x])$.  Again following the corollary of Isserlis's theorem
  (Corollary~\ref{prop:isserlisnonzero}, there are 4 ways to partition the indices $\{j_1,
  j_2, j_3\}$  into pairs and singletons: $((j_1), (j_2,j_3))$, $((j_1), (j_2), (j_3))$, $((j_1,j_2),(j_3))$, $((j_2),(j_1,j_3)$, and they
  correspond to the four terms in the summation.)
\end{proof}




\subsection{Two mixtures with same $M_4$ but different $X_4$}
\label{sec:two-mixtures-with}

Since $M_4$ gives linear observations on the symmetric low rank matrix $X_4$, it is natural to wonder whether we can use matrix completion techniques to recover $X_4$ from $M_4$. Here we show this is impossible by giving a counter example: there are two mixture of Gaussians that generates the same 4th moment $M_4$, but has different $X_4$ (even the span of $\Sigma^{(i)}$'s are different).

By $((a,b),(c,d))$ we denote a $5\times 5$ matrix $A$ which has $2$'s on diagonals, and the only nonzero off-diagonal entries are $A_{a,b} = A_{b,a} = A_{c,d} = A_{d,c} = 1$. For example, $((1,2), (4,5))$ will be the following matrix:
$$
\left(\begin{array}{ccccc} 2&1& & & \\1&2& & &  \\ & & 2 & & \\ & & & 2 & 1\\ & & & 1 & 2\end{array}\right),
$$
where all the missing entries are 0's. Now we construct two mixtures of 3 Gaussians, all with mean 0 and weight $1/3$. The covariance matrices are $((1,2),(4,5)),((1,3),(2,5)),((1,4),(3,5))$ for the first mixture and $((1,2),(3,5)), ((1,3),(4,5)), ((1,4),(2,5))$ for the second mixture. These are clearly different mixtures with different span of $\Sigma^{(i)}$'s: in the first mixture, $\Sigma^{(i)}_{1,2} = \Sigma^{(i)}_{4,5}$ for all matrices, but this is not true for the second mixture.

These two mixture of Gaussians have the same 4th moment $M_4$. This can be checked by using Isserlis' theorem to compute the moments. Intuitively, this is true because all the pairs $(1,i)$ and $(i,5)$ appeared exactly twice in the covariance matrices for both mixtures; also, every 4-tuple $(1,i,j,5)$ appeared exactly once in the covariance matrices for both mixtures.

\section{Step 1: Span Finding}

\label{sec:step-1}
Recall that in Step 1 of the algorithm for learning mixture of zero-mean Gaussians, we find the span
of the covariance matrices in three small steps.  In this section, we prove the
correctness and the robustness of each step with smoothed analysis.

For completeness we restate each substep  and highlight the  key properties we need,
followed by the detailed proofs.

\subsection{Step 1(a). Finding $\mc S$, the span of a subset of columns of $\wt\Sigma^{(i)}$'s.}
\label{sec:step-1a}

\begin{algorithm}[h]
  \caption{FindColumnSpan}
  \label{alg:columnspan}
  \DontPrintSemicolon

  \textbf{Input:}  4-th order moments $M_4$, set of indices $\mc H$.

  \tb{Output:} $ span\{\Sigma^{(i)}_j:i\in[k], j\in \mc H\}$, represented by an
  orthonormal matrix $S\in \R^{n\times |\mc H|k}$.

  \BlankLine


  \STATE Let $Q$ be a matrix  of dimension $n\times |\mc H|^3$ whose columns are all
  of $M_4(e_{i_1},e_{i_2},e_{i_3},I)$,
   for $i_1,i_2,i_3\in\mc H$.

  \STATE  Compute the SVD of $Q$: $Q = UDV^\top$.

  \BlankLine

  \tb{Return:} The first $k|\mc H|$ left singular vectors $S = [U_{[:,1]},\dots, U_{[:,k|\mc H|]}]$.

\end{algorithm}

In Step 1 (a), for any set $\mc H$ of size \qq{$\sqrt{n}$}, we want to show that the
one-dimensional slices of $M_4$ span the entire subspace $\mc S = \tx{span}\lt\{\wt
\Sigma^{(i)}_{[:,j]}: i\in[k],j\in \mc H\rt\}$, which is the span of a subset of the
columns in the covariance matrices.

Recall that in Claim~\ref{claim:M4-proj4} we showed:
  \begin{align*}
    \wt M_4(e_{j_1},e_{j_2},e_{j_3},I) = \sum_{i=1}^k\wt  \omega_i\left(\wt \Sigma^{(i)}_{j_1,j_2}
      \wt \Sigma^{(i)}_{[:,j_3]}+\wt \Sigma^{(i)}_{j_1,j_3}
     \wt  \Sigma^{(i)}_{[:,j_2]}+\wt \Sigma^{(i)}_{j_2,j_3} \wt \Sigma^{(i)}_{[:,j_1]}\right),
     \quad\fa j_1,j_2,j_3\in [n].
  \end{align*}

This in particular means when $j_1,j_2,j_3\in \mc H$, the vector $\wt
M_4(e_{j_1},e_{j_2},e_{j_3},I)$ is in $\mc S$. We need to show that the columns of the
matrix $Q$  indeed span the entire subspace $\mc S$.

It is sufficient to show that a subset of the column span the entire subspace. Form a
three-way even partition of the set $\mc H$, i.e., $|\mc H^{(1)}|= |\mc H^{(2)}|=|\mc
H^{(3)}|=|\mc H|/3 = \qq{\sqrt{n}/3}$, and only consider the one-dimensional slices of
$\wt M_4$ corresponding to the indices $j_i\in\mc H^{(i)}$ for $i = 1,2,3$.
In particular, we define matrix $\wt Q_S$ with these one-dimensional slices of $\wt M_4$:
\begin{align}
  \label{eq:Qs-def}
  \wt Q_S = \left[
   \big[ [\wt M_4(e_{j_1},e_{j_2},e_{j_3},I) :j_3\in\mc H^{(3)}]: j_2\in\mc H^{(2)} \big] : j_1\in\mc H^{(1)}
  \right]
  \in\mbb  R^{n\times (|\mc H|/3)^3}.
\end{align}
Define matrix $\wt P_S$ with the corresponding columns of the covariance matrices, forming a basis (although not orthogonal) of the desired
subspace $\mc S$:
\begin{align}
  \label{eq:Ps-def}
  \wt P_S = \left[ \big[[\wt\Sigma_{[:,j]}^{(i)}:i\in[k]]\ : j\in\mc H^{(l)} \big]: l = 1,2,3 \right]
  = \lt[ \wt\Sigma_{[:,\mc H^{(1)}]},\wt\Sigma_{[:,\mc H^{(2)}]}, \wt\Sigma_{[:,\mc H^{(3)}]} \rt]\in
  \mbb R^{n\times k|\mc H|}.
\end{align}

In the following two lemmas, we show that with high probability over the random perturbation, the column span of $\wt Q_S$ is
exactly equal to the column span of $\wt P_S$, and robustly so.

\begin{lemma}[Lemma~\ref{lem:1asmooth} restated]
  \label{prop:find-S}
  Given $\wt M_4$, the exact 4-th order moment of the $\rho$-smooth mixture of zero-mean
  Gaussians, for any subset $\mc H\in[n]$ with cardinality $|\mc H| = \qq{\sqrt{n}}$, let $\wt
  Q_S$ be the matrix defined as in \eqref{eq:Qs-def} with the one-dimensional slices of
  $\wt M_4$.
  For any $\epsilon>0$, and for some absolute constant $C_1, C_2, C_3 >0$, with probability at least
  $1-(C_1\epsilon)^{C_2 n}$, the $k|\mc H|$-th singular value of $\wt Q_S$ is bounded below by:
  \begin{align}
    \label{eq:find-S}
    \sigma_{k|\mc H|} (\wt Q_S) \ge C_3\omega_o\epsilon^2\rho^2 n.
  \end{align}
\end{lemma}

In order to prove this lemma, we first  write $\wt{Q}_S$ as the product of three matrices.

\begin{figure}
\centering
\includegraphics[height=3.5in]{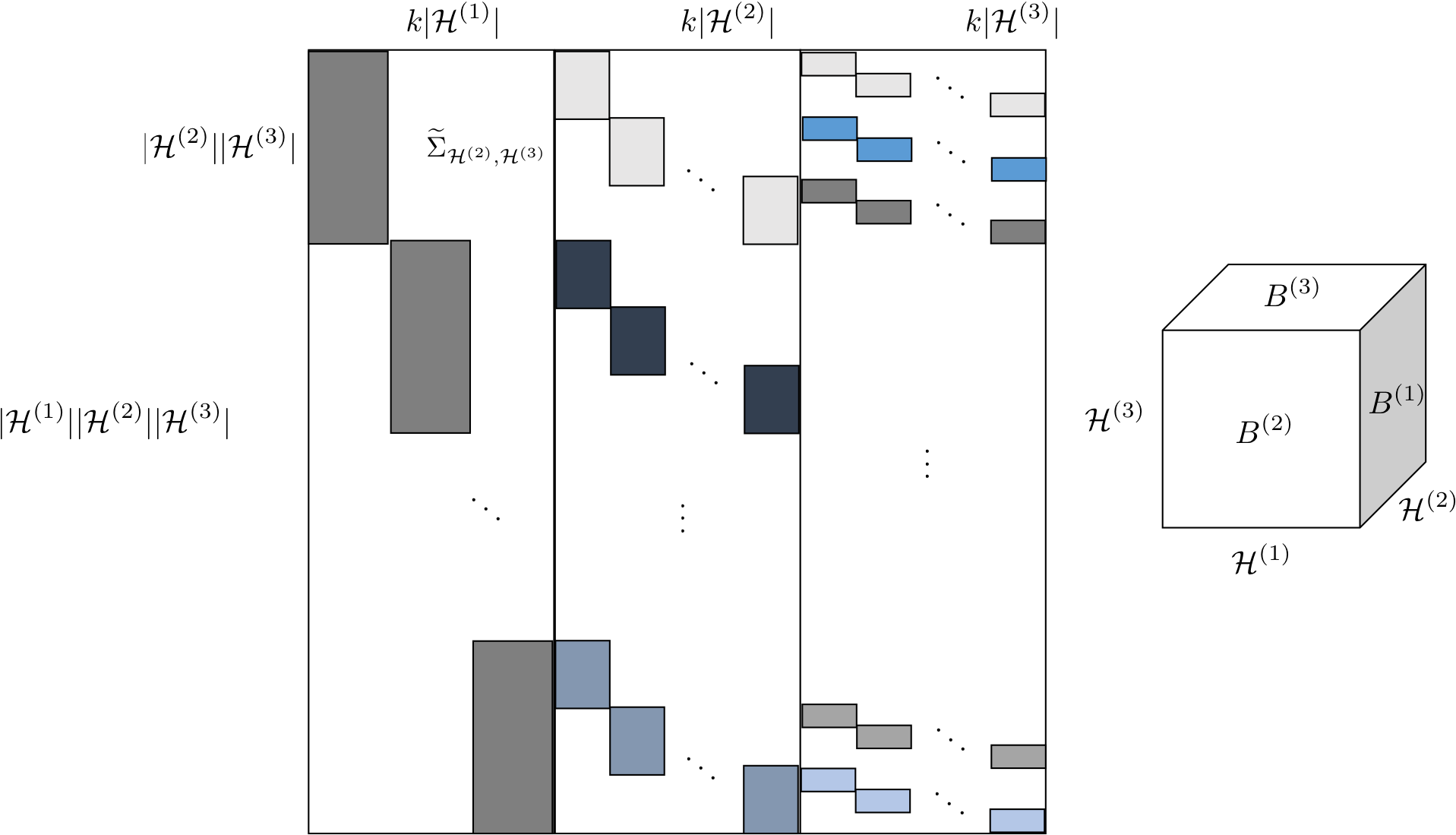}
\caption{Structure of the  matrix $B_S$}
\label{fig:bstruct}
\end{figure}

\begin{claim}[Structural]\label{claim:1astructure} Under the same assumptions of
  Lemma~\ref{prop:find-S}, the matrix $\wt Q_S$ can be written as
\begin{align}
  \label{eq:QsPsBs}
  \wt Q_S = \wt P_S \lt(D_{\wt \omega}\ot_{kr}I_{|\mc H|}\rt) (\wt B_S)^\top,
\end{align}
where $\wt P_S\in \R^{n\times k|\mc H|}$ as defined in Equation (\ref{eq:Ps-def} has
columns equal to the columns in $\wt \Sigma^{(i)}_{[:,\mc H]}$; the diagonal matrix in the
middle is the Kronecker product of two diagonal matrices and depends only on the mixing
weights $\wt \omega_i$'s.
\end{claim}
With the observation that the columns of $\wt P_S$ form a basis of the subspace $\mc S$,
and each column of $\wt Q_S$ is a linear combination of the columns in $\wt P_S$, the
rows  of $\wt B_S\in\R^{ (|\mc H|/3)^3 \times k|\mc H|}$ can be viewed as the
coefficients for the linear combinations, and has some special structures.
In particular, it consists of three blocks: $\wt B_S = \lt[ \wt B^{(1)}, \wt B^{(2)}, \wt
B^{(3)} \rt]$. The first tall matrix $\wt B^{(1)} \in\R^{(|\mc H|/3)^3 \times k(|\mc
  H|/3)}$, corresponding to the coefficient of the linear combinations on the subset of
basis $\wt \Sigma_{[:,\mc H^{(1)}]}$. By the indexing order of the columns in $\wt Q_S$,
the matrix $\wt B^{(1)}$ is block diagonal with identical blocks equal to $\wt \Sigma_{\mc
  H^{(2)}, \mc H^{(3)}}$, defined as follows:
\begin{align*} \wt \Sigma_{\mc H^{(2)},\mc
    H^{(3)}} = \lt[[\wt \Sigma^{(i)}_{j_1,j_2}: j_1\in\mc H^{(2)}, j_2\in\mc
  H^{(3)}]^\top: i\in[k] \rt] \in\R^{(|\mc H|/3)^2\times k}.
\end{align*}
With some fixed and known row permutation $\pi^{(2)}$ and $\pi^{(3)}$, the matrix $\wt
B^{(2)}$ and $\wt B^{(3)}$ can be made block diagonal with identical blocks equal to
$\wt\Sigma_{\mc H^{(1)}, \mc H^{(3)}}$ and $\wt \Sigma_{\mc H^{(1)}, \mc H^{(2)}}$,
respectively. Note that the three parts $\wt B^{(1)}, \wt B^{(2)}, \wt B^{(3)}$ do not
have any common entry, nor do they involve any diagonal entry of the covariance matrices,
therefore the three parts are independent when the covariances are randomly perturbed in
the smoothed analysis.

It is easier to understand the structure by picture, see Figure~\ref{fig:bstruct}. The
rows of the matrix should be indexed by $(j_1,j_2,j_3)\in\mc H^{(1)}\times \mc H^{(2)}
\times \mc H^{(3)}$, which can also be interpreted as a cube (in the right). The block structure in the first part $\wt B^{(1)}$ correspond to a slice in $\mc H^{(2)}\times \mc H^{(3)}$ direction (for each block, the element in $\mc H^{(1)}$ is fixed, the elements in $\mc H^{(2)}$ and $\mc H^{(3)}$ take all possible values). Similarly for $\wt B^{(2)}$ and $\wt B^{(3)}$ (as shown in figure).

\begin{proof} (of Claim~\ref{claim:1astructure} ) The proof of this claim is using
  Claim~\ref{claim:M4-proj4}, the definition of matrices and the rule of matrix
  multiplication. Consider the column in $\wt Q_S$ corresponding to the index
  $(j_1,j_2,j_3)$ for $j_1\in \mc H^{(1)}, j_2\in \mc H^{(2)}, j_3\in \mc H^{(3)}$, and
  the row of $\wt B_S$ together with the mixing wights specifies how this column is formed
  as a  linear combination
  of $3k$  columns of  $\wt P_S$.
  By the structure of $M_4$ in Claim~\ref{claim:M4-proj4}, the $(j_1,j_2,j_3)$-th row of $\wt
  B^{(1)}$ has exactly $k$ entries corresponding to $\wt \Sigma^{(i)}_{j_2,j_3}$ for $i\in [k]$,
  these entries are multiplied by $\wt \omega_i$ in the middle (diagonal) matrix. Therefore, these
  directly correspond to the $k$ terms in Claim~\ref{claim:M4-proj4}. Similarly the entries in $\wt
  B^{(2)}$ and $\wt B^{(3)}$ correspond to the other $2k$ terms.
\end{proof}

Using Claim~\ref{claim:1astructure}, we need to bound the smallest singular value for each of the
matrices in order to bound the $k|\mc H|$-th singular value of $\wt Q_S$, this is deferred to the
end of this part. The most important tool is a corollary (Lemma~\ref{lem:robust-sig-min}) of the
random matrix result proved in \cite{rudelson2009smallest}, which gives a
lowerbound on the smallest singular value of perturbed rectangular matrices.

By Lemma~\ref{prop:find-S}, we know $\wt Q_S$ has exactly rank $k|\mc H|$, and is robust
in the sense that its $k|\mc H|$-th singular value is large (polynomial in the amount of
perturbation $\rho$). By standard matrix perturbation theory, if we get $\wh Q_S$ close to
$ \wt Q_S$ up to a high accuracy (inverse polynomial in the relevant parameters), the top
$k|\mc H|$ singular vectors will span a subspace that is very close to the span of $\wt
Q_S$. We formalize this in the following lemma.

\begin{lemma}
[Lemma~\ref{lem:1aperturb} restated]
  \label{prop:bound-S-prj}
  Given the empirical estimator of the 4-th order moments $\wh M_4 = \wt M_4 + E_4$.  and suppose that the
  absolute value of entries of $E_4$ are at most $\delta$.  Let the columns of matrix $\wt S\in\mbb R^{n\times
    k|\mc H|}$ be the left singular vector of $\wt Q_S$, and let $\wh S$ be the corresponding matrix obtained
  with $\wh M_4$.
  Conditioned on the high probability event $\sigma_{k|\mc H|} ( \wt Q_S) >0$, for some
  absolute constant $C$ we have:
  \begin{align}
    \label{eq:bound-S-prj}
    \|\prj_{\wh S} - \prj_{\wt S}\| \le { C n^{1.25} \over \sigma_{k|\mc H|} ( \wt Q_S)} \delta.
  \end{align}
\end{lemma}

\begin{proof}
  Note that the columns of $S$ are the leading left singular vectors of $Q_{S}$. We apply
  the standard matrix perturbation bound of singular vectors.
  Recall that $S$ is defined to be the first $k|\mc H|$ left singular vector of $Q_S$, and we have
  \begin{align*}
    \|\wh Q_S-\wt Q_S\|\le \|\wh Q_S-\wt Q_S\|_F &\le \sqrt { n (|\mc H|/3)^3\delta^2}.
  \end{align*}
  Therefore by Wedin's Theorem (in particular the corollary Lemma~\ref{cor:perturb-prj-bound}), we can conclude
  \eqref{eq:bound-S-prj}.
\end{proof}

Next, we prove Lemma~\ref{prop:find-S}.

\paragraph{Proof of Lemma~\ref{prop:find-S}}

We first use Claim~\ref{claim:1astructure} to write $ \wt Q_S = \wt P_S \lt(D_{\wt
  \omega}\ot_{kr}I_{|\mc H|}\rt) (\wt B_S)^\top$, note that the matrix $(D_{\wt
  \omega}\ot_{kr}I_{|\mc H|})$ has dimension $k|\mc H|\times k|\mc H|$, therefore we just
need to show with high probability each of the three factor matrix has large $k|\mc H|$-th
singular value, and that implies a bound on the $k|\mc H|$-th singular value of $\wt Q_S$
by union bound. The smallest singular value of $\wt P_S$ and $\wt B_S$ are bounded below by the
following two Claims.

\begin{claim}
With high probability $\sigma_{k|\mc H|} (\wt P_S) \ge \Omega(\rho\sqrt{n})$.
\end{claim}

\begin{proof}
  This claim is easy as $\wt P_S\in\R^{n\times k|\mc H|}$ is a tall matrix with $\qq{n \ge
    5k|\mc H|}$ rows. In particular, let $\wt P'_S$ be the block of $\wt P_S$ with rows
  restricted to $\mc H^C = [n]\backslash \mc H$. Note that $\wt P'_S$ is a linear
  projection of $\mc P_S$, and by basic property of singular values in
  Lemma~\ref{lem:sigv-prj}, the $k|\mc H|$ singular values of $\wt P'_S$ provide lower
  bounds for the corresponding ones of $\wt P_S$. We only consider the restricted rows so
  that $\wt P'_S$ does not involve any diagonal elements of the covariance matrices, which
  are not randomly perturbed in our smoothed analysis framework.


  Now $\wt P'_S$ is a randomly perturbed rectangular matrix, whose smallest singular value
  can be lower bounded %
  using Lemma~\ref{lem:robust-sig-min}, and we conclude that with probability at least $1- (C\epsilon)^{ 0.25
    n}$,
  \begin{align*}
    \sigma_{ k|\mc H|} (\wt P_S) \ge \epsilon\rho\sqrt{ n}.
  \end{align*}
\end{proof}

Next, we bound the smallest singular value of $\wt B_S$.

\begin{claim}
With high probability $\sigma_{k|\mc H|}(\wt B_S) \ge \Omega(\rho\sqrt{n})$.
\end{claim}

\begin{proof}
We make use of the special structure of the three blocks of $\wt B_S$ to lower bound its
smallest singular value.

First, we prove that the block diagonal matrix $\wt B^{(1)}$ has large singular
values, even after projecting to the orthogonal subspace of the column span of $\wt
B^{(2)}$ and $\wt B^{(3)}$.  This idea appeared several times in our proof and is
abstracted in Lemma~\ref{lem:prjdiag}. Apply the lemma and we have:
  \begin{align}
    \sigma_{ k|\mc H|} (\wt B_S)
    &\ge
    \min \lt\{
    \sigma_{k(2|\mc H|/3)}([\wt B^{(2)},\wt B^{(3)}]), \
    \sigma_{k} (\prj_{([\wt B^{(2)},\wt B^{(3)}]_{
        \{j\} \times \mc H^{(2)}\times \mc H^{(3)}
      })^\perp} \wt \Sigma_{\mc H^{(2)}, \mc H^{(3)}} ): j    \in \mc H^{(1)}  \rt\}
    \label{eq:C5-1}
    \\
    &\ge
        \min \lt\{
    \sigma_{k(2|\mc H|/3)}([\wt B^{(2)},\wt B^{(3)}]), \
    \sigma_{k} (\prj_{([\wt B^{(2)},\wt B^{(3)}]_{
                \{j\} \times \mc H^{(2)}\times \mc H^{(3)}
      })^\perp} \prj_{\Sigma_{\mc H^{(2)}, \mc H^{(3)}}^\perp}\wt \Sigma_{\mc H^{(2)}, \mc H^{(3)}}
    ):j    \in \mc H^{(1)}  \rt\},
    \nonumber
  \end{align}
  where the $j$-th block of $[\wt B^{(2)},\wt B^{(3)}]$ has dimension $ (|\mc H|/3)^2 \times 2k|\mc
  H|/3$. Since
  \begin{align*}
    (|\mc H|/3)^2- k - 2k|\mc H|/3 = \Omega(n/9-k-2kn^{0.5}/3) \ge \Omega(n),
  \end{align*}
  this means for each block, even after projection it has more than $3k$ rows.
  Note that by definition the three blocks $\wt B^{(1)}$, $\wt B^{(2)}$ and $\wt B^{(3)}$ are independent and do not involve any diagonal elements of
  the covariance matrices, so each block after the two projections is again a rectangular random matrix.  We can apply
  Lemma~\ref{lem:prj-rand-gaussian}, for any $j$, for some absolute constant $C_1,C_2,C_3$ (not fixed throughout the discussion), with probability at
  least $1-(C_1\epsilon)^{C_2n}$ over the randomness of $\wt \Sigma_{\mc H^{(2)}, \mc H^{(3)}}$, we have:
  \begin{align}
    \label{eq:bound-k-prj-rand-gaussian}
  \sigma_{k} (\prj_{([\wt B^{(2)},\wt B^{(3)}]_{
        \{j\} \times \mc H^{(2)}\times \mc H^{(3)} })^\perp} \prj_{\Sigma_{\mc H^{(2)}, \mc H^{(3)}}^\perp}\wt \Sigma_{\mc H^{(2)}, \mc
    H^{(3)}} ) &\ge \epsilon\rho\sqrt{C_3 n}.
  \end{align}

  Now we can take a union bound over the blocks and conclude that with high probability,
  the smallest singular value of each block is large.

  In order to bound $\sigma_{k(2|\mc H|/3)}([\wt B^{(2)},\wt B^{(3)}])$, we use the same strategy. Note that $\wt B^{(2)}$ also has a block structure that corresponds to the $\mc H^{(1)}\times \mc H^{(3)}$ faces (see Figure~\ref{fig:bstruct}). 
  Again check the condition on dimension $ (|\mc H|/3)^2- k - k|\mc H|/3 \ge \Omega(n)> 3k$, we can
  apply Lemma~\ref{lem:prjdiag} again to show that for any $j$, with probability at least
  $1-(C_1\epsilon)^{C_2 n}$ over the randomness of $\wt \Sigma_{\mc H^{(1)}, \mc H^{(3)}}$, we have:
  \begin{align}
    \label{eq:C5-2}
    \sigma_{k(2|\mc H|/3)}([\wt B^{(2)},\wt B^{(3)}])\ge
    \min \{
    \sigma_{k(|\mc H|/3)}(\wt B^{(3)}), \
    \sigma_{k} (\prj_{([\wt B^{(3)}]_{
        \mc H^{(1)}\times \{j\} \times \mc H^{(3)}
      })^\perp} \prj_{\Sigma_{\mc H^{(1)}, \mc H^{(3)}}^\perp}\wt \Sigma_{\mc H^{(1)}, \mc H^{(3)}} ): j
    \in \mc H^{(2)}
    \}.
  \end{align}
  Again by Lemma~\ref{lem:prj-rand-gaussian}, for any $j$, with probability at least
  $1-(C_1\epsilon)^{C_2n}$ over the randomness of $\wt \Sigma_{\mc H^{(1)}, \mc H^{(3)}}$, we have:
  \begin{align}
    \label{eq:bound-k-prj-rand-gaussian-2}
    \sigma_{k} (\prj_{([\wt B^{(3)}]_{ \mc H^{(1)}\times \{j\} \times \mc H^{(3)} })^\perp} \prj_{\Sigma_{\mc H^{(1)}, \mc
        H^{(3)}}^\perp}\wt \Sigma_{\mc H^{(1)}, \mc H^{(3)}} ) &\ge \epsilon\rho\sqrt{C_3n}.
  \end{align}

  Finally, for $\wt B^{(3)}$ it is a block diagonal structure with blocks correspond to $\mc H^{(1)}\times \mc H^{(2)}$ faces (see Figure~\ref{fig:bstruct}).
    Each block is a perturbed rectangular matrix,
 therefore we apply Lemma~\ref{lem:prj-rand-gaussian} to have that with high probability over the randomness of $\wt \Sigma_{\mc
    H^{(1)}, \mc H^{(2)}}$,
  \begin{align}
    \label{eq:bound-k-prj-rand-gaussian-3}
    \sigma_{k(|\mc H|/3)}(\wt B^{(3)}) \ge \sigma_{k}(\wt \Sigma_{\mc H^{(1)}, \mc H^{(2)}}) \ge
    \epsilon\rho\sqrt{n}.
  \end{align}

  Now plug in the lower bounds in \eqref{eq:bound-k-prj-rand-gaussian}
  \eqref{eq:bound-k-prj-rand-gaussian-2} \eqref{eq:bound-k-prj-rand-gaussian-3} into the
  inequalities in \eqref{eq:C5-1} and \eqref{eq:C5-2}.  By union  bound we conclude that with high probability:
  \begin{align*}
    \sigma_{k|\mc H|} (\wt B_S) \ge \epsilon\rho \sqrt{C_3 n}.
  \end{align*}
\end{proof}

Finally, the diagonal matrix in the middle is given by the Kronecker product of $I_{|\mc H|}$ and
$D_{\wt \omega}$.  Recall that $D_{\wt \omega}$ is the diagonal matrix with the mixing
weights $\wt \omega_i$'s on its diagonal.  By property of Kronecker product and the
assumption on the mixing weights, the smallest diagonal element of $D_{\wt
  \omega}\ot_{kr}I_{|\mc H|}$ is at least $\omega_0$. Therefore $\sigma_{k|\mc H|} (D_{\wt
  \omega}\otimes_{kr} I_{|\mc H|}) \ge \omega_0$.

We have shown that the smallest singular value of all the three factor matrices are large
with high probability.  Therefore, apply union bound, we conclude that with probability at
least $1-\exp(-\Omega(n))$,
  \begin{align*}
    \sigma_{k|\mc H|} (\wt Q_S) \ge\sigma_{k|\mc H|} (\wt P_S)\sigma_{k|\mc H|} (D_{\wt \omega}\otimes_{kr} I_{|\mc H|})\sigma_{k|\mc H|} (\wt B_S) \ge O(\omega_o\rho^2 n).
  \end{align*}

\subsection{Step 1 (b). Finding $\mc U_S$, the span of $\wt\Sigma^{(i)}$'s with columns projected to $\mc
  S^\perp$.}
\label{sec:step-1b}

\begin{algorithm}[h!]
  \caption{FindProjectedSigmaSpan}
  \label{alg:projsigma}
  \DontPrintSemicolon

  \tb{Input:} 4-th order moments $M_4$, set of indices $\mc H$, subspace $S\subset \R^n$

  \tb{Output:}  $span\{\vc(\Proj_{S^\perp} \Sigma^{(i)}):i\in[k]\}$, represented
  by an orthonormal matrix $U_S\in\R^{n^2 \times k}$.

  \BlankLine

  \STATE Let $Q$ be a matrix whose columns are $\vc(\Proj_{S^\perp}M_4(e_i,e_j,I,I))$ for
  all $i,j\in \mc H$, $i\ne j$.

  \STATE Compute the SVD of $Q$: $Q = UDV^\top$.

  \BlankLine

  \tb{Return:} The first $k$ left singular vectors $U_S=[U_{[:,1]},\dots, U_{[:,k]}]$.

\end{algorithm}

In Step 1 (b), given the subset of indices $\mc H$ and the subspace $\mc S$ obtained in Step 1 (a), we want to
show that the projected two-dimensional slices of $\wt M_4$ span the subspace $\mc U_S$ defined in
\eqref{eq:def-mcUS}, which is the span of the covariance matrices with the columns projected the subspace $\mc
S^{\perp}$:
\begin{align*}
   \mc U_S =\tx{span}\lt\{\vc(\prj_{S^{\perp}}\wt \Sigma^{(i)}) : i\in[k]\rt\} \subset \R^{n^2}.
 \end{align*}

Recall that in Claim~\ref{claim:M4-proj34}, we characterized the two dimensional
slices of the 4-th moments $M_4$ of mixture of zero-mean Gaussians as below:
  \begin{align}
    \label{eq:C2-1}
    \wt M_4(e_{j_1},e_{j_2},I,I) = \sum_{i=1}^k\wt \omega_i \left(\wt \Sigma^{(i)}_{j_1,j_2}
     \wt  \Sigma^{(i)}+\wt \Sigma^{(i)}_{[:,j_1]} (\wt \Sigma^{(i)}_{[:,j_2]})^\top+\wt \Sigma^{(i)}_{[:,j_2]}
      (\wt \Sigma^{(i)}_{[:,j_1]})^\top\right),
    \quad\fa j_1,j_2\in [n].
  \end{align}

For notational convenience, we let $\mc J$ denote the set $ \mc J = \{(j_1,j_2): j_1\le j_2, \ j_1,j_2 \in\mc
H\}$, and note that the cardinality is $|\mc J| = {|\mc H| + 1\choose 2} = \qq{(n+\sqrt{n})/2}$.
First, we define the matrix $\wt Q_{U_S}\in\R^{n^2\times |\mc J|}$ whose columns are the vectorized
two-dimensional slices of $\wt M_4$ with the columns projected to the subspace $\mc S^{\perp}$:
\begin{align}\label{eq:QUs-def}
\wt   Q_{U_S} &= \lt[\vc(\prj_{S^{\perp}}\wt M_4(e_{j_1},e_{j_2},I,I)) : (j_1,j_2) \in\mc J \rt].
\end{align}
Similarly we define $\wt Q_{U_0}\in\R^{n^2\times |\mc J|}$ with the slices without the projection:
\begin{align*}
  \wt Q_{U_0} =\lt[\vc(\wt M_4(e_{j_1},e_{j_2},I,I)) :(j_1,j_2)\in\mc J \rt].
\end{align*}

Observe the structure in \eqref{eq:C2-1} and we see the columns of $\wt Q_{U_0}$ is ``almost'' in
the span of covariance matrices, except for some additive rank one terms. Note that all the rank one
terms lie in the subspace $\mc S$ obtained from Step 1 (a), and they vanish if we project the slice
to the orthogonal subspace $\mc S^\perp$.  In particular, $\prj_{\mc S^\perp} \wt
\Sigma^{(i)}_{[:,j]} = 0$ for all $j\in S$.  Let the columns of the matrix $\wt
P_{U_S}\in\R^{n^2\times k}$ be the vectorized and projected covariance matrices as below:
\begin{align}
  \label{eq:PUs-def}
  \wt P_{U_S} = \left[ \vc(\prj_{S^{\perp}} \wt \Sigma^{(i)}):i\in[k]\right].
\end{align}
In the following claim, we show that the columns of $\wt Q_{U_S}$ indeed lie in the column span of
$\wt P_{U_S}$:
\begin{claim}
\label{claim:projected34}
Given $S$ obtained in Step 1(a), the span of $\wt \Sigma_{[:,j]}^{(i)}$ for $j\in\mc H$ and for all
$i$, then for $j_1,j_2\in \mc H$, we have:
  \begin{align*}
    \prj_{ S^\perp} \wt M_4(e_{j_1},e_{j_2},I,I) = \sum_{i=1}^k\wt \omega_i \wt \Sigma^{(i)}_{j_1,j_2}
  \prj_{ S^\perp}   \wt  \Sigma^{(i)},    \quad\fa j_1,j_2\in [n].
  \end{align*}
\end{claim}
Similar as in Step 1(a), in the next lemma we show that the columns of $\wt Q_{U_S}$ indeed span the
entire column span of $\wt P_{U_S}$. Since the dimension of the column span of $\wt P_{U_S}$ is no
larger than $k$,  it is enough to  the $k$-th singular value of $\wt Q_{U_S}$:
%
%
\begin{lemma} [Lemma~\ref{lem:1bsmooth} restated]
  \label{prop:bound-sig-QUs}
  Given $\wt M_4$, the exact 4-th order moment of the $\rho$-smooth \mog,  define the matrix  $\wt
  Q_{U_S}$ as  in \eqref{eq:QUs-def} with the two-dimensional slices of $\wt M_4$.  For any $\epsilon>0$, and
  for some absolute constant $C_1,C_2,C_3>0$, with probability at least $1- 2 (C_1\epsilon)^{C_2 n}$, the $k$-th
  singular value  of $\wt Q_{U_S}$ is bounded below by:
  \begin{align*}
    \sigma_{k} ( \wt Q_{U_S}) \ge C_3 \omega_o (\epsilon\rho)^2 n^{1.5}.
  \end{align*}
\end{lemma}

Similar as before, we first examine the structure of the matrix $\wt Q_{U_S}$:

\begin{claim} [Structural]\label{claim:1bstructure} Under the same assumption as
  Lemma~\ref{prop:bound-sig-QUs}, we can write $\wt Q_{U_S}$ in the following matrix product form:
\begin{align}
  \label{eq:QUs-PUs}
  & \wt Q_{U_S} = \wt P_{U_S} D_{\wt \omega} {\wt \Sigma_J}^\top.
\end{align}
The columns of the matrix $\wt P_{U_S}\in\R^{n^2\times k}$ are the vectorized and projected
covariance matrices as defined in (\ref{eq:PUs-def}); $D_{\wt \omega}$ is the diagonal matrix with
the mixing weights $\wt\omega_i$ on its diagonal; and the  matrix  $\wt \Sigma_J$ is defined as:
\begin{align*}
  {\wt \Sigma_J} &= \lt[ \vc[ \wt \Sigma_{(j_1,j_2)}^{(i)}: (j_1,j_2)\in\mc J]:i\in[k] \rt]
  \in\R^{|\mc J|\times k}.
\end{align*}
\end{claim}

\begin{proof}
  This claim follows from Claim~\ref{claim:projected34}, and the rule of matrix product.
  The coefficients $\wt \omega_i\wt \Sigma_{j_1,j_2}^{(i)}$ for the linear combinations of $\vc(\prj_{\mc S^\perp} \wt \Sigma^{(i)})$ are given by the
  columns of the product $D_{\wt\omega}\wt \Sigma_J^\top$. The coefficients are then multiplied by $\wt P_{U_S}$ to select the correct columns.
\end{proof}

To prove Lemma~\ref{prop:bound-sig-QUs}, similar
to the proof ideas of Lemma~\ref{prop:find-S}, we lower bound the $k$-th singular value of all the three factors.

\paragraph{Proof of Lemma~\ref{prop:bound-sig-QUs}}

By the structural Claim~\ref{claim:1bstructure}, we know the matrix $\wt Q_{U_S}$ can be written as
a product of the three matrices as $\wt Q_{U_S} = \wt P_{U_S}D_{\wt \omega} {\wt \Sigma_J}^\top$.

We lower bound the $k$-th singular value of each of the three factors.  It is easy for the last two matrices. Note that by assumption $\sigma_k
(D_{\wt \omega}) \ge \omega_o$, and since ${\wt \Sigma_J}^\top$ is just a perturbed rectangular matrix,  we can  apply
Lemma~\ref{lem:prj-rand-gaussian} and  with high probability we have $\sigma_k( {\wt \Sigma_J}) \ge \Omega(\rho \sqrt{n})$.

The first matrix $\wt P_{U_S}$ is more subtle. Let us define the projection $D_{ S^\perp}= \prj_{ S^{\perp}} \ot_{kr} I_n \in\R^{n^2\times n^2}$. This
is just a way of saying ``apply the projection $\prj_{ S^{\perp}}$ to all columns'' and then vectorize the matrix. In particular, for any matrix $A$
we have $D_{\mc S^\perp} \vc(A) = \vc(\prj_{\mc S^\perp} A)$, therefore by definition of $\wt P_{U_S}$ we can write $\wt P_{U_S} = D_{ S^\perp} \wt
\Sigma$.

However, we cannot apply the same trick to directly bound the smallest singular value of $D_{S^\perp}$ and $\prj_{D_{S^\perp}}\wt \Sigma$
separately. The problem here is that $D_{\mc S^\perp}$ and $\wt \Sigma$ are not independent, as the subspace $S$ obtained in Step 1(a) also depends on
the perturbation on $\wt \Sigma$, therefore $\prj_{D_{S^\perp}}\wt \Sigma$ is not simply a projected perturbed matrix.
Instead, we show that even conditioned on the part of randomness that is common in $S$ and $\wt
\Sigma$, $\wt \Sigma$ still has sufficient randomness due to the high dimensions, and we can still extract a tall random matrix out of it.
This is elaborated in the following claim:

\begin{claim}
\label{claim:PUS}
Under the assumptions of Lemma~\ref{prop:bound-sig-QUs}, with high probability the matrix $\wt P_{U_S} = D_{\mc S^\perp} \wt \Sigma$ has smallest singular value at least $\Omega(\rho n)$.
\end{claim}

Let $\mc L$ be the set of the $(j_1,j_2)$-th entries of $\wt\Sigma^{(i)}$ for all $i$ and one of $j_1,j_2$ is in the set $\mc H$. By Step 1(a), the
subspace $\mc S' =\tx{span}( S, e_j:j\in \mc H)$ is only dependent on the entries in $\mc L$. Here we need to include the span of
$e_j$'s for $j\in \mc H$ because the {\em diagonal} entries can depend on the other random perturbations. By adding the span of the vector $e_j$'s for $j\in
\mc H$ the subspace remains invariant no matter how the diagonal entries change.

Let $\mc Z = \tx{span}(\Sigma, S' \ot_{kr}I_n)$, and recall that the columns of $\Sigma$ are
the factorization of the unperturbed covariance matrices. The subspace $\mc Z$ has dimension no larger than $|\mc H|(k+1)n+k \le n^2/10$, and depends
on the randomness of $\mc L$.

Let $\wt\Sigma = \Sigma + E $ where $E$ is the random perturbation matrix. Now we condition on the randomness in $\mc L$. By definition the subspace
$\mc Z$ is deterministic conditional on $\mc L$. However, even if we only consider entries of $E\backslash\mc L$ there are still at least ${n-k|\mc H|
  \choose 2} \ge n^2/4$ independent random variables.  We shall show the randomness is enough to guarantee that the smallest singular value of
$\prj_{D_{\mc S^\perp}} \wt \Sigma$ is lower bounded with high probability conditioned on $\mc L$:
  \begin{align*}
    \sigma_k(\wt P_{U_S}) &= \sigma_k(D_{\mc S^\perp}{\wt\Sigma})\\
    &\ge \sigma_k(\prj_{\mc Z^\perp} \wt \Sigma)
    \\
    & = \sigma_k(\prj_{\mc Z^\perp}\Sigma + \prj_{\mc Z^\perp} E)\\& = \sigma_k(\prj_{\mc Z^\perp} E).
    \end{align*}
    Here we used the fact that projection to a subspace cannot increase the singular
    values (Lemma~\ref{lem:sigv-prj}).

%
    Conditioned on the randomness of entries in $\mc L$, $E\backslash \mc L$ still has at least $n^2/4$ random directions, while the dimension of the
    deterministic subspace $\mc Z$ is at most $n^2/10$.  Therefore we can apply Lemma~\ref{lem:prj-rand-gaussian} again to argue that conditionally,
    for every $\epsilon>0$, with probability at least $1-(C_1\epsilon)^{C_2 n^2}$ we have:
  \begin{align*}
    \sigma_{k}(\wt P_{U_S}) \ge \epsilon\rho\sqrt{C_3 n^2}.
  \end{align*}

  In summary, apply union bound and we can conclude that with probability at least $1- (C_1\epsilon)^{C_2 n }$,
  \begin{align*}
    \sigma_{k} ( \wt Q_{U_S}) = \sigma_k(\wt P_{U_S}) \sigma_k(D_{\wt\omega}) \sigma_{k}(\wt{\Sigma}_J) \ge C_3
    \omega_o (\epsilon\rho)^2 n^{1.5}.
  \end{align*}

\qed

%
%
%

Next, we  again use matrix perturbation bounds to prove the robustness of this step, which depends on the singular value decomposition of the
matrix $\wt Q_{U_S}$.

\begin{lemma}
[Lemma~\ref{lem:2perturb} restated]
  \label{prop:bound-prj-Us}
  Given the empirical 4-th order moments $\wh M_4 = \wt M_4 + E_4$, and given the output $\prj_{\wh S^\perp}$
  from Step 1 (a).  Suppose that $\|\prj_{\wh S^\perp}- \prj_{\wt S^\perp}\|\le \delta_1$, and suppose that
  the absolute value of entries of $E_4$ are at most $\delta_2$ for $\delta_2\le \|\wt Q_{U_S}\|_F/\sqrt{n^3}$.
  Conditioned on the high probability event $\sigma_{k} ( \wt Q_{U_S}) >0$, we have:
  \begin{align}
    \label{eq:bound-prj-U-S}
    \|\prj_{\wh U_S} - \prj_{\wt U_S}\| \le { n^{2.5} \left( 1 + 2\delta_1/\delta_2 \right)\over \sigma_{k} (
      \wt Q_{U_S}) } \delta_2.
  \end{align}
\end{lemma}

\paragraph{Proof of Lemma~\ref{prop:bound-prj-Us}}

  Note that the columns of $U_S$ are the leading left singular vectors of $\wt Q_{U_S}$. We want to apply the
  perturbation bound of singular vectors.

  Similar to the proof of Lemma~\ref{prop:bound-S-prj}, we first need to bound the spectral distance
  between $\wh Q_{U_S}$ and $\wt Q_{U_S}$. In fact we will even bound the Frobenius norm difference:
  \begin{align*}
    \|\wh Q_{U_S} - \wt Q_{U_S}\|_F & = \|\wh D_{S^\perp}\wh Q_{U_0} - \wt D_{S^\perp} \wt
    Q_{U_0} \|_F
    \\
    &= \|\wt D_{S^\perp}(\wh Q_{U_0} - \wt Q_{U_0}) + (\wh D_{S^\perp} - \wt D_{S^\perp} )
    \wt Q_{U_0}
+ (\wh D_{S^\perp} - \wt D_{S^\perp} ) (\wh Q_{U_0} - \wt Q_{U_0})
    \|_F
    \\
    &\le \|\wt D_{S^\perp}\|_F \|\wh Q_{U_0} - \wt Q_{U_0} \|_F + 2\|\wh D_{S^\perp} - \wt
    D_{S^\perp} \|_F \|\wt Q_{U_0} \|_F
    \\
    &\le \sqrt{n^2}\|\wt D_{S^\perp}\|_2 \|\wh Q_{U_0} - \wt Q_{U_0} \|_F +
    2\sqrt{n}\|\prj_{\wh S^\perp} - \prj_{\wt S^\perp}\|_F \|\wt Q_{U_0} \|_F
    \\
    &\le n \sqrt{n^2|\mc J| \delta_2^2 } +2\sqrt{n} \sqrt{n^2 |\mc J|} \|\prj_{\wh S^\perp} -
    \prj_{\wt S^\perp}\|_F
    \\
    &\le n^2{|\mc H|\over \sqrt{2}} (1+2 \|\prj_{\wh S^\perp} - \prj_{\wt
      S^\perp}\|_2/\delta_2) \delta_2,
  \end{align*}
  where we used the assumption $\|\wt \Sigma^{(i)}\|\le 1$ to bound $\|\wt Q_{U_0} \|_F$, used the upperbound on $\|\wh Q_{U_0} - \wt Q_{U_0}\|_F$ to
  bound the term $\| (\wh D_{S^\perp} - \wt D_{S^\perp} ) (\wh Q_{U_0} - \wt Q_{U_0})\|_F\le \|(\wh D_{S^\perp} - \wt D_{S^\perp} )\|_F \delta_2
  \sqrt{n^2 |\mc J|}\le \|(\wh D_{S^\perp} - \wt D_{S^\perp} )\|_F\|\wt Q_{U_0}\|_F$, and used the fact that Frobenius norm is sub-multiplicative.
  Apply Wedin's Theorem (in particular the corollary Lemma~\ref{cor:perturb-prj-bound}), we can conclude \eqref{eq:bound-prj-U-S}.  \qed




\subsection{Step 1 (c).   Finding $\mc U$ by Merging the Two Projected Span}
\label{sec:step-1c}

\begin{algorithm}[h!]
  \caption{MergeProjections}
  \label{alg:mergeproj}
  \DontPrintSemicolon

  \tb{Input:} two subspaces $S_1,S_2\in \R^{n\times ks}$, two subspaces $U_1,
  U_2\in\R^{n^2\times k}$ (the span of covariance matrices projected to the corresponding
  $S_1^\perp,S_2^\perp$).

  \tb{Output:} $ span\{\Sigma^{(i)}:i\in[k]\}$, represented by an
  orthonormal matrix $U\in\R^{n^2 \times k}$.

  \BlankLine

  \STATE Let $A$ be the first $2ks$ left singular vectors of $[S_1,S_2]$.

  \STATE Let $S_3$ be the first $(n-2ks)$ left singular vectors of $I- AA^\top$.

  \STATE Let $Q = [I_{n^2},\prj_{(S_3\ot_{kr}I_n)} \prj_{U_1}]^\top U_2$, compute the SVD
  of $Q$.

  \BlankLine

  \tb{Return:} matrix $U$, whose columns are the first $k$ left singular vectors $Q$.

\end{algorithm}

\begin{figure}
\centering
\includegraphics[height =1.5in]{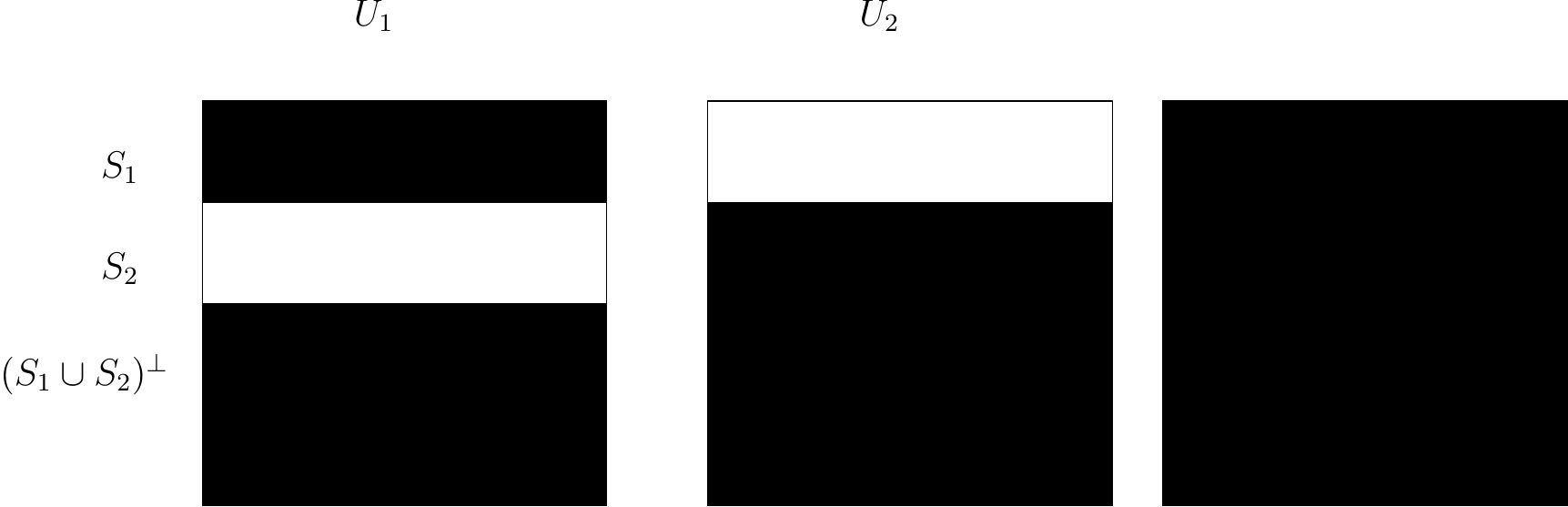}
\caption{Step 1(c): Merging two subspaces.}\label{fig:merge}
\end{figure}

Pick two disjoint sets of indices $\mc H_1, \mc H_2$, and repeat Step 1 (a) and Step 1 (b) on each of them to
get $\wt S_j^{\perp}$ and $\wt U_j$ for $j=1,2$.
In Step 1 (c), we merge the two span $\wt U_{1}$ and $\wt U_{2}$ to get $\mc U$.

If we are given two projections $\prj_{S_1^\perp}U$ and $\prj_{S_2^\perp}U$ of a {\em matrix} $U$, and if the union of the two subspaces $S_1^\perp$
and $S_2^\perp$ have full rank, namely $\tx{dim}(S_1\cup S_2) = n$, then we can  recover $U$ by:
\begin{align*}
  U = \lt[
  \begin{array}{c}{}
    \prj_{S_1^\perp}\\ \prj_{S_2^\perp}
  \end{array}
  \rt]^{\dag}
  \lt[
  \begin{array}{c}{}
    \prj_{S_1^\perp}U\\ \prj_{S_2^\perp}U
  \end{array}
  \rt].
\end{align*}
However, it is slightly different if we are given two projections of a {\em subspace} $\mc U$, since a
subspace can be equivalently represented by different orthonormal basis up to linear transformation.

In particular, in our setting for $j=1,2,$ we can write $ \wt U_j = (\prj_{ S_j^\perp}\ot_{kr}I_n) \wt{
  \Sigma} W_j$ for some fixed but unknown full rank matrix $W_j$ (which makes the columns of matrix
$\wt{\Sigma} W_j$ an orthonormal basis of $\mc U$). Recall that we define $\wt{ \Sigma}\equiv
[\vc(\wt\Sigma^{(i)}):i\in[k]]$, and $D_{S_j^\perp}\equiv \prj_{ S_j^\perp}\ot_{kr}I_n$ for $j=1,2$.

The following Lemma shows that we can still { \em robustly} recover the {\em subspace} $\mc U$ if the two projections have sufficiently large
overlapping. The basic idea is to use the overlapping part to align the two basis of the subspace which the two projections act on.
\begin{lemma}[Robustly merging two projections of an unknown subspace]
  \label{lem:merge-deterministic}
  This is the detailed statement of Condition~\ref{lem:merge-simple}.

  Let the columns of two fixed but unknown matrices $V_1\in\R^{n\times k}$ and $V_2\in\R^{n\times k}$ form
  two basis (not necessarily orthonormal) of the same $k$-dimensional fixed but unknown subspace $\mc U$ in $\R^{n}$.

  For two $s$-dimensional known subspaces $S_1$ and $ S_2$,
  Let the columns of $A$ be the first $2s$ singular vectors of $[S_1, S_2]$, and let the columns of $S_3$ correspond
  to the first $(n-2s)$ singular vectors of $(I_n - \prj_{ A})$, therefore $S_3 \subset( S_1 \cup S_2)^\perp$.
  Suppose that $\sigma_{k}(\prj_{S_3} U)>0$ and that $\sigma_{2s}([S_1, S_2])>0$.
  Define matrices $U_1 = \prj_{S_1^\perp}V_1$ and $U_2 = \prj_{S_2^{\perp}}V_2$ and we
  know that $U_1^\top U_1 = U_2^\top U_2 = I_{k}$.

  We are given $\wh S_1, \wh S_2$ and $\wh U_1, \wh U_2$, and suppose that for $j=1,2,$ we
  have $\|\wh S_j-S_j\|_F\le \delta_s$ and $\|\wh U_j-U_j\|_F\le \delta_u$, for
  $\delta_s\le 1, \delta_u\le 1$.

  Let the columns of $\wh A$ be the first
  $2s$ singular vectors of $[\wh S_1,\wh S_2]$, and let the columns of $\wh S_3$ be the
  first $(n-2s)$ singular vectors of $(I_n - \prj_{\wh A})$.
  Define matrix $\wh U\in\R^{n\times 2k}$ to be:
  \begin{align}
    \label{eq:find-Uhat}
   \wh U = \lt[
    \begin{array}[c]{c}
      \wh U_2,\quad
       \wh U_1  (\wh S_3^\top \wh U_1)^\dag (\wh S_3^\top \wh U_2)
    \end{array}
    \rt]
  \end{align}
  %
  If $ \sigma_k(\prj_{S_3}U)>0$ and $\sigma_{2s}([S_1, S_2])>0$, then for some absolute
  constant $C$ we have:
  \begin{align*}
    \|\prj_{\wh U} - \prj_{U}\|
        \le { C\sqrt{k}( \delta_u +  \delta_s/  \sigma_{2s}([S_1,S_2]) ) \over
      \sigma_k(\prj_{S_3}U)^2 \sigma_{2s}([S_1, S_2])^3 }.
  \end{align*}

\end{lemma}
\begin{proof}
  The proof will proceed in two steps, we first show that if we are given the exact
  inputs, namely $\delta_s=\delta_u=0$, then the column span of $\wh U$ defined in
  \eqref{eq:find-Uhat} is identical to the desired subspace $\mc U$.  Then we give a
  stability result using matrix perturbation bounds.

  {\noindent\em 1. Solving the problem using exact inputs.}

  Given the exact inputs $S_1,S_2$, $U_1,U_2$, first we show that under the conditions
  $\sigma_{2s}([S_1, S_2])>0$ and $\sigma_{k}(\prj_{S_3} U)>0$, then the column span of
  the matrix $\lt[U_2,\  U_1 (S_3^\top U_1)^\dag (S_3^\top U_2)\rt]$ is indeed identical to $\mc
  U = span(V_1) = span(V_2)$.

  \begin{claim}
    \label{claim:merge-with-V}
    Under the same assumptions of Lemma~\ref{lem:merge-deterministic}, given a matrix
    $V\in\R^{k\times k}$ such that $V = V_1^\dag V_2$, let $\prj_{U_0}$ be the projection
    to the column span of $U_0= [U_2,\ U_1 V]$, then we have $\prj_{U_0} = \prj_{U}$.
  \end{claim}
  \begin{proof}
    Given $V = V_1^\dag V_2$, then $U_1 V = \prj_{S_1^\perp} V_1 V = \prj_{S_1^\perp}
    V_2$. Recall that by definition $U_2 = \prj_{S_2^\perp} V_2$, then the problem is now
    reduced to the simple problem of merging two projections ($ U_2=\prj_{S_2^\perp} V_2$
    and $U_1V=\prj_{S_1^\perp} V_2$) of the same matrix ($V_2$).
    %
    Therefore, to show that the columns of $U_0=[U_2, U_1V]$ indeed span $V_2$ and thus the
    desired subspace $U$, we only need to show that $[\prj_{S_1^\perp}, \prj_{S_2^\perp}]$
    has full column span. We show this by bounding the smallest singular value of it:
    \begin{align}
      \sigma_{n} ([\prj_{S_2^\perp}, \prj_{S_1^\perp}]) \ge& \sigma_{2s}([ \prj_{
        S_2^\perp}, \prj_{ S_1^\perp}] \left[
      \begin{array}[c]{cc}
          S_1 & 0 \\ 0 &   S_2
      \end{array}
    \right] )
    \nonumber\\
    =&\sigma_{2s}(\left[
      \begin{array}[c]{c}
        (I_n -   S_2   S_2^\top)S_1 , \quad (I_n -   S_1   S_1^\top)S_2
      \end{array}
    \right])
    \nonumber \\
    =& \sigma_{2s}(\left[
      \begin{array}[c]{c}
           S_1 ,   S_2
      \end{array}
    \right]\left[
      \begin{array}[c]{cc}
        I_{s}& -  S_1^\top   S_2  \\  -  S_2^\top   S_1 & I_{s}
      \end{array}
    \right])
    \nonumber    \\
    =&\sigma_{2s}(\left[
      \begin{array}[c]{c}
           S_1 ,   S_2
      \end{array}
    \right]\left[
      \begin{array}[c]{c}
             S_1^\top \\  -  S_2^\top
      \end{array}
    \right]\left[
      \begin{array}[c]{c}
           S_1 ,  -  S_2
      \end{array}
    \right])
     \nonumber\\
    =&\sigma_{2s}(\left[
      \begin{array}[c]{c}
          S_1 ,   S_2
      \end{array}
    \right]\left[
      \begin{array}[c]{c}
           S_1 ,  -  S_2
      \end{array}
    \right]^\top\left[
      \begin{array}[c]{c}
          S_1 ,  -   S_2
      \end{array}
    \right])
     \nonumber\\
     =& \sigma_{2s}([S_1, S_2])^3
     \nonumber\\
     >&0
    \label{eq:bound-s1-s2-sv},
    \end{align}
    where the last inequality is by the assumption that $\sigma_{2s}([S_1,S_2])>0$.
  \end{proof}

  Next, we show that in the exact case, the matrix $V=V_1^{\dag} V_2$ can be computed by
  $V = (S_3^\top U_1)^\dag (S_3^\top U_2)$.  The basic idea is to use the overlapping part
  of the two projections $U_1$ and $U_2$ to align the two basis $V_1$ and $V_2$.  Recall
  that by its construction, $S_3 = (S_1\cup S_2)^\perp = S_1^\perp \cap S_2^\perp$, and
  $\prj_{S_3} = \prj_{ S_1^\perp \cap S_2^\perp}$.
  Then for $j= 1$ and $2$, we have:
  \begin{align*}
    S_3^\top U_j = S_3^\top \prj_{S_j^\perp} V_j = S_3^\top ( \prj_{S_3^{\perp}}\prj_{S_j^\perp} +
    \prj_{S_3}\prj_{S_j^\perp})V_j = S_3^{\top} (0+\prj_{S_3}) V_j = S_3^\top V_j.
  \end{align*}
  Moreover, since $U_j = \prj_{S_j^\perp}V_j$ is an orthonormal matrix, we have that all
  singular values of $V_j$ are equal or greater than 1. Also note that $U$ is an
  orthonormal matrix, so we have that $\sigma_{k}(\prj_{S_3}
  V_j)\ge\sigma_{k}(\prj_{S_3}U)>0$. In other words, $S_3^\top V_j$ has full column rank
  $k$. Therefore,
  \begin{align*}
    V &= (S_3^\top U_1)^\dag (S_3^\top U_2)
    \\ &= (S_3^\top V_1)^\dag (S_3^\top V_2)
    \\ &= (V_1^\top S_3 S_3^\top V_1)^{-1} V_1^\top S_3 (S_3^\top V_2)
    \\ &= (V_1^\top S_3 S_3^\top V_1)^{-1} V_1^\top S_3 S_3^\top V_1  V_1^\dag V_2
    \\ &= V_1^\dag V_2
  \end{align*}
  where the third equality is the Moore-Penrose definition, the fourth equality is because
  $V_1$ and $V_2$ are basis of the same subspace, there exists some full rank matrix
  $X\in\R^{k\times k}$ such that $V_2 = V_1X$, so we have $V_1V_1^\dag V_2 = V_1 V_1^\dag
  V_1 X = V_1 X = V_2$.

  {\noindent\em  2. Stability result.}

  Given $\wh S_1, \wh S_2$ and $\wh U_1, \wh U_2$ which are close to the exact $S_1,S_2,
  U_1$ and $U_2$, we then need to bound the distance $\|\prj_{\wh U} - \prj_{U}\|$. This
  follows the standard perturbation analysis. In order to apply
  Lemma~\ref{cor:perturb-prj-bound} we need to bound the distance between $\|\wh U -
  U_0\|_F$, and lower bound the smallest singular value of $U_0$, namely $\sigma_{k}( U_0)
  $. Recall that we define $U_0$ same as in \eqref{eq:find-Uhat} for the exact case with
  $\delta_s=\delta_u=0$.

  First, we bound $\|\wh U- U_0\|_F$. Note that we can write $U_0^\top$ as
  $U_0^\top=U_2B$, where $B= [I,\quad U_1(S_3^\top U_1)^{\dag} S_3]^\top$.

  Recall that $S_3 = (S_1\cup S_2)^\perp$, apply Lemma~\ref{cor:perturb-prj-bound} and we
  have:
  \begin{align*}
    \|{\wh S_3} - {S_3}\|\le \|\prj_{\wh S_1\cup \wh S_2} - \prj_{S_1\cup S_2}\| \le
    \sqrt{2} {\|[\wh S_1, \wh S_2] - [S_1, S_2]\|_F \over \sigma_{2s}([S_1,S_2])} \le
    {2\sqrt{2}\delta_s \over \sigma_{2s}([S_1,S_2])}.
  \end{align*}
  Next, note that $ \|{\wh S_3} - {S_3}\|<1$ and $\|\wh U_1 - U_1\|\le \delta_u<1$, apply
  Lemma~\ref{lem:prod-perturb} we have:
  \begin{align*}
    \|\wh S_3^\top \wh U_1 - S_3^\top U_1 \| \le 2( \|{\wh S_3} - {S_3}\| + \|\wh U_1 - U_1\|).
  \end{align*}
  Next, note that $\sigma_k(S_3^\top U_1) = \sigma_k(\prj_{S_3}V_1)>0 $ by assumption.  Apply
  Lemma~\ref{lem:perturb-pseudo-inverse}, we have:
  \begin{align*}
    \| (\wh S_3^\top \wh U_1)^\dag - (S_3^\top U_1)^\dag \| \le {2\sqrt{2}  \|\wh S_3^\top \wh U_1 - S_3^\top U_1 \| \over \sigma_k(\prj_{S_3}V_1)^2}.
  \end{align*}
  Next, apply   Lemma~\ref{lem:prod-perturb}  again we can bound the perturbation of matrix product:
  \begin{align*}
    \|\wh U - U_0\| &= \|\wh U_2 \wh B - U_2 B\|
    \\ & \le 2 (\|\wh U_2 - U_2\| + \|\wh B - B\|)
    \\ & = 2 (\|\wh U_2 - U_2\| +  \|\wh U_1(\wh S_3^\top \wh U_1)^{\dag} \wh S_3 - U_1(S_3^\top U_1)^{\dag} S_3
    \|)
    \\ & \le 2 (\|\wh U_2 - U_2\|  + 4 (\|\wh U_1 - U_1\| +  \|(\wh S_3^\top \wh U_1)^\dag - (S_3^\top U_1)^\dag \|  +
    \|\wh S_3 - S_3\| )).
    \\ & \le {  C(\delta_u +  \delta_s/  \sigma_{2s}([S_1,S_2]) ) \over
      \sigma_k(\prj_{S_3}V_1)^2},
  \end{align*}
  where $C$ is some absolute constant, and the last inequality summarizes the previous three
  inequalities, and used the fact that $\sigma_k(\prj_{S_3}V_1) <1$. Note that $\|\wh U - U_0\|_F\le
  \sqrt{k} \|\wh U - U_0\|$.

  We are left to bound $\sigma_k(U_0)$.  Recall that $\sigma_{k}(V_2)\ge \sigma_k(U_2)=1$,
  and we have shown that in the exact case $U_0 = [\prj_{S_2^\perp} V_2, \quad
  \prj_{S_1^\perp} V_2]$. Then we can bound the smallest singular value of $U_0$ following
  the inequality in \eqref{eq:bound-s1-s2-sv}:
  \begin{align*}
    \sigma_{k}(U_0) \ge& \sigma_{n} ([\prj_{S_2^\perp}, \prj_{S_1^\perp}]) \ge  \sigma_{2s}([S_1, S_2])^3.
  \end{align*}

  Finally we can apply Lemma~\ref{cor:perturb-prj-bound} to bound the distance between the projections by:
  \begin{align*}
    \|\prj_{\wh U} - \prj_{U_0}\| \le {\sqrt{2}\|\wh U- U_0\|_F\over \sigma_{k}(U_0)} \le
    { C\sqrt{k}( \delta_u + \delta_s/ \sigma_{2s}([S_1,S_2]) ) \over
      \sigma_k(\prj_{S_3}V_1)^2 \sigma_{2s}([S_1, S_2])^3 }.
  \end{align*}

\end{proof}

In Step 1 (c), we are given the output $\wt U_1$ and $\wt U_2$ from Step 1 (b), as well as the
output $\wt S_1^{\perp}$ and $\wt S_2^\perp$ from Step 1 (a).
Recall that $\mc U = span\{\vc(\wt \Sigma^{(i)}):i\in[k]\}$, and for $j=1,2$, the matrix $\wt U_j$ given by Step 1 (b)
corresponds to the subspace $\mc U$ projected to the subspace $\wt B_j = \wt S_j^\perp\ot_{kr}I_n$.

Let matrix $\wt S_3= \wt S_1^{\perp} \cap \wt S_2^{\perp} = (\wt S_1\cup \wt S_2)^\perp$ (obtained by taking the singular vectors of $(I_n -
AA^\top)$, where $A$ corresponds to the first $2k|\mc H|$ singular vectors of $[\wt S_1, \wt S_2]$), and
denote $\wt B_3 = \wt S_3\ot_{kr} I_n$. Define the matrix $\wt Q_U$ to be:
\begin{align}
  \label{eq:QU-def}
  \wt Q_U = \left[
    \begin{array}[c]{c}
     \wt U_2,\quad
      \wt U_1 (\wt B_3\wt U_1)^\dag \wt B_3\wt  U_2)
    \end{array}
  \right],
\end{align}
and similarly define the perturbed version $\wh Q_U$ to be:
\begin{align*}
  \wh Q_U = \left[
    \begin{array}[c]{c}
     \wh U_2,\quad
      \wh U_1 (\wh B_3\wh U_1)^\dag \wh B_3\wh  U_2)
    \end{array}
  \right].
\end{align*}


Now we want to apply Lemma~\ref{lem:merge-deterministic} to show that $\prj_{\wt Q_U} = \prj_{\wt \Sigma}$ and bound the distance $\|\prj_{\wh Q_U} -
\prj_{\wt \Sigma}\|$. In order to use the lemma, we first use smoothed analysis to show (in Lemma~\ref{lem:bound-B3-Sigma} and
Lemma~\ref{lem:bound-sig-S12} )that the conditions required by the lemma are all satisfied with high probability over the $\rho$-perturbation of the
covariance matrices, then conclude the robustness of Step 1 (c) in Lemma~\ref{prop:perturb-QU}.

\begin{lemma}
  \label{lem:bound-B3-Sigma}
  With high probability, for some constant $C$
  \begin{align*}
    \sigma_k(\prj_{\wt B_3}{\wt\Sigma}) \ge C\epsilon\rho n.
  \end{align*}
\end{lemma}

\begin{proof}
This is in fact exactly the same as Claim~\ref{claim:PUS}.

Given ${\wt\Sigma} = \Sigma+E$, by the definition of $\wt S_3$ and $\wt B_3$ we know that $\wt B_3$
  only depends on the randomness of $P_{J}E$ for $i=1,2$, where $$\mc J=\{(j_1,j_2):j_1\in\mc H_1\cup \mc H_2,
  \tx{ or } j_1\in\mc H_1\cup \mc H_2\},$$ and $P_{J}$ denotes the mapping that only keeps the coordinates
  corresponding to the set $\mc J$.  Therefore, we have:
  \begin{align*}
     \sigma_k(\prj_{\wt B_3} {\wt\Sigma}) \ge \sigma_k(\prj_{(\wt
      B_3^\top\Sigma)^\perp} \prj_{\wt B_3} E).
  \end{align*}
  Note that the rank of $\wt B_3^\perp$ is $2nk|\mc H|)$ and $|\mc J| = 2n |\mc H|$,
  thus $ n_2 - |\mc J| - 2nk|\mc H| - k = \Omega(n^2) > 2k$. So we can apply
  Lemma~\ref{lem:prj-rand-gaussian} to conclude that for some absolute constants $C_1, C_2, C_3$, with
  probability at least $1-(C_1\epsilon)^{C_2 n^2}$, $ \sigma_k(\wt B_3^\top {\wt\Sigma}) \ge
  \epsilon\rho\sqrt{C_3 n^2}.$
\end{proof}


\begin{lemma}
  \label{lem:bound-sig-S12}
  With high probability, for some constant $C$,
  \begin{align*}
    \sigma_{2k|\mc H|}([\wt S_1,\wt S_2 ])\ge  C\omega_o (\epsilon\rho)^2 n^{-0.25}.
  \end{align*}
\end{lemma}

\begin{proof}
  For $i=1,2$, recall that $\wt S_i$ is the singular vectors of $\wt Q_{S_i}$,  where $\wt Q_{S_i}$ is defined with the
  set $\mc H_i$ as in \eqref{eq:Qs-def}.
  We can write the singular value decomposition of $\wt Q_{S_i}$ as $\wt Q_{S_i} = \wt S_i\wt D_i
 \wt  V_i^\top$ for some diagonal matrix $\wt D_i$ and orthonormal matrix $\wt V_i$, and
  \begin{align*}
    [\wt S_1,\wt S_2] = [\wt Q_{S_1}, \wt Q_{S_2}] \lt[
    \begin{array}[c]{cc}
     \wt  V_1\wt D_1^{-1} & 0
      \\
      0 &\wt V_2\wt D_2^{-1}
    \end{array}
    \rt].
  \end{align*}
  Note that we can write $[\wt Q_{S_1}, \wt Q_{S_2}] = [\wt P_{S_1}, \wt P_{S_2}] (\diag( B_{\wt
    S_1}, B_{\wt S_2}))^{\top}$,
  and following almost exactly with the proof of Lemma~\ref{prop:find-S}, we can argue that,
  with probability at least $1-(C_1\epsilon)^{C_2 n}$,
  \begin{align*}
    \sigma_{2k|\mc H|}([\wt Q_{S_1}, \wt Q_{S_2}]) \ge C\omega_o (\epsilon\rho)^2 n.
  \end{align*}
  Moreover, by the structure of $M_4$ and the bounds on $\wt \Sigma^{(i)}\prec {1\over 2} I$, we can bound $\|\wt
  Q_{S_i}\|\le 3\sqrt{n(|\mc H|/3)^3}$, and thus:
  \begin{align*}
    \sigma_{k|\mc H|}( V_i\wt D_i^{-1}) = {1\over \sigma_{max}(\wt Q_{S_i})} \ge {1\over 3\sqrt{n(|\mc
        H|/3)^3}} = \Omega(n^{-1.25}).
  \end{align*}
  Therefore, we can conclude that, for some absolute constant $C$, we have:
  \begin{align*}
    \sigma_{2k|\mc H|} ([\wt S_1, \wt S_2]) \ge C\omega_o(\epsilon\rho)^2 n^{-0.25}.
  \end{align*}
\end{proof}

In the next lemma, we apply Lemma~\ref{lem:merge-deterministic} to show that under perturbation, with
high probability the column span of $\prj_{\wt Q_U} = \prj_{\wt \Sigma} $ and this step is robust.

\begin{lemma}
  \label{prop:perturb-QU}
  Given the output $\wh S_1, \wh S_2$ and $\wh U_1, \wh U_2$ from Step 1 (a) and (b) based on the empirical
  moments $\wh M_4$. Suppose that for $i=1,2$, $\| {\wh S_i} - {\wt S_i}\|_F \le \delta_s $, $\| {\wh U_i} -
  {\wt U_i} \|_F\le \delta_u$ for $\delta_s,\delta_u<1$.
  %
  Let the columns of $\wt U\in\R^{n^2\times k}$ be the $k$ leading singular vectors of $\wt Q_U$ defined in
  \eqref{eq:QU-def}. Then for some absolute constants $C$, with high probability,
  \begin{align}
    \label{eq:bound-prj-U}
    \|\prj_{\wh U} - \prj_{\wt U}\| \le
    {C\sqrt{k} (\delta_u + \delta_s n^{0.75}/(\omega_o\epsilon^2\rho^2) )  \over \omega_o^3
      \epsilon^8 \rho^8 n^{1.25} }.
  \end{align}
\end{lemma}

Note that $\sigma_{2k|\mc H|n}([\wt B_1, \wt B_2]) = \sigma_{2k|\mc H|}([\wt S_1,\wt S_2]) $, and for
$i=1,2$, we have $\|\wh B_i - \wt B_i\|_F \le \sqrt{n}\|\wh S_i- \wt S_i\|_F\le \sqrt{n}\delta_s$.
Therefore, with the above two smoothed analysis Lemmas showing polynomial bound of $\sigma_{2k|\mc
  H|}([\wt S_1,\wt S_2])$ and $\sigma_{k}(\prj_{\wt B_3}(\wt {\Sigma}))$,
the proof of Lemma~\ref{prop:perturb-QU} follows by applying Lemma~\ref{lem:merge-deterministic}.

\section{Step 2. Unfolding the Moments}
\label{sec:step-2}

\begin{algorithm}[h!]
  \caption{Estimate$Y_4Y_6$}
  \label{alg:estimate46}
  \DontPrintSemicolon

  \tb{Input:} 4-th order moments $\ol M_4\in\R^{n_4}$, 6-th order moments $\ol
  M_6\in\R^{n_6}$, the span of (vectorized with distinct entries) covariance matrices
  $U\in\R^{n_2\times k}$.

  \tb{Output:} Unfolded moments in the coordinate system of $U$:  $Y_4\in \R^{k\times
    k}_{sym},Y_6\in \R^{k\times k\times k}_{sym}$.

  \BlankLine

    \STATE Let $Y_4$ be  the solution to $\min_{Y_4\in\R^{k\times k}_{sym}} \|\sqrt{3}\mc
    F_4(UY_4U^\top) - \ol M_4\|_F^2$.

    \STATE Let $Y_6$ be the solution to
    $\min_{Y_6\in\R^{k\times k \times k}_{sym}} \|\sqrt{15}\mc F_6Y_6(U^\top, U^\top,U^\top) - \ol
    M_6\|_F^2$.

  \BlankLine

  \tb{Return:}  $Y_4, Y_6$.

\end{algorithm}

In the second step of the algorithm, we solve two systems of linear equations to recover the unfolded moments.

\subsection{Unfolding the $4$-th Order Moments}

Recall the first system of linear equations is
\[
\ol M_4 = \sqrt{3}\mc F_{4}\circ \mc X_4^U (Y_4).
\]
In the equation, $Y_4\in \R^{k\times k}_{sym}$ is the unknown variable which can be viewed as a $k\times k$ symmetric
matrix.
Given $U\in \R^{n_2\times k}$, the column span of ${\wt \Sigma}$ that we learned in Step 1, the
first linear transformation $\mc X_4^U$ is simply $\mc X_4^U(Y_4) = UY_4 U^\top$. It is supposed to
transform $Y_4$ into the unfolded moments $X_4\in \R^{n_2\times n_2}_{sym}$, which is defined to be
$\sum_{i=1}^k w_i \vc(\wt \Sigma^{(i)})\vc(\wt \Sigma^{(i)})^\top$.
The next transformation $\sqrt{3}\mc F_{4}$ maps the unfolded moments $X_4$ to the folded moments
$\ol M_4 \in \R^{n_4}$.  As we showed in Lemma~\ref{prop:two-linear-mapping}, the mapping $\mc F_4$
is a projection.

Since $U$ is the column span matrix of $\wt {\Sigma}$, there must exist a $Y_4$ such that $X_4 = \wt
\Sigma D_{\wt \omega} \wt \Sigma^\top = UY_4 U^\top$ (recall that $D_{\wt \omega}$ is the diagonal
matrix with entries $\wt \omega_i$), so the system must have at least one solution.

Rewrite the system of linear equations $\ol M_4/\sqrt{3} = \mc F_{4}\circ \mc X_4^U (Y_4)$ in the
canonical form: $\ol M_4\sqrt{3} = H_4 \mbox{vec}(Y_4)$ where the variable $
\mbox{vec}(Y_4)\in\R^{k_2}$, and the coefficient matrix $ H_4 \in \R^{n_4\times k_2}$ is a function
of $U$ and therefore also a function of the parameter $\Sigma$ (recall \qq{$n_4 = {n \choose 4}$ and
  $k_2 = {k+1\choose 2}$}). The system has a {\em unique} solution if the smallest singular value of
the coefficient matrix $H_4$ is greater than zero.

The main theorem of this section shows that with high probability over the $\rho$-perturbation
the system has a {\em unique} solution:

\begin{theorem}
\label{thm:unfold4}
With high probability over the $\rho$-perturbation of $\wt \Sigma$, the smallest singular value of the coefficient matrix
$\wt H_4$ is lower bounded by  $\sigma_{min}(\wt H_4) \ge \Omega(\rho^2n/k)$. As a corollary, the system has a unique
solution.
\end{theorem}

In order to prove this theorem, we first need the following structural lemma:

\begin{lemma}
\label{lem:struct4}
The coefficient matrix $\wt H_4$ is equal to $\wt A_4 \wt B_4$. The first matrix $\wt A_4\in \R^{n_4\times k_2}$ has
columns indexed by pair $\{(i,j):1\le i\le j\le k\}$, and the $(i,j)$-th column is equal to $C_{i,j} \mc
F_4(\mbox{vec}(\wt \Sigma^{(i)})\od \mbox{vec}(\wt \Sigma^{(j)}))$. Here $C_{i,j} = 1$ if $i=j$ and $C_{i,j} = 2$ if $i
< j$. The second matrix $\wt B_4\in\R^{k_2\times k_2}$ transforms a $k\times k$ symmetric matrices $Y_4$ into:
$$\wt B_4 \vc(Y_4) = \vc((\wt \Sigma^\dag U) Y_4 (\wt \Sigma^\dag U)^\top). $$
\end{lemma}

Next we need to prove the bounds on the smallest singular values for $\wt A_4$ and $\wt B_4$. The first matrix $\wt A_4$
is essentially a projection of the Kronecker product $(\wt \Sigma \ot_{kr} \wt \Sigma)$. In particular, this projection
satisfy the ``symmetric off-diagonal'' property defined below:

\begin{definition}[symmetric off-diagonal]
  \label{def:sym-off-diag-4}
  Let the columns of matrix $P \in \R^{n_2^2\times d_2}$ form an (arbitrary) basis of the subspace $\mc P$, and index the
  rows of $P$ by pair $(i,j)\in [n_2]\times [n_2]$.
  The subspace $\mc P$ and the matrix $P$ is called symmetric off-diagonal, if $(i,i)$-th row of $P$ is $0$
  (``off-diagonal''), and the $(i,j)$-th row and $(j,i)$-th row are identical (``symmetric'').
\end{definition}

\begin{remark}\label{rmk:subspace}
  Since symmetric off-diagonal is a property on the structure of rows of the basis $P$. If one basis of the subspace
  $\mc P$ is symmetric off-diagonal, then any basis is too.
  Moreover, any orthogonal basis of the subspace $\mc P$ will still be
  symmetric off-diagonal.
\end{remark}

Consider a Kronecker product of the same matrix $E\in\R^{n_2\times k}$.  The columns of $E\ot_{kr}E$
are indexed by pair $(i,j)\in[k]\times [k]$.  Consider applying a symmetric off-diagonal projection
$P^\top$ to the Kronecker product.  By the property of symmetry the projection will map two columns
$E_{[:,i]}\od E_{[:,j]}$ and $E_{[:,j]}\od E_{[:,i]}$ to the same vector. Therefore the projected
Kronecker product $P^\top (E\otimes_{kr}E)$ will not have full column rank $k^2$. However, we will show
 that the $k_2$ ``unique'' columns after the projection are linearly independent.

To formalize this, we define the matrix $ (E\ot_{kr} E)_{uniq}\in \R^{n_2^2\times k_2}$ with the ``unique''
columns of $E\ot_{kr}E$ labeled by pairs $\{(i,j):1\le i\le j\le k\}$. In particular,
\begin{align*}
[ (E\ot_{kr} E)_{uniq}]_{[:,(i,j)]} =  E_{[:,i]}\od E_{[:,j]}.
\end{align*}
In the following main lemma, we show even after projection to any symmetric off-diagonal space with sufficiently many
dimensions, the ``unique'' columns of a Kronecker product of random matrices still has good condition number.

\begin{lemma}
\label{lem:projkron4}
Let $E \in \R^{n_2\times k}$ be a Gaussian random matrix (each entry distributed as $\mc N(0,1)$). Let $P\in \R^{n_2^2
  \times d_2}$ be a symmetric off-diagonal subspace of dimension $d_2 = \Omega(n_2^2)$.
Then for any constant $C>0$, when $n_2 \ge k^{2+C}$ we have with high probability $\sigma_{min} (P^\top (E\ot_{kr}
E)_{uniq}) \ge \Omega(n_2)$.
\end{lemma}

Let us first see how Theorem~\ref{thm:unfold4} follows from the two lemmas (Lemma~\ref{lem:struct4}
and Lemma~\ref{lem:projkron4} ).

\begin{proof} (of Theorem~\ref{thm:unfold4})
Using the structural Lemma~\ref{lem:struct4}, we know we only need to bound the smallest singular value of $\wt A_4$ and $\wt B_4$ separately.
The following two claims directly imply the theorem.

\begin{claim}
\label{clm:h1}
$\sigma_{min}(\wt A_4) \ge \Omega(\rho^2 n_2).$
\end{claim}

\begin{claim}
\label{clm:h2}
$\sigma_{min}(\wt B_4) \ge 1/(4\|\wt \Sigma\|^2) \ge 1/(4nk).$
\end{claim}


Next we prove the two claims.

We apply Lemma~\ref{lem:projkron4} to prove Claim~\ref{clm:h1}. Note that the $\rho$-perturbed
covariances $\wt \Sigma$ is not a random Gaussian matrix, yet it is equal to the unperturbed matrix
$\Sigma$ plus a random Gaussian matrix $E_{\Sigma} = \rho E$\footnote{Note that the diagonal entries
  are then arbitrarily perturbed, but we will project on a symmetric off-diagonal subspace so
  changes on diagonal entries do not change the result.}.  Since we consider arbitrary $\Sigma$, the
columns of $\wt \Sigma$ as well as the columns $\wt A_4$ may not be incoherent.

Instead, we project $\wt A_4$ to a subspace to strip away the terms involving the original matrix $\Sigma$.
Let $S$ be the range space corresponding to the projection $\mc F_4$. Recall that $|S|= n_4 = \Omega(n_2^2)$, and by
the definition of $\mc F_4$, $S$ is symmetric off-diagonal.
Define the subspace $S' = \mbox{span}(S^\perp, \Sigma\ot_{kr} I_{n_2}, I_{n_2}\ot_{kr} \Sigma)$. Let $P= (S')^\perp$. By
construction $ |P| \ge| S| - 2kn_2 = \Omega(n_2^2)$. Also, since $P=(S')^\perp$ is a subspace of $S$, it must also be
symmetric off-diagonal (see Remark~\ref{rmk:subspace}).
After  projecting $\wt A_4$ to $P$, we know that the $(i,j)$-th column $(1\le i\le j\le k)$ of $P^{\top}\wt A_4$ is given by:
\begin{align*}
  P^\top [\wt A_4]_{[:,(i,j)]} &= C_{i,j} P^\top ( \Sigma_{[:,i]}\od \Sigma_{[:,j]}+\rho E_{[:,i]}\od
  \Sigma_{[:,j]}+\rho\Sigma_{[:,i]}\od E_{[:,j]}+\rho^2E_{[:,i]}\od E_{[:,j]})
  \\
  &= C_{i,j}\rho^2P^\top E_{[:,i]}\od E_{[:,j]}.
\end{align*}
Thus in $P^{\top}\wt A_4$ all the terms involving $\Sigma$ disappears. Therefore
\begin{align*}
\sigma_{min}(\wt A_4) \ge \sigma_{min}(P^\top \wt A_4) = \sigma_{min}(P^\top (\wt \Sigma\ot_{kr} \wt \Sigma)_{uniq}) = \rho^2\sigma_{min}(P^\top (E \ot_{kr} E)_{uniq}) \ge \Omega(\rho^2 n_2),
\end{align*}
where the first inequality is because the smallest singular value cannot become larger after projection, the first
equality is by definition, the second equality is by the property  of $P$, and the final step uses Lemma~\ref{lem:projkron4}\footnote{Note that although diagonal entries are not perturbed, we also have $P_{[i,i]} = 0$ so we can still apply the lemma.}.

For Claim~\ref{clm:h2}. Pick any $Y_4\in\R^{k\times k}_{sym}$, we have
$$
\|\wt B_4(Y_4)\| = \|\vc((\wt \Sigma^\dag U) Y_4 (\wt \Sigma^\dag U)^\top)\| = \|(\wt \Sigma^\dag U) Y_4 (\wt \Sigma^\dag U)^\top\|_F \ge \|Y_4\|_F \sigma_{min}(\wt \Sigma^\dag U)^2 = \|Y_4\|_F/\|\wt \Sigma\|^2,
$$
where the inequality is because $\|AB\|_F \ge\sigma_{min}(A) \|B\|_F$ if $A\in \R^{m\times n}$ and $m\ge n$.
Since $\|\vc(Y_4)\|$ is within a factor of $\sqrt{2}$ to $\|Y_4\|_F$, and by the assumption
$\wt\Sigma^{(i)}\prec {1\over 2} I$ we
can bound $\|\wt\Sigma\|\le \Omega(\sqrt{nk})$, we have the desired bound for $\sigma_{min}(\wt B_4)$.
\end{proof}

\paragraph{Structure of the Coefficient Matrix} In this part we prove the structural Lemma~\ref{lem:struct4}.

\begin{proof}(of Lemma~\ref{lem:struct4}) First, assume we know the true $\wt \Sigma$ matrix, then
  in order to get the unfolded moments $X_4$, we only need to solve the equation $\mc F_4(\wt \Sigma
  D_4 \wt \Sigma^\top)= \ol M_4$ with the $k\times k$ symmetric variable $D_4$, and the solution
  should be equal to the diagonal matrix $D_{\wt \omega}$.

  However, we only know $U$ which is the column span of $\wt \Sigma$, so we can only use $U Y_4
  U^\top$ and let $U Y_4 U^\top =\wt \Sigma D_4 \wt \Sigma^\top$. Note that there is a one-to-one
  correspondence between $Y_4$ and $ D_4$. In particular we know $D_4 = (\wt \Sigma^\dag U) Y_4 (\wt
  \Sigma^\dag U)^\top$, this is exactly the second part $\wt B_4$.

  Now the first matrix $\wt A_4$ should map $\mbox{vec}(D_4)$ to $M_4$. By construction,
  the $(i,j)$-th column $(i<j)$ of $\wt A_4$ is equal to $\mc F_4(\wt \Sigma^{(i)}\od \wt
  \Sigma^{(j)} + \wt \Sigma^{(j)}\od \wt \Sigma^{(i)}) = 2\mc F_4(\wt \Sigma^{(i)}\od \wt
  \Sigma^{(j)})$, since $\mc F_4$ is symmetric off-diagonal we know $\mc F_4 (v_1\od v_2) = \mc F_4
  (v_2\od v_1)$ for any two vectors $v_1, v_2$.  For the $(i,i)$-th column, by construction they are
  equal to $\mc F_4(\wt \Sigma^{(i)}\od \wt \Sigma^{(i)})$ as we wanted.
\end{proof}

\paragraph{Main Lemma on Projection of Kronecker Product}
In this part we prove  Lemma~\ref{lem:projkron4}.

The singular values of Kronecker Product between two matrices are well-understood: they are just the products of the singular values of the two matrices. Therefore, the Kronecker product of two rank $k$ matrices will have rank $k^2$. However, in our case the problem becomes more complicated because we only look at a projection of the resulting matrix. The projected Kronecker product may no longer have rank $k^2$ because of symmetry. Here we are able to show that even with projection to a low dimensional space, the rank of the new matrix is still as large as ${k+1 \choose 2}$.

The basic idea of the proof is to consider the inner-products between columns, and show that the columns are incoherent even after projection.

\begin{proof} (of Lemma~\ref{lem:projkron4})
Consider the matrix $(E\ot_{kr} E)_{uniq}^\top PP^\top (E\ot_{kr} E)_{uniq}$, we shall show the matrix is {\em diagonally dominant} and hence its smallest singular value must be large. In order to do that we need to prove the following two claims:

\begin{claim}
\label{clm:diag}
For any $i,j\le k$, $i\le j$, with high probability $\|P^\top (E_{[:,i]}\od E_{[:,j]})\|^2 \ge \Omega(n_2^2).$
\end{claim}

\begin{claim}
\label{clm:offdiag}
For any $i,j\le k$, $i\le j$, with high probability $$\sum_{1\le i'\le j'\le k, (i,j)\ne (i',j')} |\left<P^\top (E_{[:,i]}\od E_{[:,j]}), P^\top (E_{[:,i']}\od E_{[:,j']})\right>| \le o(n_2^2).$$
\end{claim}

With this two claims, we can apply Gershgorin's Disk Theorem~\ref{thm:gershgorin} to conclude that $\sigma_{min}((E\ot_{kr} E)_{uniq}^\top PP^\top(E\ot_{kr} E)_{uniq}) \ge \Omega(n_2^2)$. Therefore $\sigma_{min}(P^\top(E\ot_{kr} E)_{uniq}) \ge \Omega(n_2)$.

Now we prove the two claims. For Claim~\ref{clm:diag}, it essentially says the projection of a random vector to a fixed subspace should have large norm. If the vector has independent entries, this is first shown in \cite{tao2006random}. Recently \cite{vu2013random} generalized the result to $K$-concentrated vectors, see Lemma~\ref{lem:vanvu}. By Lemma~\ref{lem:Kconcentrated} we know conditioned on $\|E_{[:,i]}\|,\|E_{[:,j]}\| \le 2\sqrt{n_2}$, $(E_{[:,i]}\od E_{[:,j]})_{p,q} (p\ne q)$ is $O(\sqrt{n_2})$-concentrated. By assumption $P$ ignores all the $(E_{[:,i]}\od E_{[:,j]})_{p,p}$ entries. Therefore $\Pr[| \|P^\top (E_{[:,i]}\od E_{[:,j]})\|^2 - d_2| \ge 2t\sqrt{d_2}+t^2] \le Ce^{-\Omega(t^2/n_2)} + e^{-\Omega(n_2)}$. We then pick $t = \sqrt{d_2}/5 \ge \Omega(n_2)$, which implies $\Pr[\|P (E_{[:,i]}\od E_{[:,j]})\|^2 \le d_2/2] \le Ce^{-\Omega(n_2)}$. This is what we need for Claim~\ref{clm:diag}.

For Claim~\ref{clm:offdiag}, we need to bound terms of the form $\left<P^\top (E_{[:,i]}\od E_{[:,j]}), P^\top (E_{[:,i']}\od E_{[:,j']})\right>$. These are degree-4 Gaussian chaoses and are well-studied in \cite{latala2006estimates}.

We break the terms according to how many of $i',j'$ appears in $i,j$.

\noindent {\bf Case 1:} $i',j'\not\in\{i,j\}$. In this case we first randomly pick $E_{[:,i]},E_{[:,j]}$, and condition on the high probability event that $\|E_{[:,i]}\|,\|E_{[:,j]}\| \le 2\sqrt{n_2}$. In this case the inner-product can be rewritten as $\left<P P^\top (E_{[:,i]}\od E_{[:,j]}), (E_{[:,i']}\od E_{[:,j']})\right>$, and we know $\|PP^\top (E_{[:,i]}\od E_{[:,j]})\| \le 4n_2$. Also, since $P$ is symmetric off-diagonal we know in this degree-2 Gaussian chaos (only $E_{[:,i']}$ and $E_{[:,j']}$ are random now) there are no ``diagonal'' terms. Therefore the Decoupling Theorem~\ref{thm:decoupling} shows without loss of generality we can assume $i'\ne j'$. Apply Theorem~\ref{thm:chaos} we know this term is bounded by $O(n_2^{1+\epsilon})$ with high probability for any $\epsilon > 0$.

\noindent {\bf Case 2:} One of $i',j'$ is in $\{i,j\}$. Without loss of generality assume $i'\in \{i,j\}$ (the other case is symmetric). Again we first randomly pick $E_{[:,i]},E_{[:,j]}$ and condition on the high probability event that $\|E_{[:,i]}\|,\|E_{[:,j]}\| \le 2\sqrt{n_2}$ (but this will also determine $E_{[:,i']}$). After the conditioning, only $E_{[:,j']}$ is still random, and the inner-product can be rewritten as $\left<\mbox{mat}(PP^\top (E_{[:,i]}\od E_{[:,j]}) E_{[:,i']}, E_{[:,j']}\right>$ where the fixed vector $\mbox{mat}(PP^\top (E_{[:,i]}\od E_{[:,j]})) E_{[:,i']}$ has norm bounded by $\|PP^\top (E_{[:,i]}\od E_{[:,j]})\|\|E_{[:,i']}\| \le 8n_2^{3/2}$. By property of Gaussian with high probability the inner-product is bounded by $O(n_2^{3/2+\epsilon})$ for any $\epsilon > 0$.

\noindent {\bf Case 3:} $i',j'\in \{i,j\}$. Since $i',j'$ cannot be equal to $i,j$, there is only
one possibility: $i', j'$ are both equal to one of $i,j$ and $i\ne j$. Without loss of generality
assume $i'=j'=i\ne j$. We can swap $i,j$ with $i',j'$ and this actually becomes Case 2. By the same
argument we know this term is bounded by $O(n_2^{3/2+\epsilon})$ for any $\epsilon > 0$.

There are $O(k^2)$ terms in Case 1, $O(k)$ terms in Case 2 and $O(1)$ terms in Case 3. Therefore by union bound we know the sum is bounded by $O(kn_2^{3/2+\epsilon} + k^2n_2^{1+\epsilon})$ with high probability. Recall we are assuming $n_2\ge k^{2+C}$ (which only requires $n\ge k^{1+C/2}$). Choose $\epsilon$ to be a small enough constant depending on $C$ gives the result.
\end{proof}

\subsection{Unfolding $ 6 $-th Order Moments}
Recall the second system of linear equations is
\begin{align*}
  \ol M_6 / \sqrt{15} = \mc F_{6}\circ \mc X_6^U (Y_6).
\end{align*}

In the equation, $Y_6\in \R^{k\times k\times k}_{sym}$ is the unknown variable which can be viewed
as a $k\times k\times k$ symmetric tensro. The first linear transformation $\mc X_6^U$ transforms
$Y_6$ into the unfolded moments $X_6\in \R^{n_2\times n_2\times n_2}_{sym}$, which is supposed to be
equal to $\sum_{i=1}^k \wt w_i \vc(\wt \Sigma^{(i)})\ot^{3}$. The transformation is simply $X_6 =
\mc X_6^U(Y_6) = Y_4(U^\top, U^\top, U^\top)$ where $U\in \R^{n_2\times k}$ is the column span of
${\wt \Sigma}$ that we learned in the previous section.

The next transformation $\mc F_{6}$ maps the unfolded moments $X_6$ to the folded moments $\ol M_6
\in \R^{n_6}$, which as we showed in Lemma~\ref{prop:two-linear-mapping} is a projection. Recall
that $n_6 = {n\choose 6}$.

Rewrite the system of linear equations $\ol M_6/\sqrt{15} = \mc F_{6}\circ \mc X_6^U (Y_6)$ in the canonical
form: $\ol M_6/\sqrt{15} = \wt H_6 \mbox{vec}(Y_6)$ where the coefficient matrix $\wt H_6 \in
\R^{n_6\times k_3}$ is a function of $U$ and therefore is a function of $\wt \Sigma$ (recall $k_3 =
{k+2\choose 3}$).

The second system of linear equations tries to unfold the $6$-th order moment $\ol M_6$ to get
$Y_6$. Similar to Theorem~\ref{thm:unfold4} the following theorem  guarantees that with high
probability over the perturbation the system has a unique solution.

\begin{theorem}
\label{thm:unfold6}
 With high probability over the perturbation, the coefficient matrix $\wt H_6$ has smallest singular value $\sigma_{min}(\wt H_6) \ge \Omega(\rho^3 (n/k)^{1.5})$. As a corollary, the system has a unique solution.
\end{theorem}

The proof of this theorem is very similar to the proof of Theorem~\ref{thm:unfold4}. Here we list the important steps and highlight the differences.

As before the theorem relies on a structural lemma (Lemma~\ref{lem:struct6}), and a main lemma about
the symmetric off-diagonal projection of a Kronecker product of three identical matrices
(Lemma~\ref{lem:projkron6}).

\begin{lemma}
\label{lem:struct6}
The coefficient matrix $\wt H_6$ is equal to $\wt A_6 \wt B_6$. The first matrix $\wt A_6\in
\R^{n_6\times k_3}$ has columns indexed by triples $(i_1,i_2,i_3)$ for $1\le i_1\le i_2\le i_3\le k$,  and
are given by:
\begin{align*} [\wt A_6]_{[:,(i_1,i_2,i_3)]} = C_{i_1,i_2,i_3} \mc F_6(\mbox{vec}(\wt \Sigma^{(i_1)})\od
  \mbox{vec}(\wt \Sigma^{(i_2)})\od \mbox{vec}(\wt \Sigma^{(i_3)}) ),
\end{align*}
where $C_{i_1,i_2,i_3}$ is a constant depending only on multiplicity of the indices
$(i_1,i_2,i_3)$. The second matrix $\wt B_6\in\R^{k_3\times k_3}$ transforms a  $k\times k\times k$
symmetric tensor $Y_6$ into:
\begin{align*}
  \wt B_6(Y_6) = Y_6( (\wt \Sigma^\dag U)^\top, (\wt \Sigma^\dag U)^\top, (\wt \Sigma^\dag U)^\top).
\end{align*}

\end{lemma}

Before stating the main lemma, we  update the definition of symmetric off-diagonal subspace.

\begin{definition}
  Let the columns of matrix $P \in \R^{n_2^3\times d_3}$ form a basis of a subspace $\mc P$. Index
  the rows of $P$ by triples $(i_1,i_2,i_3)\in [n_2]\times [n_2]\times [n_2]$. The matrix $P$ and
  the subspace $\mc P$ are called symmetric off-diagonal if: whenever $i_1,i_2,i_3$ are not distinct
  the corresponding row is $0$ (``off-diagonal''); and for any permutation $\pi$ over $\{1,2,3\}$,
  the rows corresponding to $(i_1,i_2,i_3)$ and $(i_{\pi(1)},i_{\pi(2)},i_{\pi(3)})$ are identical
  (``symmetric'').
\end{definition}

It is easy to verify that since the moments in $\ol M_6$ all have indices corresponding to distinct
variables, the projection $\mc F_6$ is indeed symmetric off-diagonal. The constraints in this
definition is closely related to the decoupling Theorem~\ref{thm:decoupling} of Gaussian chaoses.

Similarly, we define the ``unique''  columns in the 3-way Kronecker product to be the matrix $(E\ot_{kr}
E\ot_{kr} E)_{uniq} \in \R^{n_2^2\times k_3}$ whose columns are labeled by triples
$(i_1,i_2,i_3):1\le i_1\le i_2\le i_3\le k$, and $(E\ot_{kr} E\ot_{kr}
E)_{uniq})_{[:,(i_1,i_2,i_3)]} = E_{[:,i_1]}\od E_{[:,i_2]}\od E_{[:,i_3]}$.


\begin{lemma}
\label{lem:projkron6}
Let $E \in \R^{n_2\times k}$ be a Gaussian random matrix. Let $P\in \R^{n_2^3\times d_3}$ be a symmetric off-diagonal
subspace of dimension $d_3 \ge \Omega(n_2^3)$. For any constant $C>0$, if $n_2 \ge k^{2+C}$, with high probability
$\sigma_{min} (P^\top (E\ot_{kr} E\ot_{kr} E)_{uniq}) \ge \Omega(n_2^{3/2})$.
\end{lemma}

The proofs of Theorem~\ref{thm:unfold6} are based on the above two lemmas. The proof of
Lemma~\ref{lem:struct6} is essentially the same as Lemma~\ref{lem:struct4}.
The proof of Lemma~\ref{lem:projkron6} is very similar to that of Lemma~\ref{lem:projkron4}, and we
highlight the only different case below:

\begin{proof} (of Lemma~\ref{lem:projkron6})

As before we try to prove that the columns of $P^\top (E\ot_{kr} E\ot_{kr} E)_{uniq}$ are incoherent. Recall we needed the following two claims:

\begin{claim}
\label{clm:diag6}
For any $1\le i_1\le i_2\le i_3\le k$, with high probability $\|P^\top (E_{[:,i_1]}\od E_{[:,i_2]}\od E_{[:,i_3]})\|^2 \ge \Omega(n_2^3).$
\end{claim}

\begin{claim}
\label{clm:offdiag6}
For any $1\le i_1\le i_2\le i_3\le k$, with high probability $$\sum_{1\le i_1'\le i_2'\le i_3', (i_1,i_2,i_3)\ne (i_1',i_2',i_3')} \left|\left<P^\top (E_{[:,i_1]}\od E_{[:,i_2]}\od E_{[:,i_3]}), P^\top (E_{[:,i_1']}\od E_{[:,i_2']}\od E_{[:,i_3']})\right>\right| \le o(n_2^3).$$
\end{claim}

The first claim can still be proved by the projection Lemma~\ref{lem:vanvu}, except the vector $E_{[:,i_1]}\od E_{[:,i_2]}\od E_{[:,i_3]}$ is now $O(n_2)$-concentrated (the proof is an immediate generalization of Lemma~\ref{lem:Kconcentrated}).

The second claim can be proved using similar ideas, however there is one new case. We again separate the terms according to the number of $i_1',i_2',i_3'$ that do not appear in $\{i_1,i_2,i_3\}$.

\noindent{\bf Case 1:} At least one of $i_1',i_2',i_3'$ does not appear in $\{i_1,i_2,i_3\}$. Suppose there are $t$ of $i_1',i_2',i_3'$ that do not appear in $\{i_1,i_2,i_3\}$, similar to before we first sample $E_{i_1},E_{i_2},E_{i_3}$ and condition on the event that they all have norm at most $2\sqrt{n_2}$. The inner-product then becomes an order $t$ Gaussian chaos with Frobenius norm $n_2^{6-t/2}$. By Theorem~\ref{thm:decoupling} and Theorem~\ref{thm:chaos} we know with high probability all these terms are bounded by $n_2^{6-t/2+\epsilon}$ for any constant $\epsilon>0$.

\noindent{\bf Case 2:} All of $i_1',i_2',i_3'$ appear in $\{i_1,i_2,i_3\}$. In the previous proof (of Lemma~\ref{lem:projkron4}), there was only one possibility and it reduces to Case 1. However for $6$-th moment we have a new case: $i = i_1=i_2=i_1' < i_2'=i_3'=i_3 = j$ (and the symmetric case $i_1=i_1'=i_2' < i_2=i_3=i_3'$). For this we will treat $T = PP^\top$ as a $6$-th order tensor with Frobenius norm at most $n_2^{3/2}$ (as a matrix it has spectral norm 1, and rank at most $n_2^3$). The tensor is applied to the vectors $E_{[:,i]}$ and $E_{[:,j]}$ as $T(E_{[:,i]},E_{[:,i]},E_{[:,j]}, E_{[:,i]}, E_{[:,j]},E_{[:,j]})$. First we sample $E_{[:,i]}$, by Lemma~\ref{lem:gaussian3rd} we know with high probability what remains will be a $3$-rd order tensor $T(E_{[:,i]},E_{[:,i]}, I, E_{[:,i]}, I, I)$ with Frobenius norm bounded by $O(n_2^{2+\epsilon})$. Notice that here it is important that Lemma~\ref{lem:gaussian3rd} can handle diagonal entries, because $E_{[:,i]}$ appears on the $1,2,4$-th coordinate (instead of the first three). We the apply Lemma~\ref{lem:gaussian3rd} again on $T(E_{[:,i]},E_{[:,i]}, I, E_{[:,i]}, I, I) (E_{[:,j]},E_{[:,j]},E_{[:,j]})$\footnote{The notation might be confusing here: $T(E_{[:,i]},E_{[:,i]}, I, E_{[:,i]}, I, I)$ is a $3$rd order tensor, and we are applying it to $E_{[:,j]},E_{[:,j]},E_{[:,j]}$. The whole expression is equal to $T(E_{[:,i]},E_{[:,i]},E_{[:,j]}, E_{[:,i]}, E_{[:,j]},E_{[:,j]})$.}, and conclude that with high probability the term is bounded by $O(n_2^{2.5+2\epsilon})$ which is still much smaller than $n_2^3$.

Finally we take the sum over all terms and choose $\epsilon$ to be small enough (depending on $C$), then when $k^{2+C}\le n_2$ the sum is a lower-order term.
\end{proof}

\subsection{Stability Bounds}
For the two linear equation systems in \eqref{eq:F4F6-2}, we can write them in canonical
form with coefficient matrices $\wt H_{4}, \wt H_6$ and the unknown variable $\vc(Y_4), \vc(Y_6)$, corresponding to the
$k_2, k_3$ distinct elements in symmetric $Y_4, Y_6$, namely:
\begin{align*}
  \wt H_4 \vc(Y_4) = \ol M_4/\sqrt{3}, \quad \wt H_6 \vc(Y_6)= \ol M_6/\sqrt{15}.
\end{align*}
When $\wh M_4, \wh M_6$, the empirical moment estimations for $\wt M_4, \wt M_6$, are used
throughout the algorithm, both the coefficient matrices $\wt H_4, \wt H_6$ and the constant terms $\ol M_4,
\ol M_6$ are affected by the noise from empirical estimation. In practice, instead of solving
systems of linear equations, we solve the least square problem:
\begin{align}
  \label{eq:LS-y4y6}
  \min_{Y_4\in\R^{k\times k}_{sym}} \|\sqrt{3}\mc F_4(UY_4U^\top) - \ol{\wh M}_4\|^2, \quad
  \min_{Y_6\in\R^{k\times k \times k}_{sym}} \|\sqrt{15}\mc F_6Y_6(U^\top, U^\top,U^\top) - \ol{\wh
  M}_6\|^2.
\end{align}
and the solution to the least square problems are given by: $\vc(\wh Y_4) = \wh H_4^\dag \ol{\wh M}_4$
and $\vc(\wh Y_6) = \wh H_6^\dag \ol{\wh M}_6$.

\begin{lemma}
  \label{prop:f4f6perturbate}
  Given the empirical 4-th and 6-th order moments $\wh M_4 = \wt M_4 + E_4$, $\wh M_6 =
  \wt M_6 + E_6$, and suppose that the absolute value of entries in $E_4$ and $E_6$ are at
  most $\delta_1$.  Let $\wh U$ be the output of Step 1 for the span of the covariance
  matrices, and suppose that $\|\wh U - \wt U\|\le \delta_2$.
  Suppose that $\delta_1\le \min\{ \|\wt M_4\|_F /\sqrt{n_4}, \|\wt M_6\|_F
  /\sqrt{n_6}\}$, and $\delta_2\le \min\{1, \sigma_{k_{2}}(\wt H_4)/2, \sigma_{k_{3}}(
  \wt H_6)/2 \}$.  Then, conditioned on the high probability event that both
  $\sigma_{k_{2}}(\wt H_4),\sigma_{k_{3}}(\wt H_6)$ are bounded below, we have:
  \begin{align*}
    \|\wh Y_4-\wt Y_4\|_F \le
    O\lt(\lt(\delta_1 + {\delta_2\over \sigma_{k_{2}}(\wt H_4)^2}\rt)\sqrt{n_4}\rt).
\\
    \|\wh Y_6-\wt Y_6\|_F \le
    O\lt(\lt(\delta_1 + {\delta_2\over \sigma_{k_{3}}(\wt H_6)^2}\rt)\sqrt{n_6}\rt).
  \end{align*}
\end{lemma}

\begin{proof}
We write the proof for $\wh Y_4$, the proof for $\wh Y_6$ is exactly the same except
changing the subscripts.

Recall that the coefficient matrix $\wt H_4$ corresponds to the composition of two linear
mappings $\mc F_4(UY_4U^\top)$ on the variable $Y_4$. Since we have showed that $\mc F_4$
is a projection determined by the Isserlis' Theorem and independent of the empirical
estimation of the moments, we can bound the perturbation on the coefficient matrices by:
  \begin{align*}
    &\|\wh H_4 - \wt H_4\| \le  \|\wh U\od ^{2} - \wt U\od ^{2}\|
    \le 2 \|\wh U - \wt U\|\|\wt U\| +   \|\wh U - \wt U\|_2^2
    \le 3 \delta_2 \le \|\wt H_4\|.
  \end{align*}
Similarly, we have $\|\wh H_6 - \wt H_6\|\le \|\wh U\od ^{3} - \wt U\od ^{3}\| \le 7
\delta_2\le \|\wt H_6\|$.

Therefore we can analyze the stability of the solution to the least square problems in
\eqref{eq:LS-y4y6} as follows:
  \begin{align*}
    \| \vc(\wh Y_4)- \vc(\wt Y_4)\| &= \lt\|\wh H_4^\dag \ol{\wh{M}}_4 -  \wt H_4^\dag \ol{\wt{M}}_4
    \rt\|
    \\ &\le O( \|\wt H_4^\dag\|\|\ol{\wh{M}}_4 - \ol{\wt{M}}_4\| + \|\wh H_4^\dag - \wt H_4^\dag\| \|\ol{\wt{M}}_4\|)
    \\ &\le O(\|\ol{\wh{M}}_4 - \ol{\wt{M}}_4\| + \|\wh H_4^\dag - \wt H_4^\dag\| \sqrt{n_4})
    \\ &\le O\lt( \sqrt{n_4}(\delta_1  +\|\wh H_4^\dag\|\|\wt H_4^\dag\|\delta_2)\rt)
    \\ & \le O\lt(\sqrt{n_4} (\delta_1 + {1\over \sigma_{k_{2}}(\wt H_4)^2 } \delta_2)\rt),
  \end{align*}
  where the first inequality is by applying Lemma~\ref{lem:prod-perturb} and note that
  $\|( \ol{\wh{M}}_4 - \ol{\wt{M}}_4)\|_F \le \delta_1 \sqrt{n_4} \le
  \|\ol{\wt{M}}_4\|_F$, the second inequality is because $ \|\ol{\wt{M}}_4\|_F \le
  O(\sqrt{n_4})$, the third inequality is by applying the perturbation bound of
  pseudo-inverse in Theorem~\ref{sec:pert-bound-pseudo}, the fourth inequality is by the
  assumption that $\delta_2$ is sufficiently small compared to the smallest singular value
  of $\wt H_4$ thus $\sigma_{k_{2}}(\wh H_4) = O(\sigma_{k_{2}}(\wt H_4))$.


\end{proof}

\section{Step 3: Tensor Decomposition}
\label{sec:step-3-zero}

\begin{algorithm}[h!]
  \caption{TensorDecomp}
  \label{alg:TensorDecomp}
  \DontPrintSemicolon

  \tb{Input:} the span of covariance matrices $U\in\R^{n_2\times k}$ (vectorized with
  distinct entries), the unfolded 4-th and 6-th moments $Y_4\in\R^{k\times k}$ and
  $Y_6\in\R^{k\times k\times k}$ in the coordinate system of $U$.

  \tb{Output:} Parameters $\mc G = \{(\omega_i, \Sigma^{(i)}):i\in [k]\}$.

  \BlankLine

  \STATE Compute the SVD of $Y_4$:  $Y_4 = V_2\Lambda_2 V_2^\top$.

  \STATE Let $G= Y_6( V_2\Lambda_2^{-1/2}, V_2 \Lambda_2^{-1/2}, V_2 \Lambda_2^{-1/2})$

  \STATE Find the (unique) first $k$ orthogonal eigenvectors $v_i$ and the corresponding
  eigenvalues $\lambda_i$ of $G$, denoted by $\{(v_i,\lambda_i):i\in[k]\}$

  \STATE For all $i\in[k]$, let $\vc(\Sigma^{(i)}) = \lambda_iU V_2\Lambda_2^{1/2}v_i$, let $\omega_i=
  (\lambda_i)^{-2}$.

  \BlankLine

  \tb{Return:}  $\mc G = \{(\omega_i, \Sigma^{(i)}):i\in [k]\}$.

\end{algorithm}

Given the estimations of the unfolded moments $Y_4$ and $Y_6$ from Step 2, and given the
span of covariance matrices $U$ from Step 1, Step 3 use tensor decomposition to robustly
find the parameters of the mixture of zero-mean Gaussians.

Recall that in the coordinate system with basis $U$, the covariance matrices (vectorized
with distinct entries) are given by $ \wt\Sigma^{(i)} = \wt U\wt\sigma^{(i)}$ for all
$i$. The unfolded moments in the same coordinate system are:
\begin{align*}
  \wt Y_4 = \sum_{i=1}^{k}\wt\omega_i \wt\sigma^{(i)}\ot^2, \quad \wt Y_6 =
  \sum_{i=1}^{k}\wt\omega_i \wt\sigma^{(i)}\ot^3.
\end{align*}
We will apply tensor decomposition algorithm to find the  $\wt \sigma^{(i)}$'s.
We restate the theorem for orthogonal symmetric tensor decomposition in Anandkumar
et~al. \cite{anandkumar2012tensor} below:
\begin{theorem}[Theorem 5.1 in \cite{anandkumar2012tensor}]
  \label{thm:thm51-2012}
  Consider $k$ orthonormal
  vector $v_1,\dots v_k\in\R^n$'s and $k$ positive weights $\lambda_1,\dots
  \lambda_k$. Define the tensor $T = \sum_{i=1}^k \lambda_i v_i\ot^3$.
  Given $\wh T = T+E$ and assume that $\|E\| \le C_1 \min\{\lambda_i\}/k$, then there is
  an algorithm that finds $\lambda_i$'s and $v_i$'s in polynomial running time with the
  following guarantee: with probability at least $1-e^{-n}$, for some permutation $\pi$
  over $[k]$ and for all $i\in[k]$, we have:
\begin{align*}
  \|v_i - \wh v_i\| \le O(\|E\|/\lambda_i), \quad |\lambda_i - \wh \lambda_i| \le
  O(\|E\|).
\end{align*}
\label{thm:tensordecomp}
\end{theorem}
In order to reduce our problem to the orthogonal tensor decomposition so that the tensor
power method (Algorithm 1, page 21 in \cite{anandkumar2012tensor}) can be applied, we use
the same ``whitening'' technique as in \cite{anandkumar2012tensor}.
We first compute the SVD of the unfolded 4-th moments $\wt Y_4 = \wt V_2\wt \Lambda_2\wt
V_2^\top$, then use the singular vectors to transform the unfolded 6-th moments $Y_6$ into
an orthogonal symmetric tensor $\wt Y_6(\wt V_2\wt \Lambda_2^{-1/2},\wt V_2\wt
\Lambda_2^{-1/2}, \wt V_2\wt \Lambda_2^{-1/2})$.

Next we complete the stability analysis for the two-step procedure, i.e. whitening and
orthogonal tensor decomposition, which was not analyzed in \cite{anandkumar2012tensor}.

\begin{theorem}
  \label{thm:tensor-decomp-23}
  Consider $k$ linearly independent vectors $a_1,\dots, a_k \in\R^{n}$, and $k$ positive weights
  $\omega_1,\dots, \omega_k$.
  %
  Define $G_2 = \sum_{i=1}^{k} \omega_i a_i\ot a_i \in\R^{n\times n}_{sym}$ and $G_3 =
  \sum_{i=1}^{k} \omega_i a_i\ot a_i\ot a_i \in\R^{n\times n\times n}_{sym}$. Let
  $\gamma_{min} = \min\{\sigma_{min}(G_2),1\}$, $\gamma_{\max} = \sigma_{max}( G_2)$, and
  let $\omega_o = \min\{\omega_i\}$.
  Given $\wh G_2, \wh G_3$ and assume that:
  \begin{align*}
    \|\wh G_2 - G_2\|_F\le \delta_2 \le o\lt ({\gamma_{min}^{2.5}\over k\|G_3\|}\rt),
    \quad
    \|\wh G_3 - G_3\|_F\le \delta_3 \le o\lt ({\gamma_{min}^{1.5}\over k}\rt).
  \end{align*}
  There exists an algorithm that finds $\wh a_i$ and $\wh\omega_i$ in polynomial (in
  variables $(n,k, 1/\sigma_{min}(G_2))$) running time with the following guarantee:
  with probability at least $1-e^{-n}$, for some permutation $\pi$ over $[k]$ and for all
  $i\in[k]$ we have:
  \begin{align*}
    \|\wh a_{\pi(i)} - a_{\pi(i)} \|
    &\le \poly(\|G_3\|, 1/\sigma_{min}(G_2), 1/ \omega_{o})\delta_2 + \poly(\|G_3\|, 1/\sigma_{min}(G_2), 1/
    \omega_{o})\delta_3,
    \\
    \|\wh \omega_i - \omega_i\|&\le \poly(\|G_3\|, 1/\sigma_{min}(G_2)) \delta_2 + \poly(\|G_3\|,
    1/\sigma_{min}(G_2)) \delta_3 .
  \end{align*}

\end{theorem}

\begin{proof}
  (to Theorem~\ref{thm:tensor-decomp-23})

  {\em 1. Algorithm}

  We first apply the whitening technique in \cite{anandkumar2012tensor}: Let $\wh G_2 =
  \wh V_2\wh \Lambda_2\wh V_2^\top$ be the singular value decomposition of $\wh G_2$, and
  note that the matrix $\wh V_2\wh \Lambda_2^{-1/2}$ whitens $G_2$ in the sense that $\wh
  G_2(\wh V_2 \wh \Lambda_2^{-1/2},\wh V_2 \wh\Lambda_2^{-1/2}) = I_{n}$.
  Similarly we can whiten $\wh G_3$ with the matrix $\wh V_2\wh \Lambda_2^{-1/2}$ and
  obtain the following symmetric 3-rd order tensor $\wh G\in\R^{k\times k\times k}_{sym}$:
  \begin{align*}
    \wh G &=\wh G_3( \wh V_2\wh \Lambda_2^{-1/2},\wh V_2 \wh \Lambda_2^{-1/2},\wh V_2
    \wh\Lambda_2^{-1/2}).
  \end{align*}
  Note th at in the exact case with $G_2$ and $G_3$, we have that:
  \begin{align*}
    G = \sum_{i=1}^k \lambda_i v_i\ot^3,
  \end{align*}
  where $ \lambda_i = \omega_i^{-1/2}$, and the vectors $ v_i = \lambda_i^{-1} V_2^\top
  \Lambda_2^{-1/2} a_i$ and they are orthonormal.  Also note that \qq{ $ \lambda_{min}\ge
    1$ and $ \lambda_{max}\le \omega_o^{-1/2}$}.
  We can then apply {\em orthogonal tensor decomposition} (Algorithm 1 in
  \cite{anandkumar2012tensor}) to $\wh G$ to robustly obtain estimations of $ v_i$'s and $
  \lambda_i$'s.
  After obtaining the estimation $\wh v_i$ and $\wh \lambda_i$'s, we can further obtain the
  estimation of $a_i$'s and $\omega_i$'s as:
  \begin{align}
    \label{eq:est-a-omega}
    \wh a_i = \wh V_2\wh \Lambda_2^{1/2} \wh v_i \wh \lambda_i, \quad \wh\omega_i = (\wh \lambda_i)^{-2}
  \end{align}

  {\em 2. Stability analysis}

  The estimation of the vectors and weights are given in \eqref{eq:est-a-omega}.  In order
  to bound the distance $\|\wh a_i - a_i\|$ and $\|\wh \omega_i - \omega_i\|$, we show the
  stability of the estimation $\wh V_2$, $\wh \Lambda_2$, and $\wh v_i$, $\wh \lambda_i$
  separately.

  First, note that by assumption $\|\wh G_2 - G_2\|_F\le \delta_2$, we can apply
  Lemma~\ref{lem:perturb-singv} and Lemma~\ref{lem:perturb-subspace} to bound the singular
  values and the singular vectors of $\wh G_2$ by:
  \begin{align*}
    &\|\wh V_2 - V_2\| \le \sqrt{2} \delta_2 /\gamma_{min},\quad \|\wh \Lambda_2 -
    \Lambda_2\| \le \delta_2.
  \end{align*}
  Define $X= V_2\Lambda_2^{-1/2}$ and define $\Delta_X = \wh X - X$.  By the assumption
  that \qq{ $ \delta_2\le o(\gamma_{min})$}, we have $\|\wh V_2 - V_2\| \le 1$ and $\|\wh
  \Lambda_2^{-1/2} - \Lambda_2^{-1/2}\|\le \|\Lambda_2^{-1/2}\| \le
  \gamma_{min}^{-1/2}$. Therefore we can apply Lemma~\ref{lem:prod-perturb} to bound
  $\|\Delta_X\|$:
  \begin{align*}
    \|\Delta_X\|& \le O (\|\wh V_2 - V_2\|\| \Lambda_2^{-1/2}\| + \| V_2\|\|\wh
    \Lambda_2^{-1/2} - \Lambda_2^{-1/2}\|)
    \\
    & \le O\lt( {\delta_2\over \gamma_{min}^{1}}{ \gamma_{min}^{-1/2}} +
    (\gamma_{min}^{-1/2})^2 \delta_2\rt)
    \\
    &\le O(\delta_2/\gamma_{min}^{1.5}.)
  \end{align*}
  Moreover, since \qq{ $\delta_2\le o( \gamma_{min})$}, we also have $\|\Delta_X\|\le \| X\|
  = \gamma_{min}^{-0.5}$.


  Next, we bound the distance $\|\wh G- G\|$.  Recall that $\wh G= \wh G_3(\wh X, \wh X,
  \wh X)$. Using the fact that tensor is a multi-linear operator, and by the assumption
  that $\|\wh G_3-G_3\|\le \delta_3$, we have:
  \begin{align*}
    \epsilon \equiv \|\wh G -  G\| &\le \|\wh G_3(\wh X, \wh X,
    \wh X) -  G_3( X,  X,  X )\|_F
    \\
    &\le \|G_3(\wh X,\wh X,\wh X) - G_3(X,X,X)\| + \|\wh G_3(\wh X,\wh X,\wh X) - G_3(\wh X,\wh
    X,\wh X)\|
    \\
    &\le 3\|G_3(\Delta_X, X, X)\| + 3\|G_3(\Delta_X,\Delta_X,X)\| +
    \|G_3(\Delta_{X},\Delta_{X},\Delta_{X})\| + \delta_3 \|\wh X\|^3
    \\
    &\le 7\|G_3\|\|X\|^2\|\Delta_X\|+(\|X\|+\|\Delta_X\|)^3\delta_3
    \\
    &\le O\lt( {\|G_3\| \over \gamma_{min}^{2.5}}\delta_2 + {1\over \gamma_{min}^{1.5}}\delta_3\rt).
  \end{align*}

  Note that by the assumption \qq{$\delta_2\le o({\gamma_{min}^{2.5}\over k\|G_3\|})$, $\delta_3\le
    o({\gamma_{min}^{1.5}\over k})$}, we have $\epsilon\le o({1\over k})$. Therefore we can apply
  Theorem~\ref{thm:tensor-decomp-23} to conclude that with probability at least $1- e^{-n}$ (over
  the randomness of the randomized algorithm itself), the tensor power  algorithm runs in time
  $\poly(n,k,1/\lambda_{min})$ and for some permutation $\pi$ over $[k]$ it returns:
  \begin{align*}
    &\|\wh v_{\pi(i)} -  v_{\pi(i)} \| \le {8\epsilon\over \lambda_{min}},
    \quad |\wh \lambda_i -  \lambda_i|\le 5\epsilon, \quad \fa j\in[k].
  \end{align*}

  Finally, since we also have $5\epsilon\le 1/2\le \lambda_{min}/2$ we can bound the
  estimation error of $\wh a_{i}$ and $\wh \omega_i$ as defined in \eqref{eq:est-a-omega}
  by:
  \begin{align*}
    \|\wh a_{\pi(i)} - a_{i}\|&\le 3(\|\Delta_X\|\lambda_{max} + {1\over \gamma_{min}^{0.5}}
    {8\epsilon\over\lambda_{min}} \lambda_{max} + {1\over \gamma_{min}^{0.5}}5\epsilon)
    \\
    &\le \poly(\|G_3\|, 1/\sigma_{min}(G_2), 1/ \omega_{o})\delta_2 + \poly(\|G_3\|, 1/\sigma_{min}(G_2), 1/
    \omega_{o})\delta_3,
    \\
    \|\wh \omega_i - \omega_i\|&\le \poly(\|G_3\|, 1/\sigma_{min}(G_2)) \delta_2 + \poly(\|G_3\|,
    1/\sigma_{min}(G_2)) \delta_3 .
  \end{align*}

\end{proof}


Now we can apply Theorem~\ref{thm:tensor-decomp-23} to our case.

\begin{lemma}
  \label{prop:perturb-step3}

  Given $\wh Y_4$, $\wh Y_6$, $\wh U$ and suppose that $\|\wh Y_4 - \wt Y_4\|_F$, $\|\wh Y_6 - \wt
  Y_6\|_F$ as well as $\|\wh U- \wt U\|$ are bounded by some $inverse\ poly( n, k,
  1/\omega_o, 1/\rho)\delta$.
  There exists an algorithm that with high probability, returns $\wh\Sigma^{(i)}$'s and
  $\wh\omega_i$'s such that for some permutation $\pi$ over $[k]$, we have the distance
  $\|\wh\Sigma^{(i)} - \wt\Sigma^{(i)}\|$ and $\|\wh\omega_i-\wt\omega_i\|$ are bounded by
  $\delta$. Moreover, the running time of the algorithm is upperbounded by $\poly(n,k, 1/\omega_o,
  1/\rho)$.
\end{lemma}


\begin{proof}
  (to Lemma~\ref{prop:perturb-step3} )

  We apply Theorem~\ref{thm:tensor-decomp-23}, and pick $G_2= \wt Y_4$, $G_3 = \wt Y_6$.
  We only need to verify that $\|\wt Y_6\|$ and $1/ \sigma_{min}(\wt Y_4)$ are polynomials of the
  relevant parameters.
  This is easy to see, since $ \sigma_{min}( \wt Y_4 ) \ge \omega_o \sigma_{min}(\wt
  \Sigma)^2$, and the matrix $\wt \Sigma$ is a perturbed rectangular matrix which by
  Lemma~\ref{lem:prj-rand-gaussian} has $\sigma_{min}(\wt\Sigma)\ge
  \Omega(\rho\sqrt{n_2})$ with high probability.


  Finally, given $\wh \sigma^{(i)}$, and given the output of Step 2, i.e. $\wh U$, with inverse polynomial accuracy,
  we can recover $\wh\Sigma^{(i)}= \wh U \wh\sigma^{(i)}$ up to accuracy polynomial in the relevant parameters.
\end{proof}

\section{Proofs of Theorem~\ref{thm:main-zero-mean}}
\label{sec:mainproofzero}

\begin{algorithm}[h]
    \caption{MainAlgorithm (Zero-mean case)}
  \label{alg:main}

  \DontPrintSemicolon

  \tb{Input:} Samples $x_i$ from the \mog, number of components $k$.

  \tb{Output:} Set of parameters $\mc G = \{(\omega_i, \Sigma^{(i)}):i\in [k]\}$.

  \BlankLine
        \STATE Estimate $M_4$, $M_6$ using the samples.
        $$ M_4 = \frac{1}{N} \sum_{i=1}^N x_i\ot^4,\quad M_6 = \frac{1}{N} \sum_{i=1}^N x_i\ot^6.$$

        \STATE Let $s = 9\lceil \sqrt{n}\rceil$\\
\COMMENT{{\bf Step 1 (a)} Algorithm~\ref{alg:columnspan}}\\
        \STATE $S_1 = \mbox{FindColumnSpan}(M_4, \{1,...,s\})$,

        \STATE $S_2 = \mbox{FindColumnSpan}(M_4,\{s+1,...,2s\})$.
        \\ \COMMENT{{\bf Step 1 (b)} Algorithm~\ref{alg:projsigma}}\\
        \STATE $U_1 = \mbox{FindProjectedSigmaSpan}(M_4,\{1,...,s\}, S_1)$,

        \STATE $U_2 = \mbox{FindProjectedSigmaSpan}(M_4,\{s+1,...,2s\}, S_2)$.
\\\COMMENT{{\bf Step 1 (c)} Algorithm~\ref{alg:mergeproj}}\\
        \STATE $U = \mbox{MergeProjections}(S_1,U_1,S_2,U_2)$.
\\ \COMMENT{{\bf Step 2} Algorithm~\ref{alg:estimate46}}\\
        \STATE $(Y_4,Y_6) = \mbox{Estimate}Y_4Y_6(M_4,M_6,U)$. \\
    \COMMENT{{\bf Step 3} Algorithm~\ref{alg:TensorDecomp}} \\
    \STATE $\mc G = \mbox{TensorDecomp}(Y_4, Y_6, U)$
 \BlankLine

 \tb{Return:} $\mc G$.
\end{algorithm}

The results in all previous sections showed the correctness and robustness of each
individual step for the algorithm for zero-mean case,
In this section, we summarize those results to prove that the overall algorithm has
polynomial time/sample complexity.

\begin{lemma}[Concentration of empirical moments]
  \label{lem:concen-empirical}
  Given $N$ samples $x_1,\dots, x_N$ drawn i.i.d.  from the $n$-dimensional mixture of $k$
  Gaussians, if $N\ge n^7/\delta^2$, then with high probability, we have that for all
  $j_1, \dots, j_6\in[n]$:
  \begin{align*}
    \lt| [\wh {M}_{4}]_{j_1, j_3,j_3,j_4} - [\wt {M}_{4}]_{j_1, j_3,j_3,j_4} \rt| \le \delta,
    \quad \lt| [\wh {M}_{6}]_{j_1, j_3,j_3,j_4,j_5,j_6} - [\wt {M}_{6}]_{j_1, j_3,j_3,j_4,j_5,j_6} \rt| \le
    \delta.
  \end{align*}
\end{lemma}
\begin{proof}
  Let $x$ denote the random vector of this mixture of Gaussians. We first truncate its
  tail probabilities to make all the entries ($[x]_j$ for $j\in[n]$) in the vector $x$ be
  in the range $[-\sqrt{n},\sqrt{n}]$.
  Apply union bound, we know that with high probability (at least $1 - O(e^{-n})$), for
  all indices $j_1,\dots, j_6\in[n]$, we have $\Big|[x]_{j_1}\dots [x]_{j_6}\Big|\le n^3$.
  Then we can apply Hoeffding's inequality to bound the empirical moments by:
  \begin{align*}
    \Pr\lt[ |\wh {\mbb E}[x_{j_1}\dots x_{j_6}] - {\mbb E}[x_{j_1}\dots x_{j_6}]| \ge
    \delta\rt]
    \le \ep(-{ 2\delta^2 N^2\over N (2n^3)^2}) + O(e^{-n})
    \le O(e^{-n}).
  \end{align*}

\end{proof}

\begin{proof}
  (of Theorem~\ref{thm:main-zero-mean} )

  We show that, to achieve $\epsilon$ accuracy in the output of Step 3 in the algorithm
  for the zero-mean case, the number of samples we need to estimate the moments $M_4$ and
  $ M_6$ is bounded by a polynomial of relevant parameters, namely $\poly(
  n,k,1/\omega_o,1/\epsilon,1/\rho)$, and each step of the algorithm can be done in
  polynomial time.

  We backtrack the input-output relations from Step 3 to Step 2 and to Step 1, and we show
  that the estimation error in the empirical moments and the inputs / outputs only {\em
    polynomially} propagate throughout the steps.

  First note that we have shown that every steps fails with negligible probability
  ($O(e^{-n^C})$ for any absolute constant $C$). Then apply union bound, we have that the
  entire algorithm works correctly with high probability.
  \begin{enumerate}
  \item By Lemma~\ref{prop:perturb-step3}, in order to achieve $\epsilon$ accuracy
    in the final estimation of the mixing weights and the covariance matrices, we need to
    drive the input accuracy of Step 3 (also the output accuracy of Step 2) to be bounded
    by some inverse polynomial in $(n, 1/\epsilon, 1/ \rho, 1/\omega_o)$,
    Also recall that this step has running time   $\poly(n,k, 1/\rho,1/\omega_{o})$.
  \item Theorem~\ref{thm:unfold4} and Theorem~\ref{thm:unfold6} guarantee that with
    smoothed analysis $\sigma_{min}(\wt H_4)$ and $\sigma_{min}(\wt H_6)$ are lower
    bounded polynomially.
    Then by Lemma~\ref{prop:f4f6perturbate}, in order to have the output accuracy of
    Step 2 be bounded by inverse $\poly(n, 1/\epsilon, 1/\rho, 1/\omega_o)$, we need to
    drive the input accuracy of Step 2 ($\wh U$, $\wh M_4$) to be bounded by some other
    inverse polynomial.
    Step 2 involves solving linear systems of dimension $n_4k_2$ and $n_6k_3$, thus it
    running time is polynomial.

  \item Lemma~\ref{lem:bound-B3-Sigma} and~\ref{lem:bound-sig-S12} guarantees that with smoothed analysis
    $\sigma_{k}(\wt Q_{U})$ is lower bounded polynomially.
    Then by Lemma~\ref{prop:perturb-QU}, in order to have the output accuracy of
    Step 1 (c) ($\wh U$) be bounded by inverse polynomial, we need to drive the input
    accuracy (output $\wh S_i$ of Step 1 (a) and output $\wh U_i$ of Step 1 (b) ) to be
    bounded by some other inverse polynomial.
    Step 1 (c) involves multiplications and factorization of matrices of polynomial size,
    and thus the running time is also polynomial.
  \item Lemma~\ref{prop:bound-sig-QUs} guarantees that with smoothed analysis
    $\sigma_{k}(\wt Q_{U_S})$ is lower bounded polynomially.
    Then by Lemma~\ref{prop:bound-prj-Us}, in order to have the output accuracy of
    Step 1 (b) ($\wh U_S$) be bounded by inverse polynomial, we need to drive the input
    accuracy (output $\wh S_i$ of Step 1 (a) ) to be bounded by some other inverse
    polynomial.
    Step 1 (b) involves multiplications and factorization of matrices of polynomial size,
    and thus the running time is also polynomial.
  \item Lemma~\ref{prop:find-S} guarantees that with smoothed analysis
    $\sigma_{k}(\wt Q_{S})$ is lower bounded by inverse polynomial.
    Then by Lemma~\ref{prop:bound-S-prj}, in order to have the output accuracy of Step
    1 (a) ($\wh S$) be bounded by inverse polynomial, we need to drive the input accuracy
    (the moment estimation $\wh M_4$) to be bounded by some other inverse polynomial.
    Step 1 (a) involves multiplications and factorization of matrices of polynomial size,
    and thus the running time is also polynomial.
  \item Finally, by Lemma~\ref{lem:concen-empirical}, in order to have the accuracy of
    moment estimation $(\wh M_4, \wh M_6)$ be bounded by inverse polynomial, we need the
    number of samples $N$ polynomial in all the relevant parameters, including $k$.
  \end{enumerate}

\end{proof}



\section{ General Case}
\label{sec:main-theorem-general}
\label{sec:general-case}

In this section, we present the algorithm for learning \mog with general means. The algorithm generalizes
the insights obtained from the algorithm for the zero-mean case.  The  steps are very similar, and
we will  highlight the differences.

\begin{algorithm}[h!]
  \caption{MainAlgorithm (General Case)}
  \label{alg:Main2}

 \tb{Input:} Samples  $\{x_i\in\R^{n}:i= 1,\dots,N\}$ from the mixture of Gaussians, number of components $k$.

 \tb{Output:} Set of parameters $\mc G = \{(\omega_i, \mu^{(i)}, \Sigma^{(i)}):i\in [k]\}$.

    \STATE Estimate $M_3$ $M_4$, $M_6$ using the samples
$$
M_3 = \frac{1}{N}\sum_{i=1}^N x_i\ot^3,\ M_4 = \frac{1}{N}\sum_{i=1}^N x_i\ot^4,\ M_6= \frac{1}{N}\sum_{i=1}^N x_6\ot^3
$$

\STATE

   \STATE\STATE Step 1 (a).  \COMMENT{This can be accomplished similar to Algorithm~\ref{alg:columnspan} FindColumnSpan}

    \STATE Let $\mc H_1 = \{1,\dots, \qq{12\sqrt{n}}\}$, find $S_1= \tx{span}\{\wt \mu^{(i)}, \wt
    \Sigma^{(i)}_{[:,j]}:i\in[k],j\in \mc H_1\}$.

    \STATE Let $\mc H_2 = \{\qq{12\sqrt{n}+1},\dots, \qq{24\sqrt{n}}\}$, find $S_2= \tx{span}\{\wt \mu^{(i)}, \wt
    \Sigma^{(i)}_{[:,j]}:i\in[k],j\in \mc H_2\}$.

    \STATE

    \STATE\STATE Step 1 (b) \COMMENT{This can be accomplished similar to Algorithm~\ref{alg:projsigma}
      FindProjectedSigmaSpan}

    \STATE Find $ U_1 = span\{ \prj_{S_1^\perp} \wt \Sigma^{(i)} :i\in[k] \}$.

    \STATE Find $ U_2 = span\{ \prj_{S_2^\perp} \wt \Sigma^{(i)} :i\in[k] \}$.

    \STATE

    \STATE\STATE Step 1 (c)  \COMMENT{This can be accomplished similar to Algorithm~\ref{alg:mergeproj} MergeProjections}

    \STATE Merge $U_1$ and $U_2$ to get $Z= span\{\mu^{(i)}:i\in[k]\}$,

    \STATE\STATE $U' = span\{\vc(\prj_{Z^\perp} \Sigma^{(i)}) :i\in[k]\}$, and $U_o =
    span\{\prj_{Z^\perp} \Sigma^{(i)} \prj_{Z^\perp}:i\in[k]\}$.

    \STATE

    \STATE\STATE {Step 2}

    \STATE Project the samples to the subspace $Z^\perp$: $\prj_{Z^\perp}x = \{\prj_{Z^\perp} x_1,\dots,\prj_{Z^\perp} x_N\}$.

    \STATE Apply the algorithm for zero mean case to the projected samples,

    \STATE let $\mc G_o = \{(\omega_i, \prj_{Z^\perp} \Sigma^{(i)}
    \prj_{Z^\perp}):i\in[k]\}=\mbox{MainAlgorithm (Zero-mean case)}(\prj_{Z^\perp}x)$.

    \STATE

    \STATE\STATE {Step 3}

    \STATE Let $T = \lt[ \vc(\prj_{Z^\perp} \Sigma^{(i)} \prj_{Z^\perp}):i\in[k]\rt]^{\dag
      \top}\in\R^{n^2\times k}$,

    \STATE and let $T^{(i)}$ for $i\in[k]$ denote the columns of $T$.

    \STATE Let $M_{3(1)}\in\R^{n\times n^2}$ be the matricization of $M_3$ along the first
    dimension.

    \STATE Let $\mu^{(i)} = M_{3(1)} T^{(i)}/\omega_i$ for $ i\in[k]$ and let $\mu
    =[\mu^{(i)}:i\in[k]]$.

    \STATE

    \STATE\STATE {Step 4}

    \STATE Let $ M_{4}' = M_4+2\sum_{i=1}^{k}\omega_i\mu^{(i)}\ot^4 $.

    \STATE Find the span $S = span\{\vc(\wt \Sigma^{(i)})+\wt \mu^{(i)}\od \wt
    \mu^{(i)}:i\in [k]\}$.

    \COMMENT{This can be achieved by treating $M_4'$ as the 4-th moments of a mixture of zero-mean
      Gaussians, and apply Step 1 in the algorithm for zero-mean case to find the span of the
      covariance matrices, and let $S$ denote the result.}

    \STATE Let $\Sigma = [\vc(\Sigma^{(i)}):i\in[k]] = (\prj_{S}U' - \mu\od \mu)$.

  \BlankLine

  \tb{Return:} $\mc G = \{(\omega_i, \mu^{(i)}, \Sigma^{(i)}):i\in[k]\}$.

\end{algorithm}

\paragraph{Step 1.  Span finding}
In  this step, we find the following two subspaces:
\begin{align*}
  \wt Z = span\{ \wt\mu^{(i)}:i\in [k]\},\quad
  \wt\Sigma_o = span\{ \prj_{ \wt Z^\perp}\wt \Sigma^{(i)}\prj_{\wt Z^\perp}  \}.
\end{align*}

This is very similar to Step 1 in the algorithm for the zero-mean case, and can be achieved in three
small steps:

\begin{enumerate}
\item Step 1 (a). For a subset $\mc H$ of size $\qq{12\sqrt{n}}$, find the span $\mc  S$ of the mean vectors
  and a subset of columns of the covariance matrices:
  \begin{align*}
    \mc S = \tx{span}\{\wt \mu^{(i)}, \wt \Sigma^{(i)}_{[:,j]}:i\in[k],j\in \mc H\}.
  \end{align*}

\item Step 1 (b). Find the span of covariance matrices projected to the subspace $S^\perp$:
  \begin{align*}
    \mc U_S = span\{ \prj_{S^\perp} \wt \Sigma^{(i)} :i\in[k] \}.
  \end{align*}

\item Step 1 (c). Run 1(a) and 1(b) on two disjoint subsets $\mc H_1$ and $\mc H_2$. Merge the two
  spans $U_1$ and $U_2$ to get $\wt Z$ and $span\{ \prj_{\wt Z^\perp} \wt\Sigma^{(i)}:i\in[k]\}$.
\end{enumerate}

Next, we discuss each small step and compare it with the similar analysis of the algorithm for the
zero-mean case.

\paragraph{Step 1 (a). Find the span $\mc S$ of the means and a subset of the columns of  the covariance matrices}

Similar to Step 1 (a) for the zero-mean case, in this step we want to find a subspace $\mc S$ which
contains the span of a subset of columns of $\wt\Sigma^{(i)}$'s. However, with the mean vector
$\wt\mu^{(i)}$'s appearing in the moments, the subspace we find also contains the span of all the
mean vectors. In particular, for a subset $\mc H\in[n]$ with $|\mc H|=\sqrt{n}$, we aim to find the
following subspace:
\begin{align}
  \label{eq:def-S-gen-2}
  &\mc S = span\{\wt\mu^{(i)},\wt\Sigma^{(i)}_{[:,j]}:i\in[k],j\in\mc H\}.
\end{align}

Similar to Claim~\ref{claim:M4-proj4} for the zero-mean case, the key observation for finding the
subspace is the structure of the one-dimensional slices of the $4$-th order moments for the general
case:
\begin{claim}
  \label{claim:M4-mean}
  For any indices $j_1,j_2,j_3\in[n]$, the one-dimensional slices of $\wt M_4$ are given by:
  \begin{align}
   \label{eq:M4-mean}
    \wt M_4(e_{j_1},e_{j_2},e_{j_3}, I) =
    \sum_{i=1}^{n} \wt\omega_i\Big(
    \wt\mu_{j_1}^{(i)}\wt\mu_{j_2}^{(i)}\wt\mu_{j_3}^{(i)}\wt\mu^{(i)} +
     \sum_{\pi\in\lt\{\substack{(j_1,j_2,j_3),\\ (j_2,j_3,j_1),\\
        (j_3,j_1,j_2)}\rt\}}
    \wt\Sigma_{\pi_1,\pi_2}^{(i)}
    \wt \Sigma_{[:,\pi_3]}^{(i)} + \wt\mu_{\pi_1}^{(i)}\wt\mu_{\pi_2}^{(i)}\wt \Sigma_{[:,\pi_3]}^{(i)} +
    \wt \Sigma_{\pi_1,\pi_2}^{(i)} \wt\mu_{\pi_3}^{(i)}\wt\mu^{(i) }
    \Big)
  \end{align}
\end{claim}

Note that if we pick the indices $j_1,j_2,j_3\in\mc H$, all such one-dimensional slice of
$\wt M_4$ lie in the subspace $\mc S$. We again evenly partition the set $\mc H$ into
three disjoint subset $\mc H^{(i)}$ and take $j_i\in\mc H^{(i)}$ for $i=1,2,3$.  Define
the matrix $\wt Q_S\in\R^{n\times (|\mc H|/3)^3}$ as in \eqref{eq:Qs-def} whose columns are
the one-dimensional slices of $\wt M_4$:
\begin{align}
  \label{eq:Qs-def2}
\wt  Q_S = \left[
   \big[ [\wt M_4(e_{j_1},e_{j_2},e_{j_3},I) :j_3\in\mc H^{(3)}]: j_2\in\mc H^{(2)} \big] : j_1\in\mc H^{(1)}
  \right]
  \in\mbb  R^{n\times (|\mc H|/3)^3}.
\end{align}

The proof of this step is  similar to the Lemmas~\ref{prop:find-S} (for smoothed analysis) and
\ref{prop:bound-S-prj} (for stability analysis). The main difference is that in the matrix $\wt B$ defined in the
structural Claim~\ref{claim:1astructure}, there is now another block $\wt B^{(0)}$ with $k$
columns that corresponds to the $\wt\mu^{(i)}$ directions, which we can again handle with
Lemma~\ref{lem:prjdiag}.

Lemma~\ref{prop:find-S-general} shows the deterministic conditions for Step 1 (a) to correctly
identify the subspace $\mc S$ from the columns of $\wt Q_S$, and uses smoothed analysis to show that
the conditions hold with high probability.

\begin{lemma}[Correctness]
 \label{prop:find-S-general} 
 Given $\wt M_4$ of a general \mog, for any subset $\mc H\in[n]$ and $|\mc H| = c_2 {k}$ with the constant
 $c_2> 9$, let $\wt Q_S$ be the matrix defined as in \eqref{eq:Qs-def2}.
 The columns of $\wt Q_S$ give the desired span $S$ defined in \eqref{eq:def-S-gen-2} if the  matrix $\wt Q_S$ achieves the
 maximal column rank $k+k|\mc H|$.
 With probability (over the $\rho$-perturbation) at least $1- C\epsilon^{0.5 n}$ for some constant
 $C$, the $k(1+|\mc H|)$-th singular value of $\wt Q_S$ is bounded below by:
 \begin{align*}
   \sigma_{k(1+|\mc H|)} ( \wt Q_S) \ge \rho\epsilon \sqrt{n}.
 \end{align*}
\end{lemma}

The proof idea is similar to that of Lemma~\ref{prop:find-S}.
We construct a basis $\wt P_S\in \mbb R^{n\times (k+ k|\mc H|)}$ for the subspace $\mc S$ as follows.
\begin{align}
 \label{eq:def-PS-general}
 &\wt P_S = \left[ \big[\wt\mu^{(i)}:i\in[k]\big], \big[[\wt\Sigma_{[:,j]}^{(i)}:i\in[k]]\ : j\in\mc H^{(l)} \big]: l =
   1,2,3 \right] =\lt[\wt\mu,\ \wt \Sigma_{[:,\mc H^{(1)}]}, \wt\Sigma_{[:,\mc H^{(2)}]},
\wt \Sigma_{[:,\mc H^{(3)}]}\rt].
\end{align}
Note that the dimension of the subspace $\mc S$ is at most $ k(|\mc H|+1)<n/3$.  Then we show by
the Claim about the moment structure that the matrix $\wt Q_{S}$ can be written as a product of $\wt
P_S$ and some coefficient matrix $\wt B_S$. Then we bound the smallest singular value of the two
matrices $\wt P_S$ and $\wt B_S$ via smoothed analysis separately. The coefficient matrix $ \wt B_S$
is slightly different than that in the zero-mean case, but has similar block-diagonal structure
properties.

The detailed proof is provided below.
\begin{proof}
 (of Proposition~\ref{prop:find-S-general} )

 Similar to structural property in Claim~\ref{claim:1astructure} for the zero-mean case, we can
 write the matrix $\wt Q_S$ in a product form:
 \begin{align*}
   \wt Q_S = \wt P_S \lt(D_{\wt \omega}\ot_{kr}I_{|\mc H|}\rt) (\wt B_S)^\top.
 \end{align*}
 We will bound the smallest singular value for each of the factor, and apply union bound to conclude
 the lower bound of $ \sigma_{k(1+|\mc H|)} ( \wt Q_S)$.

 The matrix $\wt P_S\in\R^{n\times (k+ k|\mc H|)}$ is defined in
 \eqref{eq:def-PS-general}. Restricting to the rows corresponding to $[n]\backslash \mc H$,
 we can use Lemma~\ref{lem:robust-sig-min} to argue that $\sigma_{k(1+ |\mc H|)}\ge
 \epsilon\rho\sqrt{n}$ with probability at least $1-(C\epsilon)^{0.25 n}$.

 \medskip

In order to lower bound $\sigma_{min}(\wt B_S)$, we first analyze the structure of this coefficient matrix.
 The matrix $\wt B_S$ has the following block structure:
 \begin{align*}
   \wt B_S = \lt[\wt B^{(0)},\wt B^{(1)}, \wt B^{(2)},\wt B^{(3)}\rt].
 \end{align*}
 The first block $\wt B^{(0)}\in\R^{(|\mc H|/3)^3 \times k}$ is a summation of four matrices $\wt
 B^{(0)}_i$ for $i=0, 1,2,3$, where $\wt B^{(0)}_0 = \wt\mu_{\mc H^{(3)}}\od \wt\mu_{\mc H^{(2)}}\od
 \wt\mu_{\mc H^{(1)}}$, and $\wt B^{(0)}_1=  \wt\Sigma_{\mc H^{(3)},\mc
   H^{(2)}} \od \wt\mu_{\mc H^{(1)}}$. With some fixed and known row permutation $\pi^{(2)}$ and $\pi^{(3)}$, the other two
 matrix blocks $\wt B^{(0)}_2$ and $\wt B^{(0)}_3$ are equal to $
 \wt\Sigma_{\mc H^{(3)},\mc H^{(1)}}\od \wt\mu_{\mc H^{(2)}}$ and $ \wt\Sigma_{\mc H^{(2)},\mc
   H^{(1)}}\od  \wt\mu_{\mc H^{(3)}}$, separately.

 The block $\wt B^{(1)}\in\R^{(|\mc H|/3)^3 \times k|\mc H|/3 }$ is block diagonal with the identical block $
 \wt\Sigma_{\mc H^{(3)}, \mc H^{(2)}}+ \wt\mu_{\mc H^{(3)}} \od \wt\mu_{\mc H^{(2)}}$. Similarly, with the row
 permutation $\pi^{(2)}$, $\pi^{(3)}$, the other two matrix blocks $\wt B^{(2)},\wt B^{(3)}$ are equal to the block
 diagonal matrices with the identical block $( \wt\Sigma_{\mc H^{(3)}, \mc H^{(1)}}+ \wt\mu_{\mc H^{(3)}} \od
 \wt\mu_{\mc H^{(1)}})$ and $( \wt\Sigma_{\mc H^{(2)}, \mc H^{(1)}}+ \wt\mu_{\mc H^{(2)}} \od \wt\mu_{\mc H^{(1)}})$
 respectively.

 Note that we can write the block $\wt B^{(0)  }$ as:
 \begin{align*}
   \wt B^{(0)} =& (\wt \mu_{\mc H^{(3)}} \od\wt\mu_{\mc H^{(2)}} +
   \wt\Sigma_{\mc H^{(3)},\mc H^{(2)}}) \od \wt\mu_{\mc H^{(1)}}
   +(\pi^{(2)})^{-1 } (\wt\mu_{\mc H^{(3)}}\od \wt\mu_{\mc H^{(1)}} +
   \wt\Sigma_{\mc H^{(3)},\mc H^{(1)}}) \od \wt \mu_{\mc H^{(2)}}
   \\ &+ (\pi^{(3)})^{-1 } (\wt\mu_{\mc H^{(2)}}\od \wt\mu_{\mc H^{(1)}} +
   \wt\Sigma_{\mc H^{(2)},\mc H^{(1)}}) \od  \wt \mu_{\mc H^{(3)}} - 2\wt \mu_{\mc H^{(3)}} \od \wt\mu_{\mc H^{(2)}}\od \wt\mu_{\mc H^{(1)}},
 \end{align*}
 where it is easy to see the first summand $(\wt \mu_{\mc H^{(3)}} \od\wt\mu_{\mc H^{(2)}} +
 \wt\Sigma_{\mc H^{(3)},\mc H^{(2)}}) \od \wt\mu_{\mc H^{(1)}}$ is a linear combination of the
 columns of the block diagonal matrix $\wt B^{(1)}$, and similarly the second and third summands are
 linear combinations of the columns of $\wt B^{(2)}$ and $\wt B^{(3)}$, and the last summand is
 simply $-2\wt B_0^{(0)}$. Therefore for some absolute constant $C$ (the smallest singular value
 corresponding to the linear transformation) we have that:
 \begin{align*}
   \sigma_{min} (\wt B_S)  \ge C\sigma_{min }( \lt[  \wt B_0^{(0)}, \wt B^{(1)}, \wt B^{(2)},\wt B^{(3)}\rt] )
 \end{align*}

 \medskip

 Note that $\wt B_0^{(0)} = \wt\mu_{\mc H^{(3)}}\od \wt\mu_{\mc H^{(2)}}\od \wt\mu_{\mc
   H^{(1)}}$ only depends on the randomness over the mean vectors. Note that the
 Khatri-Rao product is a submatrix of the Kronecker product, therefore for tall matrices
 $Q_1$ and $Q_2$, we have that $\sigma_{min}(Q_1\od Q_2) \le \sigma_{min}(Q_1\ot_{kr} Q_2)
 = \sigma_{min}(Q_1) \sigma_{min}( Q_2) $. In particular, we can bound the smallest
 singular value of $\wt B_0^{(0)}$ with high probability (at least $1-C\epsilon^{0.5n }$)
 as follows:
 \begin{align*}
   \sigma_k( \wt B_0^{(0)} ) \ge \sigma_k(  \wt\mu_{\mc H^{(3)}}) \sigma_k( \wt\mu_{\mc H^{(2)}} )
   \sigma_k(\wt\mu_{\mc H^{(1)}})
   \ge( \rho\epsilon \sqrt{n})^3.
 \end{align*}
 Then condition on the value of the means, we further exploit the randomness over the covariance matrices to
 lower bound $\sigma_{k|\mc H|}\lt( \prj_{ \wt B_0^{(0)\perp} }[ \wt B^{(1)}, \wt B^{(2)},\wt
 B^{(3)}] \rt) $.
 It is almost the same as the argument of the proof for Proposition~\ref{prop:find-S}. For example,
 compared to \eqref{eq:bound-k-prj-rand-gaussian} we have the following inequality instead:
 \begin{align*}
   \sigma_{k} \lt(\prj_{([\wt B^{(0)},\wt B^{(2)},\wt B^{(3)}]_{\{j\}\times \mc H^{(2)}\times \mc H^{(3)}  })^\perp}
   \prj_{(\Sigma_{\mc H^{(2)}, \mc H^{(3)}}+ \wt\mu_{\mc H^{(2)}} \od \wt\mu_{\mc H^{(3)}}) ^\perp}(\wt \Sigma_{\mc
     H^{(2)}, \mc H^{(3)}}+ \wt\mu_{\mc H^{(2)}} \od \wt\mu_{\mc H^{(3)}} )\rt) &\ge \epsilon\rho\sqrt{n},
 \end{align*}
 and note that any block in $\wt B^{(0)}$ is independent of the randomness of covariance matrices,
 and we have $(|\mc H|/3)^2-k-2k|\mc H|/3 \ge 2k$. Similar modifications apply to the inequalities in
 \eqref{eq:bound-k-prj-rand-gaussian-2},\eqref{eq:bound-k-prj-rand-gaussian-3}.

 Finally by the argument of Lemma~\ref{lem:prjdiag} we can bound $\sigma_{min}(\wt B_S)$ with
 probability at least $1-C\epsilon^{0.5n}$ (over the randomness of both the perturbed means and
 covariance matrices):
 \begin{align*}
   \sigma_{min}(\wt B_S) \ge \min \{  ( \rho\epsilon \sqrt{n})^3, \epsilon\rho\sqrt{n} \} = \epsilon\rho\sqrt{n},
 \end{align*}
as we assume $\rho$ to be small perturbation and  $\rho\epsilon \sqrt{n}<1$.

\end{proof}

\paragraph{Step 1 (b). Find the projected span of covariance matrices}
Given the subspace $\mc S = span\{\wt\mu^{(i)}, \wt\Sigma_{[:,\mc H]}^{(i)}: i\in[k]\}$ obtained from Step 1 (a), Step
1(b) finds the span of the covariance matrices with the columns projected to $S^\perp$, namely:
\begin{align*}
  \mc U_S = span\{ \prj_{S^\perp} \wt \Sigma^{(i)} :i\in[k] \}.
\end{align*}

This is in parallel with Step 1 (b) for the zero-mean case, and we rely on the structure of the two-dimensional slices
of $\wt M_4$ to find the span of the projected covariance matrices. Similar to Claim~\ref{claim:projected34} for the
zero-mean case, the following claim shows how the structure of the two-dimensional slices is related to the desired
span.

\begin{claim}
\label{claim:M4-struct-general}
 For a mixture of general Gaussians, the two-dimensional slices of $\wt M_4$ are given by:
  \begin{align*}
    \wt M_4(\mb e_{j_1},\mb e_{j_2},I,I) =& \sum_{i=1}^k \wt \omega_i \Big( (\wt\Sigma^{(i)}_{j_1,j_2} +
    \wt\mu_{j_1}^{(i)}(\wt\mu_{j_2}^{(i)})^\top) (\wt\Sigma^{(i)} + \wt\mu^{(i)} (\wt\mu^{(i)})^\top)
    \\
    &+ \wt\mu_{j_1}^{(i)} (\wt\mu^{(i)} (\wt\Sigma^{(i)}_{[:,j_2]})^\top + \wt\Sigma^{(i)}_{[:,j_2]}(\wt\mu^{(i)} )^\top
    ) +\wt\mu_{j_2}^{(i)} (\wt\mu^{(i)} (\wt\Sigma^{(i)}_{[:,j_1]})^\top + \wt\Sigma^{(i)}_{[:,j_1]}(\wt\mu^{(i)} )^\top
    )
    \\
    &+ \wt\Sigma^{(i)}_{[:,j_1]} ( \wt\Sigma^{(i)}_{[:,j_2]})^\top+ \wt\Sigma^{(i)}_{[:,j_2]} (
    \wt\Sigma^{(i)}_{[:,j_1]})^\top \Big),
    \hfill\quad\fa j_1,j_2\in [n].
  \end{align*}
\end{claim}

Note that given the set of indices $\mc H$ we chose in Step 1 (a) and the subspace $S$, if we pick the indices
$j_1,j_2\in\mc H$, project the two-dimensional slice to $S^{\perp}$, all the rank one terms in the sum are eliminated
and the projected slice  lies in  the desired span $\mc U_S$:
\begin{align*}
  \prj_{S^\perp} \wt M_4(\mb e_{j_1},\mb e_{j_2},I,I) = \sum_{i=1}^k\wt \omega_i
  (\wt\Sigma^{(i)}_{j_1,j_2} + \wt\mu_{j_1}^{(i)}(\wt\mu_{j_2}^{(i)})^\top)
  \prj_{S^\perp}\wt\Sigma^{(i)}, \quad \fa j_1, j_2\in\mc H.
\end{align*}

Applying the same argument as in Lemma~\ref{prop:bound-sig-QUs} for the zero-mean case, we can show that
with high probability over the perturbation, all the projected slices span the subspace $\mc U_S$.



\paragraph{Step 1 (c). Merge the two projections of covariance matrices}
Pick two disjoint index set $\mc H_1$ and $\mc H_2$ and repeat the previous two steps 1 (a) and 1 (b), we can obtain the
two spans $U_1$ and $U_2$, corresponding to the subspace of the covariance matrices projected to $\mc S_1$ and $\mc S_2$,
respectively.

In this step, we apply similar techniques as in Step 1 (c) for the zero-mean case to merge the two spans $U_1$ and
$U_2$: we first use the overlapping part of the two projections $\prj_{ S_1^\perp}$ and $\prj_{ S_2^\perp}$ to align the
basis of $U_1$ and $U_2$, then merge the two spans using the same basis.

Note that for the general case, by definition the span of the mean vectors $\wt Z$ lie in both subspaces $\mc S_1$ and $\mc
S_2$, therefore we have $\mc S_1^\perp\subset \wt Z^\perp$ and $\mc S_2^\perp\subset \wt
Z^\perp$. We can show that $\mc S_1^\perp \cup \mc S_2^\perp = \wt Z^\perp$ by lower
bounding $\sigma_{n-k}([\prj_{\mc S_1^\perp}, \prj_{\mc S_2^\perp}])$  with high
probability, similar to that in \eqref{eq:bound-s1-s2-sv}. This gives us the span of the
mean vectors $\wt Z$.

Moreover, in the general case, from merging $U_1$ and $U_2$ we are only able to find the
span of covariance matrices projected to the subspace $\wt Z^\perp$.
In particular, we can follow Lemma~\ref{lem:merge-deterministic} and
Lemma~\ref{prop:perturb-QU} in Step 1 (c) for the zero-mean case to show that for the
general case, we can merge $U_1$ and $U_2$ to obtain the span $span\{ \prj_{\wt Z^\perp}
\wt \Sigma^{(i)} :i\in[k]$. By further projecting the span to $\wt Z^\perp$ from the right
side, we can also obtain $\wt\Sigma_o = span\{ \prj_{\wt Z^\perp} \wt
\Sigma^{(i)}\prj_{\wt Z^\perp}: i\in[k]\}$.

\paragraph{Step 2. Find the covariance matrices in the subspace orthogonal to the means }

Given the subspace $\wt Z$ and $\wt\Sigma_o = span\{ \prj_{\wt Z^\perp} \wt
\Sigma^{(i)}\prj_{\wt Z^\perp}: i\in[k]\}$ obtained from Step 1, Step 2 applies the
zero-mean case algorithm to find the covariance matrices projected to the subspace $\wt
Z^\perp$, i.e., $ \prj_{\wt Z^\perp} \wt \Sigma^{(i)}\prj_{\wt Z^\perp}$'s, as well as
find the mixing weights $\wt \omega_i$'s.

This follows the same arguments as in Step 2 and Step 3 for the zero mean case.
Consider projecting all the samples to $\wt Z^\perp$, the subspace orthogonal to all the
means. In this subspace, the samples are like from a mixture of zero-mean Gaussians with
the projected covariance matrices, and the $4$-th and $6$-th order moment are given by
$\wt M_4(\prj_{\wt Z^\perp},\prj_{\wt Z^\perp},\prj_{\wt Z^\perp},\prj_{\wt Z^\perp})$ and
$\wt M_6(\prj_{\wt Z^\perp},\prj_{\wt Z^\perp},\prj_{\wt Z^\perp},\prj_{\wt
  Z^\perp},\prj_{\wt Z^\perp},\prj_{\wt Z^\perp})$. Since $\wt Z$  is of dimension $k$,
the dimension of the zero-mean Gaussian in the projected space is at least $n-k= O(n)$.

Note that the subspace $\wt Z^{\perp}$ only depends on the randomness of the means, and
random perturbation on the covariance matrices is independent of that of $\wt \mu$.  The
smoothed analysis for the moment unfolding in Step 2 and tensor decomposition in Step 3
for the zero-mean case, which only depend on the randomness of the covariance matrices,
still go through in the projected space.

\paragraph{Step 3. {Find the means}}
This step finds the mean vectors based on the outputs of the previous steps.  The key
observation for this step is about the structure of the 3-rd order moments in the
following claim:
\begin{claim}
  \label{claim:M3-proj}
  Let the matrix $\wt M_{3(1)}\in\R^{n\times n^2}$ be the matricization of $\wt M_3$ along
  the first dimension.  The $j$-th row of $\wt M_{3(1)}$ is given by:
  \begin{align}
    [ \wt M_{3(1)} ]_{[j,:]} &= \Big[[\mbb E[x_{j} x_{j_1}x_{j_2}]: j_1\in[n]]:j_2\in[n]\Big]
    \nonumber
    \\
    \label{eq:M3-proj}
    &= \sum_{i=1}^{k} \wt\omega_i\lt( \wt\mu_j^{(i)} \vc(\wt\Sigma^{(i)}) + \wt\mu_j^{(i)} \wt\mu^{(i)} \od
    \wt\mu^{(i)} + \wt\Sigma^{(i)}_{[:,j]}\od \wt\mu^{(i)} + \wt\mu^{(i)} \od \wt\Sigma^{(i)}_{[:,j]} \rt)^\top
  \end{align}
\end{claim}

The following lemma shows how to extract the means $\wt\mu^{(i)}$'s from $\wt M_{3(1)}$
using the information of the covariance matrices projected to the subspace orthogonal to
the means, i.e. $\wt \Sigma_o$, and the mixing weights $\wt \omega_i$'s.
\begin{lemma}
  \label{lem:find-mean-M3}
  Given the mixing weights $\wt\omega_i$'s and the projected covariances $\wt \Sigma^{(i)}_o$'s, define the
  matrix $\wt T\in\R^{n^2\times k}$ to be the pseudo-inverse of ${\wt\Sigma}_o$:
  \begin{align*}
   \wt T = \lt[\vc(\wt \Sigma^{(i)}_o):i\in[k]\rt]^{\dag\top}.
  \end{align*}
  The mean $\wt \mu^{(i)}$ of the $i$-th component can be obtained by:
  \begin{align*}
   \wt \mu^{(i)} = {1\over \wt \omega_i} \wt M_{3(1)} \wt T_{[:,i]} .
 \end{align*}
 This step correctly finds the means if the $\wt\Sigma_o$ is full rank with good
 condition  number, and this holds with high probability over the  perturbation.
\end{lemma}

\begin{proof}
  (of Lemma~\ref{lem:find-mean-M3} )

  The basic idea is that since $\wt \Sigma_o $ lies in the span of $\wt P =\prj_{\wt
    Z^\perp} \ot_{kr} \prj_{\wt Z^\perp} $, and the last three summands in the parenthesis
  in \eqref{eq:M3-proj} all lie in $span\{I_{n}\ot_{kr}\prj_{\wt Z},\ \prj_{\wt Z}\ot_{kr} I_{n}\}
  = span\{\wt P^\perp\}$. Therefore hitting the matrix $\wt M_{3(1)}$ with $\wt
  \Sigma_o^\dagger$ from the right will eliminate those summands and pull out only the
  mean vectors.

  Recall that the columns of the matrix $\wt \Sigma_o$ are $\vc(\prj_{\wt Z^\perp} \wt
  \Sigma^{(i)}\prj_{\wt Z^\perp})= \wt P \vc(\wt \Sigma^{(i)}) $'s, and the columns of
  $\wt \Sigma$ are $\vc(\wt \Sigma^{(i)})$'s.

  Note that $\wt T =( \wt P{\wt\Sigma})^{\dag\top} = \wt P{\wt\Sigma}^{\dag\top}$, and the
  columns of $\wt T$ lie in $span\{\wt P\}$.
  Also note that for all $i,j\in[k]$ the vectors $ \wt\mu^{(i)} \od \wt\mu^{(i)}$,
  $\wt\Sigma^{(i)}_{[:,j]}\od \wt\mu^{(i)}$ and $\wt\mu^{(i)} \od \wt\Sigma^{(i)}_{[:,j]}
  $ all lie in the subspace $span\{I_{n}\ot_{kr}\prj_{\wt Z}, \ \prj_{\wt Z}\ot_{kr}
  I_{n}\}=span\{\wt P^\perp\}$. Therefore these terms will be eliminated if we multiply
  the columns of $\wt T$ to the right of $\wt M_{3(1)}$.
  For the first term $\wt \mu_j^{(i)}\vc(\wt \Sigma^{(i)})$, since $
  \vc(\wt\Sigma^{(j)})^\top \wt T_{[:,i]} = (\wt P\vc(\wt \Sigma^{(j)}))^\top \wt
  T_{[:,i]} = 1_{[i=j]}$.
  Therefore, we have $\wt M_{3(1)}\wt T_{[:,i]} = \wt\omega_i\wt\mu^{(i)}$.

  The smoothed analysis for the correctness of this step is easy. We only need to show
  that both ${\wt\Sigma}_o$ and ${\wt\Sigma}$ robustly have full column rank with high
  probability over perturbation of the covariance matrices, and thus the pseudo-inverse
  $\wt T$ is well defined. This follows from Lemma~\ref{lem:prj-rand-gaussian}.

  Finally, the stability analysis for this step is also straightforward using the
  perturbation bound for pseudo-inverse in Theorem~\ref{sec:pert-bound-pseudo}.
\end{proof}

\paragraph{Step 4. {Find the unprojected  covariance matrices}}

Note that by definition $\wt Z = span\{ \wt \mu^{(i)}:i\in[k]\}$, the projected covariance
$\prj_{\wt Z^\perp}(\wt \Sigma^{(i)})$ we obtained in Step 2 is also equal to $\prj_{\wt
  Z^\perp}(\wt \Sigma^{(i)}+ \wt \mu^{(i)}(\wt \mu^{(i)})^\top)$.
In Step 4 we try to recover the missing part of the covariance matrices in the subspace
$\wt Z$. Note that since we have also obtained the means in Step 3, it is equivalent to
finding $ (\wt\Sigma^{(i)} + \wt \mu^{(i)} (\wt\mu^{(i)})^\top)$ for all $i$.
We will show that if we can find the $span\{ (\wt\Sigma^{(i)} + \wt \mu^{(i)}
(\wt\mu^{(i)})^\top):i\in[k]\}$, the projected vector $\prj_{\wt Z^\perp}(\wt
\Sigma^{(i)}+ \wt \mu^{(i)}(\wt \mu^{(i)})^\top)$ can be used as anchor to pin down the
unprojected vector.
%
%

They key observation for finding the span of $span\{ (\wt\Sigma^{(i)} + \wt \mu^{(i)}
(\wt\mu^{(i)})^\top):i\in[k]\}$ is to first construct a 4-th order tensor $\wt M'_4$ which
corresponds to the 4-th moment of a mixture of zero-mean Gaussians with covariance
matrices $ (\wt\Sigma^{(i)} + \wt \mu^{(i)} (\wt\mu^{(i)})^\top)$, and then follow Step 1
in the algorithm for zero-mean case to find the span of the covariance matrices for this
new mixture of Gaussians.

The next lemma shows how to construct such 4-th order tensor:
\begin{lemma}
  Given the $4$-th moment $\wt M_4$ for a mixture of Gaussians with parameters $\{\wt
  \omega_i, \wt \mu^{(i)}, \wt \Sigma^{(i)}\}$, define the 4-th order tensor $\wt M'_4$ to
  be:
\begin{align*}
  \wt M'_4 = \wt M_4 + 2\sum_{i=1}^k \wt \omega_i \wt \mu^{(i)}\ot^4,
\end{align*}
then $\wt M'_4$ is equal to the $4$-th moment of a mixture Gaussians with parameters
$\{\wt \omega_i, 0, \wt \Sigma^{(i)} + \wt \mu^{(i)}(\wt \mu^{(i)})^\top\}$.
\end{lemma}

The proof follows directly from Isserlis' Theorem.  Therefore we can repeat Step 1 in the
zero-mean case here to find the span of the space $\{\vc(\wt \Sigma^{(i)})+\wt
\mu^{(i)}\od \wt \mu^{(i)}:i\in [k]\}$. Since we also know the projection of $\wt
\Sigma^{(i)}$'s in a large subspace (in the subspace $\prj_{\wt Z^\perp} \ot_{kr} \prj_{\wt
  Z^\perp} $ obtained from Step 2),
we can easily recover $\wt \Sigma^{(i)}$'s:

\begin{lemma}
For any matrix $U\in \R^{d\times k}$ and any subspace $P$, given $P^\top U$ and the span
$ S$ of columns of $U$, the matrix $U$ can be computed as
\begin{align*}
  U =  S (P^\top  S)^\dag (P^\top U).
\end{align*}
Further, this procedure is stable if $\sigma_{min}(P^\top S)$ is lower bounded.
\end{lemma}

\begin{proof}
  This is a special case of the Step 1 (c) where we merge two projections of an unknown subspace.

The span $ S$ is equal to $UV$ for some unknown matrix $V$. We can compute $V = (P^\top U)^\dag P^\top  S$, and hence $U =  SV^{-1} =  S (P^\top S)^\dag (P^\top U)$. The stability analysis is similar (and simpler than) Lemma~\ref{lem:merge-deterministic}.
\end{proof}

We will apply this lemma to where the subspace $P$  is $\prj_{\wt Z^\perp} \ot_{kr}
\prj_{\wt Z^\perp}$. Since the perturbation of the means and the covariance matrices are
independent, we can lower bound the smallest singular value of $P^\top S$.

\subsection{Proof Sketch of the Main Theorem~\ref{thm:main}}

\begin{figure}
\centering
\includegraphics[width =7in]{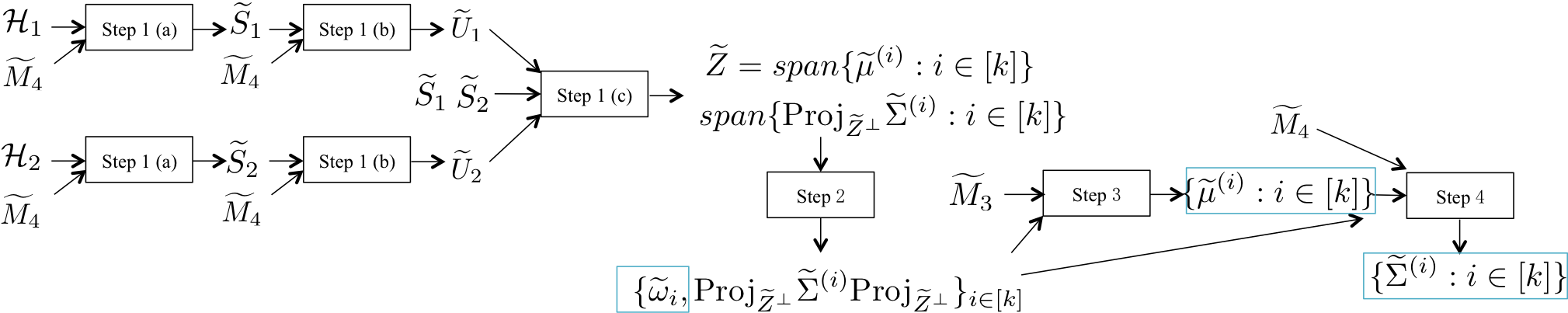}
\caption{Flow of the algorithm for the general case}\label{fig:flow-general}
\end{figure}

The proof follows the same strategy as Theorem~\ref{thm:main-zero-mean}. First we apply the union bound to all the smoothed analysis lemmas, this will ensure the matrices we are inverting all have good condition number, and the whole algorithm is robust to noise.

Then in order to get the desired accuracy $\epsilon$, we need to guarantee inverse polynomial accuracy in different steps (through the stability lemmas). The flow of the algorithm is illustrated in Figure~\ref{fig:flow-general}. In the end all the requirements becomes a inverse polynomial accuracy requirement on $\wh M_4$ and $\wh M_6$, which we obtain by Lemma~\ref{lem:concen-empirical}.

\section{Matrix Perturbation, Concentration Bounds and Auxiliary Lemmas}
\label{sec:auxil-lemm-proofs}

In this section we collect known results on matrix perturbation and concentration bounds. In general, matrix perturbation bounds are the key for the perturbation lemmas, and concentration bounds are crucial for the smoothed analysis lemmas. We also prove some corollaries of known results that are very useful in our settings.

\subsection{Matrix Perturbation Bounds}

Given a matrix $\wh{A} = A + E$ where $E$ is a small perturbation, how does the singular values and singular vectors of $A$ change? This is a well-studied problem and many results can be found in Stewart and Sun~\cite{stewart1977perturbation}. Here we review some results used in this paper, and prove some corollaries.

Given $\wh{A} = A + E$, the perturbation in individual singular values can be bounded by Weyl's theorem:

\begin{theorem}[Weyl's theorem]\label{thm:weyl}
Given $\wh{A} = A+E$, we know $\sigma_k(A) - \|E\| \le \sigma_k(\wh{A}) \le \sigma_k(A) + \|E\|$.
\end{theorem}

We can also bound the $\ell_2$ norm change in singular values by Mirsky's Theorem.

\begin{lemma}[Mirsky's theorem]
  \label{lem:perturb-singv}
  Given matrices $A, E\in\R^{m\times n }$ with $m\ge n$, then
  \begin{align*}
    \sqrt{\sum_{i=1}^{n}  (\sigma_{i}(A+E)-\sigma_i(A))^2}\le \|E\|_F.
  \end{align*}
\end{lemma}

For singular vectors, the perturbation is bounded by Wedin's Theorem:

\begin{lemma}[Wedin's theorem; Theorem 4.1, p.260 in \cite{stewart1990matrix}]
  \label{lem:perturb-subspace}
  Given matrices $A, E\in\R^{m\times n }$ with $m\ge n$. Let $A$ have the singular value
  decomposition
  \begin{align*}
    A = [U_1, U_2, U_3]\left[
    \begin{array}[c]{cc}
      \Sigma_1 & 0 \\ 0 & \Sigma_2 \\ 0 & 0
    \end{array}
  \right]
  [V_1, V_2]^\top.
  \end{align*}
  Let $\wh A = A + E$, with analogous singular value decomposition. Let $\Phi$ be the matrix of
  canonical angles between the column span of $U_1$ and that of $\wh U_1$, and $\Theta$ be the matrix of
  canonical angles between the column span of $V_1$ and that of $\wh V_1$. Suppose that there exists
  a $\delta$ such that
  \begin{align*}
    \min_{i,j}|[\Sigma_{1}]_{i,i}-[\Sigma_2]_{j,j}|>\delta,\quad \tx{and } \quad \min_{i,i}|[\Sigma_1]_{i,i}|> \delta,
  \end{align*}
  then
  \begin{align*}
    \|\sin(\Phi)\|^2 + \|\sin(\Theta)\|^2 \le 2{\|E\|^2\over \delta^2}.
  \end{align*}
\end{lemma}

We do not go into the definition of canonical angles here. The only way we will be using this lemma is by combining it with the following:

\begin{lemma}[Theorem 4.5, p.92 in \cite{stewart1990matrix}]
  \label{lem:prj-to-angle}
  Let $\Phi$ be the matrix of canonical angles between the column span of $U$ and that of $\wh U$,
  then
  \begin{align*}
    \|\prj_{\wh U} - \prj_{U}\| = \|\sin\Phi\|.
  \end{align*}
\end{lemma}

As a corollary, we have:

\begin{lemma}
  \label{cor:perturb-prj-bound}
  Given matrices $A, E\in\R^{m\times n }$ with $m\ge n$. Suppose that the $A$ has rank $k$ and the
  smallest singular value is given by $\sigma_k(A)$.  Let $\mc S$ and $\wh{ \mc S}$ be the subspaces
  spanned by the first $k$ eigenvectors of $A$ and $\wh A = A +E$, respectively. Then we have:
  \begin{align*}
    \|\wh S - \wt S\|\le
     \|\prj_{\wh{\mc
        S}}-\prj_{\mc S}\| = \|\prj_{\wh{\mc S}^\perp}-\prj_{\mc S^\perp}\| \le { \sqrt{2}\|E\|_F \over\sigma_k(A)}.
  \end{align*}
  Moreover, if $\|E\|_F\le {\sigma_k(A) / \sqrt{2}}$ we have
  $ \|\wh S-\wt S\| \le { \sqrt{2}\|E\| \over\sigma_k(A)}.  $
\end{lemma}

\begin{proof}
We first prove the first inequality:
  \begin{align*}
    \|\prj_{\wh{\mc S}}-\prj_{\mc S}\| = \|2\wt S(\wh S-\wt S)^\top + (\wh S-\wt S)(\wh
    S-\wt S)^\top \| \ge 2 \|\wt S\| \|\wh S-\wt S\| - \|\wh S-\wt S\|^2 \ge \|\wt S\|
    \|\wh S-\wt S\| = \|\wh S-\wt S\|.
\end{align*}

The equality is because $\prj_{S^\perp} = I - \prj_{S}$ so the two differences are the same. The final step follows from Wedin's Theorem and Lemma~\ref{lem:prj-to-angle}.
\end{proof}

Often we need to bound the perturbation of a product of perturbed matrices, where we apply
the following lemma:
\begin{lemma}
  \label{lem:prod-perturb}
  Consider a product of matrices $A_1\cdots A_k$, and consider any sub-multiplicative norm on matrix $\|\cdot\|$.
  Given $\wh A_1,\dots, \wh A_k$ and assume that $\|\wh A_i-A_i\|\le \|A_i\|$, then we have:
  \begin{align*}
    \|\wh A_1\cdots \wh A_k - A_1\cdots A_k\| \le 2^{k-1} \prod_{i=1}^k \|A_i\|  \sum_{i=1}^{k} {\|\wh A_i-A_i\|  \over \|A_i\|}.
  \end{align*}

\end{lemma}
The proof of this lemma is straightforward by induction.

\paragraph{Perturbation bound for pseudo-inverse}

When we have a lowerbound on $\sigma_{min}(A)$, it is easy to get bounds for the perturbation of pseudoinverse.

\begin{theorem}[Theorem 3.4 in \cite{stewart1977perturbation}]
  \label{sec:pert-bound-pseudo}
Consider the perturbation  of a matrix $A\in\R^{m\times n}$: $B = A+E$. Assume that $rank(A)
=rank(B)= n$, then
\begin{align*}
  \|B^\dag - A^\dag\|\le \sqrt{2}\|A^\dag\|\|B^\dag\| \|E\|.
\end{align*}
\end{theorem}

As a corollary, we often use:

\begin{lemma}
\label{lem:perturb-pseudo-inverse}
Consider the perturbation  of a matrix $A\in\R^{m\times n}$: $B = A+E$ where $\|E\| \le \sigma_{min}(A)/2$. Assume that $rank(A)
=rank(B)= n$, then
\begin{align*}
  \|B^\dag - A^\dag\|\le 2\sqrt{2} \|E\|/\sigma_{min}(A)^2.
\end{align*}
\end{lemma}

\begin{proof}
We first apply Theorem~\ref{sec:pert-bound-pseudo}, and then bound $\|A^\dag\|$ and $\|B^{\dag}\|$. By definition we know $\|A^\dag\| = 1/\sigma_{min}(A)$.  By Weyl's theorem $\sigma_{min}(B) \ge \sigma_{min}(A) - \|E\| \ge \sigma_{min}(A)/2$, hence $\|B^\dag\| = \sigma_{min}(B)^{-1} \le 2\sigma_{min}(A)^{-1}$.
\end{proof}

\subsection{Lowerbounding the Smallest Singular Value}

Gershgorin's Disk Theorem is very useful in bounding the singular values.

\begin{theorem}[Gershgorin's theorem]
  \label{thm:gershgorin}
  Given a symmetric matrix $X\in\mbb R^{k\times k}$, a lower bound on the smallest eigenvalue is
  given by:
  \begin{align*}
    \sigma_{min}(X) \ge \min_{i\in[k]} \left\{X_{i,i} - \sum_{j\in[k], j\neq i} X_{i,j} \right\}.
  \end{align*}
\end{theorem}

Sometimes, it is easier to consider the projection of a matrix. Lowerbounding the smallest singular value of a projection will imply the same lowerbound on the original matrix:

\begin{lemma}
\label{lem:sing-projection}
Suppose $A\in \R^{m\times n}$, let $P \in \R^{m\times d}$ be a subspace, then $\sigma_k(P^\top A) \le \sigma_k(A)$.
\end{lemma}

\begin{proof}
Observe that $(P^\top A)^\top(P^\top A) = A^\top (PP^\top) A \preceq A^\top A$ (because $P$ is a subspace). Therefore the eigenvalues of $(P^\top A)^\top(P^\top A)$ must be dominated by the eigenvalues of $A^\top A$. Then the lemma follows from the definition of singular values.
\end{proof}

As a corollary we have the following lemma:

\begin{lemma}
  \label{lem:sigv-prj}
  Let $A\in\R^{m\times n}$ and suppose that $m\ge n$. For any projection $\prj_{S}$, we
  have that the singular values are non-increasing after the projection:
  \begin{align*}
    \sigma_{i} (\prj_S(A)) \le \sigma_i( A), \quad \tx{for} i=1,\dots, n.
  \end{align*}

\end{lemma}

In several places of this work we want to bound the singular value of a matrix, where part of the matrix has a block structure.

\begin{lemma}
  \label{lem:prjdiag}
  For given matrices $B^{(i)}\in\R^{m\times n}$ and $C^{(i)}\in\R^{m\times n'}$ for
  $i=1,\dots,d$. Suppose $md > (n+n'd)$,  Define the tall matrix $A\in\R^{md\times
    (n+dn')}$:
  \begin{align*}
    A = \lt[
    \begin{array}[c]{ccccc}
      B^{(1)}& C^{(1)} & 0 & \cdots & 0
      \\
      B^{(2)} & 0 & C^{(2)} & \cdots & 0
      \\
      \vdots & \vdots & \vdots & \ddots & \vdots
      \\
      B^{(d)} & 0 & 0 &\cdots &C^{(d)}
    \end{array}
    \rt] = \lt[B,\diag(C^{(i)})\rt].
  \end{align*}
  The smallest singular value is bounded by:
  \begin{align*}
    \sigma_{(n+dn')}(A) \ge \min\{ \sigma_{n}(B), \
    \sigma_{n'} (\prj_{(B^{(i)})^\perp}
    C^{(i)}):i=1,\dots, d\}.
  \end{align*}
\end{lemma}
\begin{proof}
The idea is to break the matrix into two parts $A = \prj_B A + \prj_{B^\perp} A$.Since these two spaces are orthogonal we know $\sigma_{(n+dn')}(A) \ge \min\{\sigma_n(\prj_B A) ,\sigma_{dn'}(\prj_{B^\perp} A)\}$.

For the first part, clearly $\sigma_n(\prj_B A) \ge \sigma_n(B)$, as $B$ is a submatrix of $\prj_B A$.

For the second part, we actually do the projection to a smaller subspace: for each block we project to the orthogonal subspace of $B^{(i)}$. Under this projection, the block structure is preserved. The $dn'$-th singular value must be at least the minimum of the $n'$-th singular value of the blocks. In summary we have:
  \begin{align*}
    \sigma_{(n+dn')}(A) &\ge \min\{ \sigma_{n}(B), \ \sigma_{dn'}( \prj_{B^{\perp}} \diag(C^{(i)})) \}
    \\
    &\ge  \min \{ \sigma_{n}(B), \ \sigma_{dn'}( \prj_{\diag((B^{(i)})^{\perp})} \diag(C^{(i)})) \}
    \\
    &\ge \min\{\sigma_n(B), \ \sigma_{dn'}( \diag(\prj_{(B^{(i)})^{\perp}} C^{(i)}))\}
    \\
    &\ge \min\{\sigma_n(B), \ \sigma_{n'}( \prj_{(B^{(i)})^{\perp}} C^{(i)}):i=1,\dots,d\}.
  \end{align*}
\end{proof}

\paragraph{Smallest singular value of random matrices}

In our analysis, we often also want to bound the smallest singular value of a matrix whose entries
are Gaussian random variables. Our analysis mostly builds on the following results in random matrix
theory.

For a random rectangular matrix, \cite{rudelson2009smallest} gives the following nice result:

\begin{lemma}[Theorem 1.1 in \cite{rudelson2009smallest}]
  \label{lem:rudelson}
  Let $A\in\mbb R^{m\times n}$ and suppose that $m\ge n$. Assume that the entries of $A$ are independent
  standard Gaussian variable, then for every $\epsilon>0$, with probability at least $1-(C\epsilon)^{m-n+1} +
  e^{-C'n}$, where $C,C'$ are two absolute constants, we have:
  \begin{align*}
    \sigma_{n}(A)\ge \epsilon(\sqrt{m}-\sqrt{n-1}).
  \end{align*}
\end{lemma}
We will mostly use an immediate corollary of the above lemma with slightly simpler form:
\begin{corollary}
  Let $A\in\mbb R^{m\times n}$ and suppose that $m\ge 2n$. Assume that the entries of $A$ are independent
  standard Gaussian variable, then for every $\epsilon>0$, and for some absolute constant $C$, with probability at least $1-(C\epsilon)^{0.5m}$, we have:
  \begin{align*}
    \sigma_{n}(A)\ge \epsilon\sqrt{m}.
  \end{align*}
\end{corollary}

This lemma can also be applied to a projection of a Gaussian matrix:

\begin{lemma}
  \label{lem:prj-rand-gaussian}
  Given a Gaussian random matrix $E\in\R^{m\times n}$, for some set $\mc J\in[m]$ define
  $E_J = [E_{[j,:]}:j\in\mc J]$ and $E_{J^c} = [E_{[j,:]}:j\in[m]/\mc J]$.  Define matrix
  $S\in\R^{n\times r}$ whose columns are orthonormal. Suppose that the matrix $S$ is an
  arbitrary function of $E_J$ and is independent of $E_{J^c}$.  Assume that
  \begin{align}
    \label{eq:cond-prj-rand-g}
    m-|\mc J| - r \ge 2n
  \end{align}
  Then for any $\epsilon>0$, we have that with probability at least $1-(C\epsilon)^{0.5(m-|\mc J| - r )}$, for
  some absolute constant $C$, the smallest singular value of the projected random matrix is bounded by:
  \begin{align}
    \label{eq:ineq-prj-rand-g}
    \sigma_n(\prj_{S^\perp}E) \ge \epsilon\sqrt{m-|\mc J| - r }.
  \end{align}
\end{lemma}
\begin{proof}
  For a matrix $A\in\R^{m\times n}$, define the fixed matrix $ P_{J^c}\in\R^{(m-|\mc J|)\times m}$ such that:
  \begin{align*}
    \left[ [P_{J^c}]_{[:,j]}:j\in \mc J\right]=0 ,\qquad \left[[P_{J^c}]_{[:,j]}:j\in [n]/\mc J\right] = I_{(m-|\mc
      J|)\times (m-|\mc J|)},
  \end{align*}
  which only keeps the coordinates that correspond to $[m]/\mc J$ of any vector in $\R^{m}$.
  Note that
  \begin{align*}
    \sigma_n(\prj_{S^\perp}E)&\ge \sigma_{n}(P_{J^c}(\prj_{S^\perp}E))
    \\
    &\ge \sigma_{n}(\prj_{(P_{J^c} S)^\perp}P_{J^c}\prj_{S^\perp}E)
    \\
    &=\sigma_n (\prj_{(P_{J^c} S)^\perp}P_{J^c}E).
  \end{align*}
  We justify the last equality below. Note that
  \begin{align*}
    \prj_{S^\perp}E  = E - \prj_{S}E,
  \end{align*}
  and note that the columns of $(P_{J^c}\prj_{S} E)$ lie in the column span of
  $P_{J^c} S$, therefore,
  \begin{align*}
    \prj_{(P_{J^c} S)^\perp} P_{J^c}\prj_{S^\perp}E &= \prj_{(P_{J^c} S)^\perp} P_{J^c}E -
    \prj_{(P_{J^c} S)^\perp} (P_{J^c}\prj_{S} E)
    \\\
    & = \prj_{(P_{J^c} S)^\perp} P_{J^c}E.
  \end{align*}
  Finally, note that $P_{J^c} S$, with column rank no more than $r$, is independent of
  $P_{J^c} E$, which is a random Gaussian matrix of size $(m-|\mc J|)\times n$,
  therefore we have that $\prj_{(P_{J^c} S)^\perp}P_{J^c}E$ is equivalent to a $( m-
  |\mc J|-r)\times n$ random Gaussian matrix. Since
  \eqref{eq:cond-prj-rand-g} is satisfied, we can apply Lemma~\ref{lem:rudelson} and
  conclude \eqref{eq:ineq-prj-rand-g} with high probability.
\end{proof}

However, in the smoothed analysis setting, the matrix we are interested in are often not random Gaussian matrices. Instead they are fixed matrices perturbed by Gaussian variables. We call these ``perturbed rectangular matrices'', their singular values can be bounded as follows:

\begin{lemma}[Perturbed rectangular matrices]
  \label{lem:robust-sig-min}
  Let $A \in \R^{m\times n}$ and suppose that $m \ge 3n$.  If all the entries of $A$ are independently
  $\rho$-perturbed to yield $\wt A$, then for any $\epsilon>0$, with probability at least
  $1-(C\epsilon)^{0.25m}$, for some absolute constant $C$, the smallest singular value of $\wt A$ is bounded
  below by:
  \begin{align*}
    \sigma_{n}(\wt A) \ge \epsilon\rho\sqrt{m}.
  \end{align*}

\end{lemma}

\begin{proof}
The idea is to use the previous lemma and project to the orthogonal subspace of $A$.
  We have that $\wt A= A + E$, where $E\in\R^{m\times n}$ is a random Gaussian matrix.
  %
  \begin{align*}
    \sigma_{n}(\wt A) \ge \sigma_n(\prj_{A^{\perp}} \wt A) = \sigma_n(\prj_{A^{\perp}} E).
  \end{align*}
  %
  Since $m-n>2n$, we can apply Lemma~\ref{lem:prj-rand-gaussian} to conclude that for any
  $\epsilon>0$,
  \begin{align*}
    \sigma_n(\prj_{A^{\perp}}E)\ge \epsilon \rho\sqrt{m},
  \end{align*}
  with probability at least $1-(C\epsilon)^{0.5(m-n)}\le 1-(C\epsilon)^{0.25m}$.
\end{proof}


\subsection{Projection of random vectors}

In Step 2, we need to bound the norm of a random vector of the form $u\od v$ after a projection,
where $u$ and $v$ are two Gaussian vectors. In order to show this, we apply the  result in
\cite{vu2013random} which provides a  concentration bound
of projection of well-behaved ($K$-concentrated) random vectors.

First we cite the definition of ``$K$-concentrated'' below:

\begin{definition}
\label{def:kconcentrated}
A random vector $X = (\xi_1, \xi_2, ..., \xi_n)$ is $K$-concentrated (where $K$ may depend on $n$)
if there are constants $C,C'>0$ such that for any convex, 1-Lipschitz function
$f:\mathbb{C}^n\rightarrow \R$ and for any $t > 0$, we have:
\begin{align*}
  \Pr[|F(X) - \mbox{med}(F(X))| \ge t] \le C \exp\left(-C'\frac{t^2}{K^2}\right),
\end{align*}
where $med(\cdot)$ denotes the median of a random variable (choose an arbitrary one if there are
many).
\end{definition}

\begin{lemma}[Concentration for Random Projections (Lemma 1.2 in \cite{vu2013random})]
  \label{lem:vanvu}
  Let $v$ be a $K$-concentrated random vector in $\mbb C^n$. The entries of $v$ has expected norm
  1. Then there are constants $C, C'>0$ such that the following holds. Let $\prj_S$ be a projection
  to a $d$-dimensional subspace in $\mbb C^n$.
  \begin{align*}
    \mbb P\lt( \lt| v^\top \prj_S v- d \rt| \ge 2t\sqrt{d} + t^2\rt) \le C\ep(-C'{t^2\over K^2}).
  \end{align*}
\end{lemma}

In order to apply this lemma in our setting, we need to prove the vectors that we are interested in
is $K$-concentrated:

\begin{lemma}
\label{lem:Kconcentrated}
Conditioned on the high probability event that $\|E_{[:,i]}\|, \|E_{[:,j]}\| \le 2\sqrt{n_2}$, the vector
$[[E_{[:,i]}\od E_{[:,j]}]_{s,s'}: s<s']$ is $2\sqrt{n_2})$-concentrated.
\end{lemma}

\begin{proof}
  For any $1$-Lipschitz function $F$ on $[[E_{[:,i]}\od E_{[:,j]}]_{s,s'}: s<s']$, we can define a
  function $G(E_{[:,i]},E_{[:,j]}) = F([[E_{[:,i]}\od E_{[:,j]}]_{s,s'}: s<s'])$ (if $i=j$ then
  the function $G$ only takes $E_{[:,i]}$ as the variable). Under the assumption that
  $\|E_{[:,i]}\|, \|E_{[:,j]}\| \le 2\sqrt{n_2}$, this new function $G$ is $2\sqrt{n_2}$-Lipschitz.

  Now we extend $G$ to $G^*$ when the input $\|E_{[:,i]}\|, \|E_{[:,j]}\| > 2\sqrt{n_2}$. Define the
  truncation function $\mbox{trunc}(v) = v$ for $\|v\| \le 2\sqrt{n_2}$, and $\mbox{trunc}(v) =
  2\sqrt{n_2}v/\|v\|$ for $\|v\| > 2\sqrt{n_2}$. Define the extended function
  $G^*(E_{[:,i]},E_{[:,j]}) = G(\mbox{trunc}(E_{[:,i]}),\mbox{trunc}(E_{[:,j]}))$, which is still
  $2\sqrt{n_2}$-Lipschitz since the truncation function is $1$-Lipschitz.

  Note that for the two Gaussian random vectors ${E_{[:,i]},E_{[:,j]}\sim N(0, I)}$, we can apply
  Gaussian concentration bound in Theorem~\ref{thm:gaussian-concen} on $G^*$, which implies
  \begin{align*}
    \mbb P[|G^*(E_{[:,i]},E_{[:,j]}) - \mbox{med}(G^*(E_{[:,i]},E_{[:,j]}))| \ge t] \le
    C\exp(-C't^2/4n_2).
  \end{align*}
  Since the probability of the event $\|E_{[:,i]}\|,\|E_{[:,j]}\|>2\sqrt{n_2}$ is very small
  ($\sim\exp(-\Omega(n_2))$), we have $\delta = |\mbox{med}(G(E_{[:,i]},E_{[:,j]})) -
  \mbox{med}(G^*(E_{[:,i]},E_{[:,j]}))| $ in the order of $ O(\sqrt{n_2})$.
  Therefore, for $t\sim \Omega(\sqrt{n_2})$, we have
  \begin{align*}
    \mbb P[|G^*(E_{[:,i]},E_{[:,j]}) - \mbox{med}(G(E_{[:,i]},E_{[:,j]}))| \ge t] &\le \mbb
    P[|G^*(E_{[:,i]},E_{[:,j]}) - \mbox{med}(G(E_{[:,i]},E_{[:,j]}))| \ge t - \delta ]
    \\
    &\le C\exp(-C't^2/4n_2).
  \end{align*}
  Finally,
  \begin{align*}
    &\mbb P\lt[ \Big |G(E_{[:,i]},E_{[:,j]}) - \mbox{med}(G(E_{[:,i]},E_{[:,j]}))| \ge t \Big |
    \|E_{[:,i]}\|,\|E_{[:,j]}\|\le 2\sqrt{n_2}\rt]
    \\
    \le &\frac{\mbb P[|G^*(E_{[:,i]},E_{[:,j]}) - \mbox{med}(G(E_{[:,i]},E_{[:,j]}))| \ge t]}{\mbb
      P[\|E_{[:,i]}\|\ge 2\sqrt{n_2}\mbox{ or }\|E_{[:,i]}\|\ge 2\sqrt{n_2}]}
    \\
    \le &C\exp(-C't^2/4n_2).
  \end{align*}
  Therefore the random vector $[[E_{[:,i]}\od E_{[:,j]}]_{s,s'}: s<s']$ is
  $2\sqrt{n_2}$-concentrated.
\end{proof}

\begin{theorem}[Gaussian concentration bound]
  \label{thm:gaussian-concen}
  Let $f:\R^{n}\to \R$ be a function which is Lipschitz with constant $1$. Consider a random vector
  $X\sim \mc N(0,I_n)$.  For any $s>0$ we have
  \begin{align*}
    \mbb P\lt(\big|f(X) - \mbb E[f(X)]\big|\ge s\rt) \le 2 e^{-C s^2},
  \end{align*}
  for all $s>0$ and some absolute constant $C>0$.
\end{theorem}

\subsection{Gaussian Chaoses}
In Step 2, we want to show that the inner product of two random vectors of the form $<\prj(u\od v),
\prj (u\od v)>$ is small, where $u,u'$ and $v,v'$ are Gaussian vectors.
In order to show this, we treat the inner product as a (homogeneous) Gaussian chaos, which is
defined to be a homogeneous polynomial over Gaussian random variables\footnote{In fact, the squared
  norm of projected random vectors considered previously is a special case of Gaussian chaos, and we
  treat it separately.}.
Our analysis builds on the results of many works studying the concentration bound of Gaussian
chaoses.

For decoupled Gaussian chaoses, we mostly use the following theorem, which is a simple corollary of
Lemma~\ref{lem:gaus-chaos}.

\begin{theorem}
\label{thm:chaos}
Suppose $a = (a_{i_1,...,i_d})_{1\le i_1,...,i_d\le n}$ is a $d$-indexed array, and $\|a\|_F$
denotes its Frobenius norm.  Let $(X^{(j)}_i)_{1\le i\le n, j=1,...,d}$ be independent copies of
$X\sim\mc N(0,I_{n})$. For any fixed $\epsilon > 0$, with probability at least $1- C\ep\lt( -C'
{n^{2\epsilon /d}} \rt)$,
\begin{align*}
  \left|\sum_{i_1,...,i_d=1}^n a_{i_1,...,i_d} X^{(1)}_{i_1}\cdots X^{(d)}_{i_d}\right| \le \|a\|_F
  n^\epsilon.
\end{align*}
\end{theorem}

\begin{lemma}[Gaussian chaoses concentration (Corollary 1 in
  \cite{latala2006estimates})]
  \label{lem:gaus-chaos}
  Suppose $a = (a_{i_1,...,i_d})_{1\le i_1,...,i_d\le n}$ is a $d$-indexed array.  Consider a
  decoupled Gaussian chaos $G = \sum_{i_1,...,i_d} a_{i_1,...,i_d} X_{i_1}^{(1)}\cdots X_{i_d}^{(d)}
  $, where $ X_{i}^{(k)}$ are independent copies of the standard normal random variable for all
  $i\in[n], k\in[d]$.
  \begin{align*}
    \mbb P\lt(|G|\ge t \rt) &\le  C_d \ep\lt(
    -{1 \over C_d} \ \min_{1\le k\le d} \ \min_{(I_1,\dots,I_k)\in
      S(k,d)} \lt( {t\over \|a\|_{I_1,\dots,I_k}}\rt)^{2/k}
    \rt),
  \end{align*}
  where $C_d\in (0,\infty)$ depends only on $d$, and $S(k,d)$ denotes a set of all partitions of
  $\{1,\dots,d\}$ into $k$ nonempty disjoint sets $I_1,\dots,I_k$, and the norm
  $\|\cdot\|_{I_1,\dots,I_k}$ is given by:
  \begin{align*}
    \|a\|_{I_1,\dots,I_k} \defeq \sup\lt\{ \sum_{i_1,...,i_d} a_{i_1,...,i_d}
    x_{i_{I_1}}^{(1)}\cdots x_{i_{I_k}}^{(k)}: \sum_{i_{I_1}}(x^{(1)}_{i_{I_1}})^2\le 1,\dots,
    \sum_{i_{I_k}}(x^{(k)}_{i_{I_k}})^2\le 1 \rt\}.
  \end{align*}
\end{lemma}

\begin{proof}
  (of Theorem~\ref{thm:chaos}) Apply the inequality:
  \begin{align*}
    \|a\|_{\{1\},\dots,\{d\}} \le \|a\|_{I_1,\dots,I_k} \le \|a\|_{[d]} = \|a\|_F, \quad
    \fa (I_1,\dots,I_k)\in S(k,d).
  \end{align*}
  For a fixed order $d$ and for any $\epsilon>0$, apply Lemma~\ref{lem:gaus-chaos} and set ${ t=
    n^{\epsilon}\|a\|_F}$.  We have that $ \mbb P\lt(|G|\ge t \rt) \le C\ep\lt( -C' {n^{2\epsilon
      /d}} \rt) $, for some constant $C, C'$.
\end{proof}

For coupled Gaussian chaoses, namely when $X^{(j)}$'s are identical copies of the same $X$, we first
cite the following decoupling theorem in \cite{de1995decoupling}.
\begin{theorem}
  (Decoupling) Let $(a_{i_1,...,i_d})_{1\le i_1,...,i_d\le n}$ be a symmetric $d$-indexed array such
  that $a_{i_1,...,i_d} = 0$ whenever there exists $k\neq l$ such that $i_k = i_l$. Let
  $X_1,...,X_n$ be independent random variables and $(X^{(j)}_i)_{1\le i\le n}$ for $j=1,dots, d$,
  be independent copies of the sequence $(X_i)_{1\le i \le n}$, then for all $t\ge 0$,
\begin{align*}
  L_d^{-1} \Pr\left[\left|\sum_{i_1,...,i_d=1}^n a_{i_1,...,i_d} X^{(1)}_{i_1}\cdots
      X^{(d)}_{i_d}\right|\ge L_d t\right]& \le \Pr\left[\left|\sum_{i_1,...,i_d=1}^n
      a_{i_1,...,i_d} X_{i_1}\cdots X_{i_d}\right|\ge L_d t\right]
  \\
  & \le L_d\Pr\left[\left|\sum_{i_1,...,i_d=1}^n a_{i_1,...,i_d} X^{(1)}_{i_1}\cdots
      X^{(d)}_{i_d}\right|\ge L_d^{-1} t\right],
\end{align*}
where $L_d\in (0,\infty)$ depends only on $d$.
\label{thm:decoupling}
\end{theorem}
Essentially this theorem shows for a symmetric tensor with no ``diagonal'' terms, i.e.,
$a_{i_1,...,i_d} = 0$ whenever there exists $k\neq l$ such that $i_k = i_l$), there is only a
constant factor difference between the coupled and decoupled Gaussian chaos distribution.

In most of our applications, we do have symmetric tensors with no ``diagonal'' terms. However there
is one case where we do have diagonal terms, for which we need the following lemma.
\begin{lemma}
\label{lem:gaussian3rd}
Let $(a_{i_1,i_2,i_3})_{1\le i_1,...,i_3\le n}$ be a symmetric $3$-indexed array and let $\|a\|_F$
denote its Frobenius norm. Let $X\sim\mc N(0,I_n)$, then for any $\epsilon > 0$, with probability at
least $1- Cn \ep(-C'n^{2\epsilon/3}) $,
\begin{align*}
  \left|\sum_{i_1,i_2,i_3=1}^n a_{i_1,i_2,i_3} X_{i_1}X_{i_2}X_{i_3}\right| \le 4\|a\|_F n^{0.5 +
    \epsilon}.
\end{align*}
\end{lemma}

\begin{proof}
  The sum of the ``diagonal'' terms is equal to $3\sum_{i\neq j} a_{i,i,j}X_i^2 X_j + 1/2
  \sum_{i}a_{i,i,i} X_i^3$.  Since $X_i$ are independent standard Gaussian random variables, with
  probability at least $1- Cn \ep(-C'n^{2\epsilon/3}) $ (union bound), $|X_i| \le n^{\epsilon/3}$
  for all $i\in[n]$. Conditioned on this high probability event, the absolute value of the sum is
  bounded by:
  \begin{align*}
    \lt|3\sum_{i\neq j} a_{i,i,j}X_i^2 X_j + {1\over 2} \sum_{i}a_{i,i,i} X_i^3 \rt| &\le 3
    \sum_{i,j =1}^n |a_{i,i,j}| |X_j| X_i^2
    \\
    &\le 3 \|(a_{i,i,j})_{1\le i,j\le n}\|_1  n^\epsilon
    \\
    &\le 3 \sqrt{n}\|(a_{i,i,j})_{1\le i,j\le n}\|_F  n^\epsilon
    \\
    &\le 3  \|a\|_F n^{0.5+\epsilon}.
  \end{align*}

  By Theorem~\ref{thm:chaos}, we know that with probability at least $1- C\ep\lt( -C' {n^{2\epsilon
      /3}} \rt)$, the absolute value of the sum of the ``non-diagonal'' terms is bounded by $
  \|a\|_F n^{\epsilon}$. Therefore we can conclude the proof by applying the union bound.
\end{proof}


\begin{thebibliography}{36}
\providecommand{\natexlab}[1]{#1}
\providecommand{\url}[1]{\texttt{#1}}
\expandafter\ifx\csname urlstyle\endcsname\relax
  \providecommand{\doi}[1]{doi: #1}\else
  \providecommand{\doi}{doi: \begingroup \urlstyle{rm}\Url}\fi

\bibitem[Anandkumar et~al.(2014)Anandkumar, Ge, Hsu, Kakade, and
  Telgarsky]{anandkumar2012tensor}
Animashree Anandkumar, Rong Ge, Daniel Hsu, Sham~M. Kakade, and Matus
  Telgarsky.
\newblock Tensor decompositions for learning latent variable models.
\newblock \emph{Journal of Machine Learning Research}, 15:\penalty0 2773--2832,
  2014.
\newblock URL \url{http://jmlr.org/papers/v15/anandkumar14b.html}.

\bibitem[Anderson et~al.(2013)Anderson, Belkin, Goyal, Rademacher, and
  Voss]{anderson2013more}
Joseph Anderson, Mikhail Belkin, Navin Goyal, Luis Rademacher, and James Voss.
\newblock The more, the merrier: the blessing of dimensionality for learning
  large gaussian mixtures.
\newblock \emph{arXiv preprint arXiv:1311.2891}, 2013.

\bibitem[Belkin and Sinha(2009)]{belkin2009learning}
Mikhail Belkin and Kaushik Sinha.
\newblock Learning gaussian mixtures with arbitrary separation.
\newblock \emph{arXiv preprint arXiv:0907.1054}, 2009.

\bibitem[Belkin and Sinha(2010)]{belkin2010polynomial}
Mikhail Belkin and Kaushik Sinha.
\newblock Polynomial learning of distribution families.
\newblock In \emph{Foundations of Computer Science (FOCS), 2010 51st Annual
  IEEE Symposium on}, pages 103--112. IEEE, 2010.

\bibitem[Bhaskara et~al.(2014)Bhaskara, Charikar, Moitra, and
  Vijayaraghavan]{bhaskara2013smoothed}
Aditya Bhaskara, Moses Charikar, Ankur Moitra, and Aravindan Vijayaraghavan.
\newblock Smoothed analysis of tensor decompositions.
\newblock In \emph{Proceedings of the 46th ACM symposium on Theory of
  computing}, 2014.

\bibitem[Brubaker and Vempala(2008)]{brubaker2008isotropic}
S~Charles Brubaker and Santosh~S Vempala.
\newblock Isotropic pca and affine-invariant clustering.
\newblock In \emph{Building Bridges}, pages 241--281. Springer, 2008.

\bibitem[Chan et~al.(2014)Chan, Diakonikolas, Servedio, and
  Sun]{densityestimation14}
Siu-On Chan, Ilias Diakonikolas, Rocco~A. Servedio, and Xiaorui Sun.
\newblock Efficient density estimation via piecewise polynomial approximation.
\newblock In \emph{Proceedings of the 46th Annual ACM Symposium on Theory of
  Computing}, STOC '14, pages 604--613, New York, NY, USA, 2014. ACM.
\newblock ISBN 978-1-4503-2710-7.
\newblock \doi{10.1145/2591796.2591848}.
\newblock URL \url{http://doi.acm.org/10.1145/2591796.2591848}.

\bibitem[Dasgupta(1999)]{dasgupta1999learning}
Sanjoy Dasgupta.
\newblock Learning mixtures of gaussians.
\newblock In \emph{Foundations of Computer Science, 1999. 40th Annual Symposium
  on}, pages 634--644. IEEE, 1999.

\bibitem[Dasgupta and Schulman(2000)]{dasgupta2000two}
Sanjoy Dasgupta and Leonard~J Schulman.
\newblock A two-round variant of em for gaussian mixtures.
\newblock In \emph{Proceedings of the Sixteenth conference on Uncertainty in
  artificial intelligence}, pages 152--159. Morgan Kaufmann Publishers Inc.,
  2000.

\bibitem[de~la Pe{\~n}a and Montgomery-Smith(1995)]{de1995decoupling}
Victor~H de~la Pe{\~n}a and Stephen~J Montgomery-Smith.
\newblock Decoupling inequalities for the tail probabilities of multivariate
  u-statistics.
\newblock \emph{The Annals of Probability}, pages 806--816, 1995.

\bibitem[Feldman et~al.(2006)Feldman, Servedio, and O'Donnell]{feldman2006pac}
Jon Feldman, Rocco~A Servedio, and Ryan O'Donnell.
\newblock Pac learning axis-aligned mixtures of gaussians with no separation
  assumption.
\newblock In \emph{Learning Theory}, pages 20--34. Springer, 2006.

\bibitem[Hardt et~al.(2014{\natexlab{a}})Hardt, Meka, Raghavendra, and
  Weitz]{hardt2014computational}
Moritz Hardt, Raghu Meka, Prasad Raghavendra, and Benjamin Weitz.
\newblock Computational limits for matrix completion.
\newblock In \emph{Proceedings of The 27th Conference on Learning Theory},
  pages 703--725, 2014{\natexlab{a}}.

\bibitem[Hardt et~al.(2014{\natexlab{b}})Hardt, Meka, Raghavendra, and
  Weitz]{moritz}
Moritz Hardt, Raghu Meka, Prasad Raghavendra, and Benjamin Weitz.
\newblock Computational limits for matrix completion.
\newblock In \emph{Proceedings of The 27th Conference on Learning Theory, COLT
  2014, Barcelona, Spain, June 13-15, 2014}, 2014{\natexlab{b}}.

\bibitem[Hsu and Kakade(2013)]{hsu2013learning}
Daniel Hsu and Sham~M Kakade.
\newblock Learning mixtures of spherical gaussians: moment methods and spectral
  decompositions.
\newblock In \emph{Proceedings of the 4th conference on Innovations in
  Theoretical Computer Science}, pages 11--20. ACM, 2013.

\bibitem[Jain and Oh(2014)]{jain2014learning}
Prateek Jain and Sewoong Oh.
\newblock Learning mixtures of discrete product distributions using spectral
  decompositions.
\newblock In \emph{Proceedings of The 27th Conference on Learning Theory},
  pages 824--856, 2014.

\bibitem[Jain et~al.(2013)Jain, Netrapalli, and Sanghavi]{jain2013low}
Prateek Jain, Praneeth Netrapalli, and Sujay Sanghavi.
\newblock Low-rank matrix completion using alternating minimization.
\newblock In \emph{Proceedings of the forty-fifth annual ACM symposium on
  Theory of computing}, pages 665--674. ACM, 2013.

\bibitem[Kalai et~al.(2009)Kalai, Samorodnitsky, and Teng]{kalai2009learning}
Adam~Tauman Kalai, Alex Samorodnitsky, and Shang-Hua Teng.
\newblock Learning and smoothed analysis.
\newblock In \emph{Foundations of Computer Science, 2009. FOCS'09. 50th Annual
  IEEE Symposium on}, pages 395--404. IEEE, 2009.

\bibitem[Kalai et~al.(2010)Kalai, Moitra, and Valiant]{kalai2010efficiently}
Adam~Tauman Kalai, Ankur Moitra, and Gregory Valiant.
\newblock Efficiently learning mixtures of two gaussians.
\newblock In \emph{Proceedings of the 42nd ACM symposium on Theory of
  computing}, pages 553--562. ACM, 2010.

\bibitem[Lata{\l}a et~al.(2006)]{latala2006estimates}
Rafa{\l} Lata{\l}a et~al.
\newblock Estimates of moments and tails of gaussian chaoses.
\newblock \emph{The Annals of Probability}, 34\penalty0 (6):\penalty0
  2315--2331, 2006.

\bibitem[McLachlan and Peel(2004)]{mclachlan2004finite}
Geoffrey McLachlan and David Peel.
\newblock \emph{Finite mixture models}.
\newblock John Wiley \& Sons, 2004.

\bibitem[Moitra and Valiant(2010)]{moitra2010settling}
Ankur Moitra and Gregory Valiant.
\newblock Settling the polynomial learnability of mixtures of gaussians.
\newblock In \emph{Foundations of Computer Science (FOCS), 2010 51st Annual
  IEEE Symposium on}, pages 93--102. IEEE, 2010.

\bibitem[Pearson(1894)]{pearson1894contributions}
Karl Pearson.
\newblock Contributions to the mathematical theory of evolution.
\newblock \emph{Philosophical Transactions of the Royal Society of London. A},
  pages 71--110, 1894.

\bibitem[Permuter et~al.(2003)Permuter, Francos, and
  Jermyn]{permuter2003gaussian}
H~Permuter, J~Francos, and H~Jermyn.
\newblock Gaussian mixture models of texture and colour for image database
  retrieval.
\newblock In \emph{Acoustics, Speech, and Signal Processing, 2003.
  Proceedings.(ICASSP'03). 2003 IEEE International Conference on}, volume~3,
  pages III--569. IEEE, 2003.

\bibitem[Recht et~al.(2010)Recht, Fazel, and Parrilo]{recht2010guaranteed}
Benjamin Recht, Maryam Fazel, and Pablo~A Parrilo.
\newblock Guaranteed minimum-rank solutions of linear matrix equations via
  nuclear norm minimization.
\newblock \emph{SIAM review}, 52\penalty0 (3):\penalty0 471--501, 2010.

\bibitem[Reynolds and Rose(1995)]{reynolds1995robust}
Douglas~A Reynolds and Richard~C Rose.
\newblock Robust text-independent speaker identification using gaussian mixture
  speaker models.
\newblock \emph{Speech and Audio Processing, IEEE Transactions on}, 3\penalty0
  (1):\penalty0 72--83, 1995.

\bibitem[Rudelson and Vershynin(2009)]{rudelson2009smallest}
Mark Rudelson and Roman Vershynin.
\newblock Smallest singular value of a random rectangular matrix.
\newblock \emph{Communications on Pure and Applied Mathematics}, 62\penalty0
  (12):\penalty0 1707--1739, 2009.

\bibitem[Sanjeev and Kannan(2001)]{sanjeev2001learning}
Arora Sanjeev and Ravi Kannan.
\newblock Learning mixtures of arbitrary gaussians.
\newblock In \emph{Proceedings of the thirty-third annual ACM symposium on
  Theory of computing}, pages 247--257. ACM, 2001.

\bibitem[Spielman and Teng(2004)]{spielman2004smoothed}
Daniel~A Spielman and Shang-Hua Teng.
\newblock Smoothed analysis of algorithms: Why the simplex algorithm usually
  takes polynomial time.
\newblock \emph{Journal of the ACM (JACM)}, 51\penalty0 (3):\penalty0 385--463,
  2004.

\bibitem[Stewart and Sun(1990)]{stewart1990matrix}
Gilbert~W Stewart and Ji-guang Sun.
\newblock \emph{Matrix perturbation theory}.
\newblock Academic press, 1990.

\bibitem[Stewart(1977)]{stewart1977perturbation}
GW~Stewart.
\newblock On the perturbation of pseudo-inverses, projections and linear least
  squares problems.
\newblock \emph{SIAM review}, 19\penalty0 (4):\penalty0 634--662, 1977.

\bibitem[Tao and Vu(2006)]{tao2006random}
Terence Tao and Van Vu.
\newblock On random$\pm$1 matrices: singularity and determinant.
\newblock \emph{Random Structures \& Algorithms}, 28\penalty0 (1):\penalty0
  1--23, 2006.

\bibitem[Titterington et~al.(1985)Titterington, Smith, Makov,
  et~al.]{titterington1985statistical}
D~Michael Titterington, Adrian~FM Smith, Udi~E Makov, et~al.
\newblock \emph{Statistical analysis of finite mixture distributions},
  volume~7.
\newblock Wiley New York, 1985.

\bibitem[Valiant(2012)]{valiant2012algorithmic}
Gregory~John Valiant.
\newblock \emph{Algorithmic approaches to statistical questions}.
\newblock PhD thesis, University of California, Berkeley, 2012.

\bibitem[Vempala and Wang(2004)]{vempala2004spectral}
Santosh Vempala and Grant Wang.
\newblock A spectral algorithm for learning mixture models.
\newblock \emph{Journal of Computer and System Sciences}, 68\penalty0
  (4):\penalty0 841--860, 2004.

\bibitem[Vu and Wang(2013)]{vu2013random}
Van Vu and Ke~Wang.
\newblock Random weighted projections, random quadratic forms and random
  eigenvectors.
\newblock \emph{arXiv preprint arXiv:1306.3099}, 2013.

\bibitem[Zoran and Weiss(2012)]{NIPS2012_4758}
Daniel Zoran and Yair Weiss.
\newblock Natural images, gaussian mixtures and dead leaves.
\newblock In F.~Pereira, C.J.C. Burges, L.~Bottou, and K.Q. Weinberger,
  editors, \emph{Advances in Neural Information Processing Systems 25}, pages
  1736--1744. Curran Associates, Inc., 2012.
\newblock URL
  \url{http://papers.nips.cc/paper/4758-natural-images-gaussian-mixtures-and-dead-leaves.pdf}.

\end{thebibliography}
\end{document}